\documentclass[10pt]{article}
\usepackage{fullpage}
\usepackage[margin=1in]{geometry}

\usepackage{epsf}
\usepackage{fancyheadings}
\usepackage{graphics}
\usepackage{graphicx,subfigure}
\usepackage{psfrag}
\usepackage{bm}

\usepackage{color}

\usepackage{amsthm}
\usepackage{amsfonts}
\usepackage{amsmath}
\usepackage{amssymb}
\usepackage{mathrsfs}

\usepackage{accents}
\usepackage{listings}

\usepackage{caption}
\usepackage[linesnumbered, ruled, vlined]{algorithm2e}

\theoremstyle{plain}

\newtheorem{theorem}{Theorem}

\newtheorem{proposition}{Proposition}
\newtheorem{lemma}{Lemma}
\newtheorem{corollary}{Corollary}

\newtheorem{definition}{Definition}
\newtheorem{example}{Example}
\newtheorem{fact}{Fact}
\newtheorem{assumption}{Assumption}

\newlength{\widebarargwidth}
\newlength{\widebarargheight}
\newlength{\widebarargdepth}

\makeatletter
\long\def\@makecaption#1#2{
        \vskip 0.8ex
        \setbox\@tempboxa\hbox{\small {\bf #1:} #2}
        \parindent 1.5em  
        \dimen0=\hsize
        \advance\dimen0 by -3em
        \ifdim \wd\@tempboxa >\dimen0
                \hbox to \hsize{
                        \parindent 0em
                        \hfil 
                        \parbox{\dimen0}{\def\baselinestretch{0.96}\small
                                {\bf #1.} #2
                                } 
                        \hfil}
        \else \hbox to \hsize{\hfil \box\@tempboxa \hfil}
        \fi
        }
\makeatother

\long\def\comment#1{}

\newcommand{\vecnorm}[2]{\| #1\|_{#2}}

\newcommand{\inprod}[2]{\ensuremath{\langle #1 , \, #2 \rangle}}

\DeclareMathOperator*{\argmin}{argmin}

\newcommand{\NORMAL}{\ensuremath{\mathcal{N}}}

\newcommand{\UNIF}{\ensuremath{\text{Unif}}}

\newcommand{\vvec}{\ensuremath{\mathbf{v}}}
\newcommand{\Xvec}{\ensuremath{\mathbf{X}}}

\newcommand{\EE}{\ensuremath{\mathbb{E}}}

\newcommand{\alphabold}{\ensuremath{\bm{\alpha}}}
\newcommand{\betabold}{\ensuremath{\bm{\beta}}}

\newcommand{\Deltabold}{\ensuremath{\bm{\Delta}}}
\newcommand{\alphastar}{\ensuremath{\bm{\alpha}^*}}

\newcommand{\Sigmabold}{\ensuremath{\bm{\Sigma}}}

\newcommand{\Deltahat}{\ensuremath{\widehat{\bm{\Delta}}}}
\newcommand{\alphahat}{\ensuremath{\widehat{\bm{\alpha}}}}
\newcommand{\betahat}{\ensuremath{\widehat{\bm{\beta}}}}

\newcommand{\testerr}{\ensuremath{\mathcal{E}_{\mathsf{test}}}}

\newcommand{\Oh}{\ensuremath{\mathcal{O}}}

\newcommand{\Ind}{\ensuremath{\mathbb{I}}}
\newcommand{\reals}{\ensuremath{\mathbb{R}}}
\newcommand{\cone}{\ensuremath{\mathbb{C}}}

\newcommand{\quadfigureexterior}[3]{
\begin{figure}[htbp]
 #3
 \caption{#2}
 \label{#1}
\end{figure}
}

\newcommand{\multifigureexterior}[3]{\quadfigureexterior{#1}{#2}{#3}}

\newcommand{\subfigl}[4]{\subfigure[#3]{\includegraphics[width=#1]{#4} \label{#2}}}
\newcommand{\subfig}[3]{\subfigl{#1}{}{#2}{#3}}

\newcommand{\supp}{\mathsf{supp}}
\newcommand{\Atrain}{\mathbf{A}_{\mathsf{train}}}
\newcommand{\Btrain}{\mathbf{B}_{\mathsf{train}}}

\newcommand{\Ytrain}{\mathbf{Y}_{\mathsf{train}}}
\newcommand{\Wtrain}{\mathbf{W}_{\mathsf{train}}}
\newcommand{\bvec}{\mathbf{b}}
\newcommand{\avec}{\mathbf{a}}
\newcommand{\Bmat}{\mathbf{B}}
\usepackage[utf8]{inputenc} 
\usepackage[T1]{fontenc}
\usepackage{url}
\usepackage{ifthen}
\usepackage{cite}
\usepackage{amsmath} 

\begin{document}

\title{Harmless interpolation of noisy data in regression} 


\author{
    Vidya Muthukumar, Kailas Vodrahalli, Vignesh Subramanian and Anant Sahai \\
    BLISS \& ML4Wireless, EECS, UC Berkeley\\
    \{vidya.muthukumar, kailasv, vignesh.subramanian, asahai\}@berkeley.edu 
}


\maketitle



\begin{abstract}
A continuing mystery in understanding the empirical success of deep neural networks is their ability to achieve zero training error and generalize well, even when the training data is noisy and there are more parameters than data points.  
We investigate this overparameterized regime in linear regression, where all solutions that minimize training error interpolate the data, including noise.
We characterize the fundamental generalization (mean-squared) error of any interpolating solution in the presence of noise, and show that this error decays to zero with the number of features.
Thus, overparameterization can be explicitly beneficial in ensuring harmless interpolation of noise.
We discuss two root causes for poor generalization that are complementary in nature -- signal ``bleeding" into a large number of alias features, and overfitting of noise by parsimonious feature selectors.
For the sparse linear model with noise, we provide a hybrid interpolating scheme that mitigates both these issues and achieves order-optimal MSE over all possible interpolating solutions.
\end{abstract}


\section{Introduction}

In statistical machine learning, we have high-dimensional data in the form of $n$ covariate-response pairs $(\Xvec_i \in \reals^d, Y_i)_{i=1}^n$.
When training parametric models for functions fitting covariate to response, the traditional wisdom~\cite{friedman2001elements} is to select function classes with a number of parameters $d < n$.
In classification (when the labels are discrete), the scaling of the test error with respect to $n$ is determined by the VC-dimension~\cite{vapnik1999overview}/Rademacher complexity~\cite{bartlett2005local} of the function class, which in the worst case increases with the number of parameters $d$.
In regression (when the labels are continuous), the mean-squared error (MSE) of the least-squares estimator scales as the condition number of the data matrix, which is reasonable for smaller ratios of $d/n$ but increases astronomically as $d$ approaches $n$.
The qualitative fear is the same: if the function class is too complex, it starts to overfit noise and generalizes poorly to unseen test data. 

This wisdom has been challenged by the recent advent of deeper and deeper neural networks.
In particular, a thought-provoking paper~\cite{zhang2016understanding} noted that several deep neural networks generalize well despite achieving zero or close to zero training error, and being so expressive that they even have the ability to fit pure noise.
As they put it, ``understanding deep learning requires rethinking generalization". 
How can we reconcile the fact that good interpolating solutions exist with the classical bias-variance tradeoff?

These phenomena are being actively investigated for neural networks and kernel methods.
A natural \textit{simpler} setting to investigate is overparameterized \textit{linear regression} on noisy data.
The overparameterized (high-dimensional) linear model has the advantage of being classically studied under estimation procedures that incorporate some kind of regularization\footnote{The most commonly known are Tikhonov (ridge) regularization and $\ell_1$-regularization, but we refer the interested reader to Miller's book~\cite{miller2002subset} for a comprehensive list.}.
Such procedures typically avoid fitting noise, thus incurring some non-zero training error.
We are now interested in whether we can achieve any success with solutions that interpolate the training data.
In particular, {\bf can we ever interpolate the noise and retain favorable generalization guarantees? If so, which interpolators should we use, and which should we avoid?}

In this paper, we provide constructive answers to the above questions for overparameterized linear regression using elementary machinery.
Our contributions are as follows:
\begin{enumerate}
    \item We give a fundamental limit (Theorem~\ref{thm:idealinterpolator}, Corollaries~\ref{cor:fundamentalprice} and~\ref{cor:crabpot}) for the excess MSE of \textit{any} interpolating solution in the presence of noise, and show that it converges to $0$ as the number of features goes to infinity for feature families satisfying mild conditions.
    \item We provide a Fourier-theoretic interpretation of concurrent analyses~\cite{hastie2019surprises,belkin2019two,bartlett2019benign,bibas2019new,mitra2019understanding} of the minimum $\ell_2$-norm interpolator.
    \item We show (Theorem~\ref{thm:parsimoniousnoisefit}) that parsimonious interpolators (like the $\ell_1$-minimizing interpolator and its relatives) suffer the complementary problem of \textit{overfitting pure noise}.
    \item We construct two-step hybrid interpolators that successfully recover signal and harmlessly fit noise, achieving the order-optimal rate of test MSE among all interpolators (Proposition~\ref{prop:sparseinterpolators} and all its corollaries).
\end{enumerate}

\subsection{Related work}\label{sec:relatedwork}

We discuss prior work in three categories: a) overparameterization in \textit{deep} neural networks, b) \textit{interpolation} of high-dimensional data using kernels, and c) high-dimensional linear regression.
We then recap work on overparameterized linear regression that is concurrent to ours.

\subsubsection{Recent interest in overparameterization}

Conventional statistical wisdom is that using more parameters in one's model than data points leads to poor generalization.
This wisdom is corroborated in theory by worst-case generalization bounds on such overparameterized models following from VC-theory in classification~\cite{vapnik1999overview} and ill-conditioning in least-squares regression~\cite{miller2002subset}.
It is, however, contradicted in practice by the notable recent trend of empirically successful \textit{overparameterized} deep neural networks.
For example, the commonly used CIFAR-$10$ dataset contains $60000$ images, but the number of parameters in all the neural networks achieving state-of-the-art performance on CIFAR-$10$ is at least $1.5$ million~\cite{zhang2016understanding}.
These neural networks have the ability to memorize pure noise -- somehow, they are still able to generalize well when trained with meaningful data.

Since the publication of this observation~\cite{neyshabur2014search,zhang2016understanding}, the machine learning community has seen a flurry of activity to attempt to explain this phenomenon, both for classification and regression problems, in neural networks.
The problem is challenging for three core reasons\footnote{This exposition is inspired by Suriya Gunasekar's presentation at the Simons Institute, Summer $2019$.}:
\begin{enumerate}
\item The optimization landscape for loss functions on neural networks is notoriously non-convex and complicated, and even proving convergence guarantees to \textit{some} global minimum, as is observed in practice in the overparameterized regime~\cite{neyshabur2014search,zhang2016understanding}, is challenging.
\item In the overparameterized regime, there are multiple global minima corresponding to a fixed neural network architecture and loss function -- which of these minima the optimization algorithm selects is not always clear.
\item Tight generalization bounds for the global minimum that is selected need to be obtained to show that overparameterization can help with generalization.
This is particularly non-trivial to establish for \textit{deep}, i.e. $\geq 3$-layer neural networks.
\end{enumerate}

Promising progress has been made in all of these areas, which we recap only briefly below.
Regarding the first point, while the optimization landscape for deep neural networks is non-convex and complicated, several independent recent works (an incomplete list is~\cite{allen2019convergence,allen2018learning,azizan2019stochastic,chizat2018global,du2018gradient,mei2018mean,soltanolkotabi2018theoretical}) have shown that overparameterization can make it more attractive, in the sense that optimization algorithms like stochastic gradient descent (SGD) are more likely to actually converge to a global minimum.
These interesting insights are mostly unrelated to the question of generalization, and should be viewed as a coincidental benefit of overparameterization.

Second, a line of recent work~\cite{soudry2018implicit,gunasekar2018characterizing,nacson2019convergence,woodworth2019kernel} characterizes the \textit{inductive biases} of commonly used optimization algorithms, thus providing insight into the identity of the global minimum that is selected.
For the \textit{logistic loss} of linear predictors on separable data, the global optimum that is found by SGD in overparameterized settings has been shown to be precisely the margin-maximizing support vector machine (SVM)~\cite{soudry2018implicit}, and extensions have been obtained for deep \textit{linear} networks~\cite{nacson2019convergence} and stochastic mirror descent~\cite{gunasekar2018characterizing}.
More pertinent to linear regression, the behavior of SGD on the quadratic loss function~\cite{woodworth2019kernel} depends on the initialization point: the global convergence guarantee smoothly interpolates between a ``kernel" regime (corresponding to the minimum $\ell_2$-norm interpolator) and a ``structured" regime (corresponding to the minimum $\ell_1$-norm interpolator).

Finally, simple theoretical insights into which solutions (global minima) generalize well under what conditions, if any, remain elusive.
The generalizing ability of solutions can vary for different problem instances: adaptive methods\footnote{for which, interestingly, the induced inductive bias is unknown.} in optimization need not always improve generalization for specially constructed examples in overparameterized linear regression~\cite{wilson2017marginal}.
On the other hand, on a different set of examples~\cite{shah2018minimum}, adaptive methods converge to better-generalizing solutions than SGD.
Evidence suggests that norm-based complexity measures predict generalizing ability~\cite{neyshabur2014search}, and for neural networks, such complexity measures have been developed that do not depend on the width, but can depend on the depth~\cite{neyshabur2015norm,bartlett2017spectrally,golowich2017size,neyshabur2017exploring}.
A classification-centric explanation for the possibility of overparameterization improving generalization is that SGD with logistic-style loss converges to the solution that maximizes training data margin.
Then, increasing overparameterization increases model flexibility, allowing for solutions that increase the margin.
This is a classical observation for AdaBoost~\cite{schapire1998boosting} and is recently given as a justification for using overparameterization in neural networks~\cite{wei2018margin}.
However, margin does not always imply generalization~\cite{shah2018minimum}.
An alternative, intriguing explanation~\cite{wyner2017explaining} for this phenomenon does not consider margin, but instead connects the AdaBoost procedure to random forests, i.e. ensembles of \textit{randomly initialized} decision trees, each of which interpolate the training data.
An averaging and localizing-of-noise effect that is shown to be present both in the random forests ensemble, and the interpolating ensemble of AdaBoost, results in good generalization.

\subsubsection{Kernels for \textit{interpolation} of data}

In the overparameterized regime, solutions that minimize classification/regression loss \textit{interpolate} the training data.
Another class of functions that have the ability to interpolate the training data are not explicitly overparameterized -- they are non-parametric functions corresponding to particular reproducing kernel Hilbert spaces (RKHS).
In fact, Belkin, Ma and Mandal~\cite{belkin2018understand} empirically recovered several of the overparameterization phenomena of deep learning in kernel classifiers that interpolate\footnote{In the paper, a subtle distinction is made between overfitting, which corresponds to close to zero classification loss, and interpolation, which corresponds to close to zero squared loss.} the training data.
They observed that these solutions generalize well, even in the presence of \textit{label noise}; and moreover, regularization (either through explicit norm control or early stopping of SGD) yields only a marginal improvement\footnote{This was also observed by~\cite{neyshabur2014search} with overparameterized neural networks.}.
Finally, they showed that the minimum (RKHS) norm of such interpolators in the presence of label noise cannot explain these properties; and generalizing ability likely depends on specific structure of the kernel.
Other interpolators (e.g. based on local methods) have subsequently been analyzed for specific kernels~\cite{belkin2018overfitting,belkin2019does}, but most relevant to the setting of regression is recent analysis of the test MSE of the minimum-RKHS-norm kernel interpolator~\cite{liang2018just}.
It was shown here that this could be controlled under appropriate conditions on the eigenvalues of the kernel as well as the data matrix, and critical to these results is the dimension of the data growing with the number of samples~\cite{rakhlin2018consistency}.
Implicitly, all the analyses require successful interpolators to have a delicate balance between \textit{preserving} the structure of the true function explaining the (noiseless) data and \textit{minimizing} the harmful effect of regression noise in the data: thus, the properties of the chosen kernel are key.
We will see that this tradeoff manifests very explicitly and clearly in high-dimensional linear regression.

\subsubsection{High-dimensional linear regression}

Most recently, a \textit{double-descent} curve on the test error ($0-1$ loss and MSE) as a function of the number of parameters of several parametric models was observed on several common datasets by physicists~\cite{geiger2018jamming} and machine learning researchers~\cite{belkin2019reconciling} respectively.
In these experiments, the minimum $\ell_2$-norm interpolating solution is used, and several feature families, including kernel approximators~\cite{rahimi2008random}, were considered.
The experiments showed that the effect of increased parameterization on the test MSE is nuanced: at $d \sim n$, fitting noise has an adverse effect on the test MSE; but as the number of features $d$ becomes many times more than the number of data points $n$, the test MSE decays and approaches the test MSE of a corresponding kernel interpolator of minimum RKHS norm.

When the number of features is greater than the number of training data points, this corresponds to overparameterized linear regression (on lifted features).
This is more commonly known as high-dimensional, or underdetermined linear regression.
The classical underdetermined analysis~\cite{miller2002subset} tells us that the data matrix could be ill-conditioned, especially at $d \sim n$: thus, underdetermined linear regression is usually carried out with some form of regularization.
One of the earliest regularization procedures is Tikhonov ($\ell_2$) regularization, which can be used to improve the conditioning of the data matrix - elegant analysis of this for a random design, but in the overdetermined regime, is considered in~\cite{hsu2012random}.
In high dimensions, $\ell_2$-regularization unfortunately leads to a generic ``bleeding" across features that causes a loss of the signal; this classical observation, and exceptions to this observation, are discussed in detail in Section~\ref{sec:l2}.
In fact, signal recovery in the high-dimensional regime, even in the absence of noise, is only possible when there is sparsity in the coefficients\footnote{This is precisely shown by information-theoretic lower bounds~\cite{wainwright2009information,aeron2010information} and justifies the principle to ``bet on sparsity"~\cite{friedman2001elements}.}.
The literature on sparse signal recovery is rich, with algorithms that are motivated by iterative perspectives from signal processing (orthogonal matching pursuit (OMP)~\cite{pati1993orthogonal} and stagewise OMP~\cite{donoho2012sparse}) and convex relaxation (basis pursuit/Lasso~\cite{chen2001atomic}).
Unlike $\ell_2$-regularized solutions, which tend to spread energy out across multiple features, these procedures are parsimonious and select features in a data-dependent manner.
Traditionally, these solutions are applied on denoised data -- or, equivalently, with regularization\footnote{To be precise, they are analyzed for sufficiently large values of the regularizer $\lambda$. 
An analysis of the equivalent interpolators corresponds to an analysis tending $\lambda \to 0$.
} to explicitly \textit{avoid} fitting noise.
We investigate the performance of solutions that interpolate both signal and noise in the high-dimensional, sparse linear model in Section~\ref{sec:sparse}.

\subsubsection{Concurrent work in high-dimensional linear regression}

Since the publication\footnote{In fact, as pointed out in~\cite{bartlett2019benign}, the initial discussions around this area began at the Simons Institute $2017$ program on foundations of machine learning.} of the double descent experiments~\cite{geiger2018jamming,belkin2019reconciling} at the end of the year $2018$, there has been extremely active interest in understanding the generalization abilities of interpolating solutions in linear regression.
An earlier edition of our work was presented at Information Theory and Applications, February $2019$ and subsequently accepted to IEEE International Symposium on Information Theory, July $2019$.
Several elegant and interesting papers~\cite{hastie2019surprises,bartlett2019benign,belkin2019two,bibas2019new, mitra2019understanding,mei2019generalization} have appeared around this time.
All of these center around the analysis of the $\ell_2$-minimizing interpolator.
We discuss~\cite{hastie2019surprises,bartlett2019benign,belkin2019two} in substantial detail in Section~\ref{sec:l2}, and provide a succinct description of each of these paper's contributions here:

\begin{enumerate}
\item Belkin, Hsu and Xu~\cite{belkin2019two} consider an effectively misspecified setting for Gaussian and Fourier features and recover the double descent phenomenon for the $\ell_2$-minimizing interpolator in cases where the data can be generated as a linear combination of a huge number of features -- therefore, adding more features into the model in a ``prescient" manner can have an approximation-theoretic benefit.
On the other hand, double descent no longer happens under this model when the features are randomly selected.
\item Hastie, Tibshirani, Rosset and Montanari~\cite{hastie2019surprises} analyze the \textit{asymptotic} risk of the $\ell_2$-minimizing interpolator for independent and whitened features with a bounded fourth moment as a function of the ``overparameterizing" factor: $\frac{\text{number of features}}{\text{number of samples}}$.
They asymptotically prove several phenomena corresponding to the non-asymptotic results in our paper -- harmless noise fitting with increased overparameterizing factor (our Corollaries~\ref{cor:fundamentalprice} and~\ref{cor:crabpot}), what we call (in our discussion in Section~\ref{sec:l2}) signal ``bleed" and the resulting convergence to the null risk.
Additionally, they study the misspecified model asymptotically and show that double descent can occur with sufficiently improved approximability, and replicate the phenomena in a specific non-linear model.
\item Bartlett, Long, Lugosi and Tsigler~\cite{bartlett2019benign} sharply upper and lower bound the (non-asymptotic) generalization error of the $\ell_2$-minimizing interpolator for Gaussian features (which are not whitened/independent in general).
They characterize necessary and sufficient conditions for the $\ell_2$-minimizing interpolator to avoid what we call signal ``bleed" and noise overfitting in terms of functionals of the spectrum of the Gaussian covariance matrix.
Initial discussions with Peter Bartlett as well as the paper's random-matrix-theory-based analysis significantly inspired our subsequent Fourier-theoretic interpretation that we provide in Section~\ref{sec:l2}.
\item Bibas, Fogel and Feder~\cite{bibas2019new} study online/universal linear regression in the overparameterized regime and show that when the high-dimensional training data is mostly along certain pre-determined directions that favor the true parameter vector, good generalization is possible with both $\ell_2$-regularized and minimum $\ell_2$-norm interpolators used for prediction.
\item Mitra~\cite{mitra2019understanding} also studies the asymptotically overparameterized regime and obtains precise analytical risk curves for both the $\ell_2$-minimizing interpolator and the $\ell_1$-minimizing interpolator, with special focus on the overfitting peak at the ``interpolation threshold", i.e. $d \sim n$.
\end{enumerate}

Subsequent to all of this work, Mei and Montanari~\cite{mei2019generalization} very recently proved that Tikhonov regularization in the overparameterized regime can give rise to the double-descent behavior when random Fourier features are used, in \textit{both} the noiseless and noisy cases.
They also recover this behavior with $\ell_2$-minimizing interpolation, but in the low-noise (high SNR) limit.
For a more complete understanding of interpolation in overparameterized linear regression, we recommend to the reader all of these papers along with ours.

\section{Problem Setting}

Throughout, we consider data that is \textit{actually generated} from the high-dimensional/overparameterized linear model\footnote{The more general misspecified case is considered in concurrent and subsequent work~\cite{hastie2019surprises,mei2019generalization}.}.
We consider covariate-response pairs $(\Xvec_i, Y_i \in \reals^p \times \reals)_{i=1}^n$ and generative model $Y = \inprod{\avec(\Xvec)}{\alphastar} + W$ for feature vector $\avec(\Xvec) \in \reals^d$ and Gaussian \textit{noise} $W \sim \NORMAL(0, \sigma^2)$ that is independent of $X$.
We generically assume that the covariates $\{\Xvec_i\}_{i=1}^n$ are iid random samples, but will also consider regularly spaced data on bounded domains for some special cases. 
The \textit{signal} $\alphastar$ is unknown apriori to an estimator.
We also assume a distribution on $\Xvec \in \reals^p$, which induces a distribution on the $d$-dimensional feature vector $\avec(\Xvec)$.
Let $\Sigmabold = \EE[\avec(\Xvec) \avec(\Xvec)^\top]$ denote the covariance matrix of the feature vector under this induced distribution.
We assume that $\Sigmabold$ is invertible; therefore it is positive definite and its square-root-inverse $\Sigmabold^{-1/2}$ exists.

We define shorthand notation for the training data: let
\begin{align*}
    \Atrain := \begin{bmatrix}
    \avec(\Xvec_1)^\top & \avec(\Xvec_2)^\top & \ldots \avec(\Xvec_n)^\top
    \end{bmatrix}^\top
\end{align*}

denote the data (feature) matrix, and let $\Ytrain, \Wtrain \in \reals^n$ denote the output and noise vectors respectively.

We will primarily consider the overparameterized, or high-dimensional regime, i.e. where $d > n$.
We are interested in solutions $\alphabold$ that satisfy the following \textit{feasibility condition} for interpolation:
\begin{align}\label{eq:interpolatingsoln}
    \Atrain \alphabold = \Ytrain
\end{align}

We assume that $\text{rank}(\Atrain) = n$, so the set $\{\alphabold \in \reals^d: \Atrain \alphabold = \Ytrain \}$ is non-empty in $\reals^d$.

For any solution $\alphahat \in \reals^d$, we define the generalization error as test MSE\footnote{This constitutes all quadratic loss functions on error in estimation of the signal $\alphastar$, and is the standard error metric for regression.} below.

\begin{definition}
The expected test mean-squared-error (MSE) \textbf{minus irreducible noise error} of any estimator $\alphahat((\Xvec_i,Y_i)_{i=1}^n)$ is given by
\begin{align*}
    \testerr(\alphahat) := \EE[(Y - \inprod{\avec(\Xvec)}{\alphahat})^2] - \sigma^2
\end{align*}

where the expectation is taken \textbf{only} over the joint distribution on the fresh test sample $(\Xvec,Y)$, and we subtract off the \textbf{irreducible} noise error $\EE[W^2] = \sigma^2$.
\end{definition}

We have chosen the convention to subtract off the unavoidable error arising from noise, $\sigma^2$, as is standard.
From now on, we will denote this quantity to be the test MSE as shorthand.

\section{The fundamental price of interpolation}

Before analyzing particular interpolating solutions, we want to understand whether interpolation can ever lead to a desirable guarantee on the test MSE $\testerr$.
To do this, we characterize the fundamental price that any interpolating solution needs to pay in test MSE.
The constraint in Equation~\eqref{eq:interpolatingsoln} is sufficiently restrictive to not allow trivial solutions of the form $\alphabold = \alphastar$ --- so this is a surprisingly well-posed problem.
In fact, we can easily define the \textit{ideal interpolator} below.

\begin{definition}
\label{def:ideal_interpolator}
The \textit{ideal interpolator} $\alphahat_{\mathsf{ideal}}$ is defined as:
\begin{align*}
    \alphahat_{\mathsf{ideal}} := {\arg \min} \text{ } \testerr(\alphabold) \\
    \text{ subject to }
    \Atrain \alphabold = \Ytrain .
\end{align*}

We also denote the test MSE of the ideal interpolator, which we henceforth call the \textbf{ideal test MSE}, as $\testerr^*$. This is, by definition, a lower bound on the test MSE of any interpolator.
\end{definition}

The following result \textit{exactly} characterizes the ideal interpolator and the ideal test MSE.
\begin{theorem}\label{thm:idealinterpolator}
For any joint distribution on $(\Xvec,Y)$ and realization of training data matrix $\Atrain$ and noise vector $\Wtrain$, the lowest possible test MSE any interpolating solution can incur is bounded below as $\testerr \geq \testerr^*$, where
\begin{align}\label{eq:idealMSE}
    \testerr^* = \Wtrain^\top (\Btrain \Btrain^\top)^{-1} \Wtrain .
\end{align}
Here, $\Btrain := \Atrain \Sigmabold^{-1/2}$ is the \textbf{whitened} training data matrix.
\end{theorem}

The proof of Theorem~\ref{thm:idealinterpolator} is outlined in Section~\ref{sec:thm1proof}.
Theorem~\ref{thm:idealinterpolator} provides an explicit expression for a fundamental limit on the generalization ability of interpolation.
Thus, we can easily evaluate it (numerically) when the training data matrix $\Atrain$ is generated by a number of choices for feature families $\avec(\Xvec)$.
These choices are listed below as examples.

\begin{example}[Gaussian features]\label{eg:gaussian}
The Gaussian features on $d$-dimensional data comprise of $\avec(\Xvec) := \Xvec \sim \NORMAL(\mathbf{0}, \Sigmabold)$, where $\Sigmabold \in \reals^{d \times d}$ and $\Sigmabold \succ 0$.
A special case is iid Gaussian features, i.e. $\Sigmabold = \mathbf{I}_d$.
\end{example}

\begin{example}[Fourier features]\label{eg:fourier}
Let $i := \sqrt{-1}$ denote the imaginary number.
For one-dimensional data $X \in [0,1]$, we can write the $d$-dimensional Fourier features in their \textbf{complex form} as
\begin{align*}
\avec(X) = \begin{bmatrix}
1 & e^{2 \pi i X} & e^{2\pi (2i)X} & \ldots & e^{2\pi ((d-1)i)) X}
\end{bmatrix} \in \mathbb{C}^d .
\end{align*}


This is clearly an \textbf{orthonormal} feature family in the sense that \newline $\EE_{X \sim \UNIF[0,1]}\left[\avec(X)_j \avec(X)_k^* \right] = \delta_{j,k}$, where $\delta_{j,k}$ denotes the Kronecker delta and $(\cdot)^*$ denotes the complex conjugate.
When evaluating interpolating solutions for Fourier features, we will consider one of two models for the data $\{X_i\}_{i=1}^n$:
\begin{enumerate}
\item $n$-regularly spaced training data points, i.e. $x_i = \frac{(i-1)}{n} \text{ for all } i \in [n]$, which we consider empirically \textbf{and} theoretically in Section~\ref{sec:l2}.
\item $n$-random training data points, i.e. $X_i \text{ i.i.d } \sim \UNIF[0,1]$, which we evaluate only empirically.
\end{enumerate}
\end{example}

\begin{example}[Legendre/Vandermonde polynomial features]\label{eg:legendre}
For one-dimensional data $X \in [-1,1]$, we can write the $d$-dimensional Vandermonde features as 
\begin{align*}
\avec(X) = \begin{bmatrix}
1 & X & X^2 & \ldots X^{d-1} .
\end{bmatrix}
\end{align*}
We can also uniquely define their orthonormalization with respect to the uniform measure on $[-1,1]$.
In other words, we define the $d$-dimensional Legendre features as polynomials
\begin{align*}
\avec(X) = \begin{bmatrix}
p_0(X) & p_1(X) & \ldots p_{d-1}(X) ,
\end{bmatrix}
\end{align*}
where $\text{deg}(p_j(X)) = j$ for every $j \geq 0$, and $\{p_j(X)\}_{j \geq 0}$ is uniquely defined such that $\EE_{X \sim \UNIF[-1,1]} \left[p_j(X) p_k(X)\right]\allowbreak = \delta_{j,k}$, i.e. the Legendre polynomials form the orthonormal basis with respect to the uniform measure on $[-1,1]$.
When evaluating interpolating solutions for both these polynomial features, we consider one of two models for the training data $\{X_i\}_{i=1}^n$:
\begin{enumerate}
\item $n$-regularly spaced training data points, i.e. $x_i = -1 + \frac{2(i-1)}{n} \text{ for all } i \in [n]$.
\item $n$-random training data points, i.e. $X_i \text{ i.i.d } \sim \UNIF[-1, 1]$.
\end{enumerate}
\end{example}

\begin{figure}[htbp]
  \centering
  \includegraphics[width=3in]{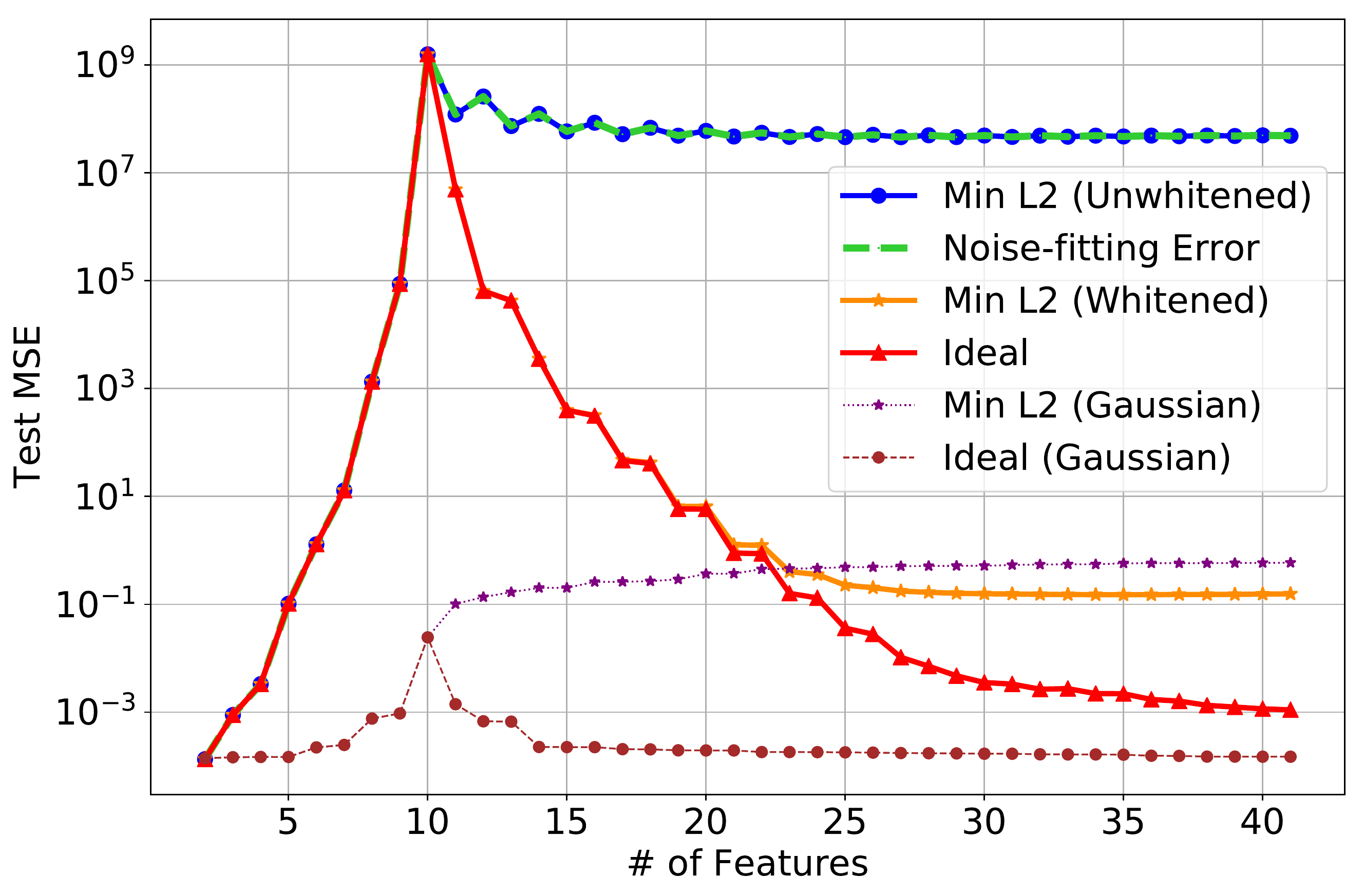}
  \caption{Test MSE for polynomial and Gaussian features. Polynomials -- here, data is sampled from $\text{Unif}[-1,1]$ with $n=10$, $k=2$ (a degree 1 polynomial), and $W \sim \NORMAL(0, 10^{-4})$. The unwhitened case uses Vandermonde features. The whitened version of these features correspond to the Legendre polynomial features. Gaussian -- we sample from $\NORMAL(0, 1)$ and use the same settings as for polynomials.
  } 
  \label{fig:poly_MSE}
\end{figure}

Figure~\ref{fig:RG_sparse_MSE} evaluates the quantity $\testerr^*$ as a function of $d$ for iid Gaussian features.
For $d < n$, we always evaluate the test MSE of the unique least-squares solution as Equation~\eqref{eq:interpolatingsoln} is no longer feasible.
We observe a spike at $d = n$, and a decay in the generalization error as $d >> n$, implying that potentially harmful effects of noise can be mitigated for these feature families.
These properties also manifest in Figure~\ref{fig:poly_MSE} for the orthonormalized Legendre polynomial features, but not for the Vandermonde polynomial features --- illustrating the importance of the whitening step.

\subsection{Converse: Lower bound on test MSE for interpolating solutions}

As mentioned in Definition~\ref{def:ideal_interpolator}, the quantity $\testerr^*$ represents the minimum test error that any interpolating solution satisfying Equation~\eqref{eq:interpolatingsoln} \textit{must} incur.
Noteworthy from the expression in Equation~\eqref{eq:idealMSE} is that the test MSE naturally depends on the singular value spectrum of the random matrix $\Btrain^\top$ that has columns $\{\bvec(\Xvec_i)\}_{i=1}^n$ corresponding to \textit{whitened} features, i.e. $\EE\left[\bvec(\Xvec_i)\bvec(\Xvec_i)^\top\right] = \mathbf{I}_d$ for every $i \in [n]$.
Because of this whitening, we might expect that the matrix $\Btrain^\top$ becomes better and better conditioned and the error arising solely from fitting noise reduces.
To understand the fundamental price of interpolation, we make one of three assumptions on columns of $\Btrain^\top$, denoted by $\{\bvec(\Xvec_i) := \Sigmabold^{-1/2} \avec(\Xvec_i)\}_{i=1}^n$.
Recall that the random feature vectors $\bvec(\Xvec_i)$ are always iid, because we have assumed iid samples $\{\Xvec_i\}_{i=1}^n$.
We state our candidate assumptions in increasing order of strength.

\begin{assumption}\label{as:ideal1}[iid heavy-tailed feature vectors]
The random feature vectors $\bvec(\Xvec_i)$ are bounded almost surely, i.e. $\vecnorm{\bvec(\Xvec_i)}{2} \leq \sqrt{d}$ almost surely.
Note that this assumption is satisfied by discrete Fourier features and random Fourier features.
\end{assumption}

\begin{assumption}\label{as:ideal2}[iid feature vectors with sub-Gaussianity]
The whitened feature vectors $\bvec(\Xvec_i)$ are sub-Gaussian with parameter at most $K > 0$.
A special case of this includes the case of independent \textbf{entries}: in this case, the random \textbf{variables} $B_{ij}$ are independent and sub-Gaussian, all with parameter at most $K > 0$.
We reproduce the definition of sub-Gaussianity and sub-Gaussian parameter for both random vector and random variable in Definition~\ref{def:subgaussian} in Appendix~\ref{app:technical}.
\end{assumption}


\begin{assumption}\label{as:ideal3}[iid Gaussian entries]
The entries of the whitened feature matrix $\Btrain$ are iid Gaussian, i.e. $B_{ij} \text{ i.i.d } \sim \NORMAL(0,1)$.
Note that this exactly describes all cases where the \textbf{original} data matrix $\Atrain$ has iid Gaussian row vectors, i.e. $\avec(\Xvec_i) \sim \NORMAL(\mathbf{0}, \Sigma)$.
\end{assumption}

Notice that the Gaussian Assumption~\ref{as:ideal3} constitutes a special case of sub-Gaussianity of rows (Assumption~\ref{as:ideal2}).
Independence of elements of the feature vector, \textit{even when the features are whitened}, is impossible when lower-dimensional data is lifted into high-dimensional features, i.e. the problem is one of \textit{lifted} linear regression.
It is in view of this that we have included consideration of the far weaker assumptions of sub-Gaussianity of random feature \textit{vectors} (Assumption~\ref{as:ideal2}) and even heavy-tailed features (Assumption~\ref{as:ideal1}).
We will see that the strength of the conclusions we can make is accordingly lower for these more general cases.
However, for random feature vectors satisfying \textit{any} of the above assumptions (which, together, constitute very mild conditions), we can always characterize the fundamental price of interpolation by lower bounding the ideal MSE.

\begin{corollary}\label{cor:fundamentalprice}
For any $0 < \delta < 1$, the fundamental price of any interpolating solution is at least:
\begin{enumerate}
\item
\begin{align}\label{eq:fp_ideal1}
    \testerr^* \geq \left(\frac{n(1 - \delta)}{(C\sqrt{d \ln n} + \sqrt{n})^2} + 1\right) \sigma^2 = \omega\left(\frac{ \sigma^2n}{d \ln n}\right)
\end{align}
with probability greater than or equal to $(1 - \frac{1}{n} - e^{-\delta^2 n/8})$ for any feature family for which the random whitened feature matrix $\Btrain$ satisfies the heavy-tailed Assumption~\ref{as:ideal1}.
Here, $C > 0$ is some positive constant independent of the choice of feature family.
\item 
\begin{align}\label{eq:fp_ideal2}
    \testerr^* \geq \left(\frac{n(1 - \delta)}{(C_K\sqrt{d} + \sqrt{n})^2} + 1\right) \sigma^2  = \omega\left(\frac{\sigma^2 n}{d}\right)
\end{align}
with probability greater than or equal to $(1 - e^{-c_K n} - e^{-\delta^2 n/8})$ for any feature family for which the random whitened feature matrix $\Btrain$ satisfies Assumption~\ref{as:ideal2} of iid sub-Gaussian \textbf{rows}.
(This includes the special case in which $\Btrain$ has independent sub-Gaussian \textbf{entries}.)
Here, $C_K, c_K > 0$ are positive constants that depend on the upper bound on the sub-Gaussian parameter, $K > 0$.
\item 
\begin{align}\label{eq:fp_ideal3}
    \testerr^* \geq \left(\frac{n(1 - \delta)}{(\sqrt{d} + 2\sqrt{n})^2} + 1\right) \sigma^2  = \omega\left(\frac{\sigma^2 n}{d}\right)
\end{align}
with probability greater than or equal to $(1 - e^{-n/2} - e^{-n \delta^2/8})$ for any Gaussian feature family, i.e. any feature family satisfying the Gaussian Assumption~\ref{as:ideal3}.
\end{enumerate}
\end{corollary}

Corollary~\ref{cor:fundamentalprice} characterizes the fundamental price of \textit{any} interpolating solution as $\widetilde{\Oh}\left(\frac{n}{d}\right)$ at a significant level of generality\footnote{Note that in the case of data satisfying Assumption~\ref{as:ideal1}, the $\widetilde{\Oh}(\cdot)$ omits the $(\ln n)$ factor in the denominator arising from heavy tailed-ness.}.
It tells us that extreme overparameterization is \textit{essential} for harmlessness of interpolation of noise.
To see this, consider how the number of features $d$ could scale as a function of the number of samples $n$.
Say that we grew $d = \gamma n$ for some \textit{constant} $\gamma > 1$. Then, the lower bound on test MSE (minus the irreducible error $\sigma^2$ arising from the prospect of noise in the test points as well) scales as $\omega\left(\frac{\sigma^2 n}{d}\right) = \omega\left(\frac{\sigma^2}{\gamma}\right)$, which asymptotes to a constant as $n \to \infty$.
This tells us that \textit{the level of overparameterization necessarily needs to grow faster than the number of samples} for harmless interpolation to even be possible. For example, this could happen at a polynomial rate (e.g. $d = n^q$ for some $q > 1$) or even an exponential rate (e.g. $d = e^{\lambda n}$ for some $\lambda > 0$).


\subsection{The possibility of harmless interpolation}

Corollary~\ref{cor:fundamentalprice} only provides a lower bound on the ideal MSE, not an upper bound.
Thus, it does not tell us whether harmless interpolation is ever actually \textit{possible}.
This turns out to be a more delicate question in general, and is difficult to characterize under the weaker Assumptions~\ref{as:ideal1} and~\ref{as:ideal2} in their full generality (for a detailed discussion of why this is the case, see Appendix~\ref{app:technical}).
However, for two special cases: a) Gaussian features (Assumption~\ref{as:ideal3}), and b) independent sub-Gaussian feature vectors (Assumption~\ref{as:ideal2} for independent entries): we can show that harmless interpolation is always possible, with an upper bound on the ideal MSE that matches the lower bound provided in Corollary~\ref{cor:fundamentalprice}.
We state this result below.

\begin{corollary}\label{cor:crabpot}
For any $0 < \delta < 1$ and: 
\begin{enumerate}
    \item Random whitened feature matrix $\Btrain$ satisfying the Gaussian Assumption~\ref{as:ideal3}, the fundamental price of interpolation is at most
\begin{align}\label{eq:crabpot}
    \testerr^* \leq \left(\frac{n(1 + \delta)}{(\sqrt{d} - 2\sqrt{n})^2} + 1\right) \sigma^2  = \Oh\left(\frac{\sigma^2 n}{d}\right)
\end{align}
with probability greater than or equal to $(1 - e^{-n/2} - e^{-n \delta^2/8})$.
    \item Random matrix feature matrix $\Btrain$ satisfying \textbf{independent, sub-Gaussian entries} with unit variance (special case of Assumption~\ref{as:ideal2}), the fundamental price of interpolation is at most
\begin{align}\label{eq:crabpot_sg}
\testerr^* \leq \left(\frac{4C_K^2 n(1 + \delta)}{(\sqrt{d} - \sqrt{n-1})^2} + 1\right) \sigma^2  = \Oh\left(\frac{\sigma^2 n}{d}\right)
\end{align}
with probability greater than or equal to $\left(1 - \left(\frac{1}{2}\right)^{d - n + 1} - c_K^d - e^{-n \delta^2/8}\right)$, where constants $c_K \in (0,1)$ and $C_K > 0$ only depend on the upper bound on the sub-Gaussian parameter $K$.
\end{enumerate}
\end{corollary}

Thus, when the features given to us in the form of training data matrix $\Atrain$ are normally distributed (in general unwhitened), Corollary~\ref{cor:crabpot} helps explain the harmless interpolations that we are seeing empirically.
Together with Corollary~\ref{cor:fundamentalprice}, this is a sharp, up to constants\footnote{The concurrent work of Hastie, Tibshirani, Rosset and Montanari~\cite{hastie2019surprises} exactly characterizes the effect of noise-fitting for feature vectors with independent entries (Lemma $3$ in their paper) satisfying a bounded fourth moment assumption in the asymptotic regime where $d, n \to \infty, \lim \frac{d}{n} = \gamma$.
Their scaling in $\gamma$ is equivalent to ours in the constant high-dimensional regime $d = \gamma n$.
Their result is a consequence of being able to exactly characterize the asymptotic distribution of the spectrum of the infinite-dimensional random matrix $\Btrain \Btrain^\top$ by applying a generalized Marchenko-Pastur law.
We use non-asymptotic random matrix theory for our results and thus sharply characterize the dependence on $n$ and $d$ in our rates, but not the dependence on constant factors.
}, characterization of the fundamental price of any interpolating solution when using a Gaussian feature family\footnote{For the Gaussian case, this scaling can also be derived from the work of Bartlett, Long, Lugosi and Tsigler~\cite{bartlett2019benign}, which analyes the generalization error of the $\ell_2$-minimizing interpolation of unwhitened Gaussian features to Gaussian noise more generally.
As we see from the second part of Corollary~\ref{cor:crabpot}, this scaling holds more generally for whitened, iid sub-Gaussian features.
}. 
To make this notion of harmlessness more concrete, it is useful to consider two \textit{ultra-overparameterized} settings:
\begin{enumerate}
\item $d = n^q$ for some parameter $q > 1$, which constitutes a \textit{polynomially} high-dimensional regime. Here, the ideal test MSE is upper bounded by $\Oh\left(\frac{\sigma^2 n}{n^q}\right) = \omega\left(\frac{\sigma^2}{n^{1-q}}\right)$, which goes to $0$ as $n \to \infty$ at a polynomial rate.
\item $d = e^{\lambda n}$ for some parameter $\lambda > 0$, which constitutes an \textit{exponentially} high-dimensional regime.
Here, ideal test MSE is at most $\Oh\left(\sigma^2 ne^{- \lambda n}\right)$ which goes to $0$ as $n \to \infty$ at an exponentially decaying rate. 
\end{enumerate}

Thus, there always exists an interpolating solution that fits noise in such a manner that \textit{the effect of fitting this noise on test error} decays to $0$ as the number of features goes to infinity.
Of course, Corollary~\ref{cor:crabpot} is not particularly meaningful for $d \sim n$, i.e. near the interpolation threshold: for a detailed discussion of this regime, see Section~\ref{sec:threshold}.

We defer the proofs of Corollary~\ref{cor:fundamentalprice} and Corollary~\ref{cor:crabpot} for the more general cases of sub-Gaussian and heavy-tailed feature vectors (Assumptions~\ref{as:ideal1} and~\ref{as:ideal2}) to Appendix~\ref{app:technical}.
We here prove Theorem~\ref{thm:idealinterpolator} and Corollaries~\ref{cor:fundamentalprice} and~\ref{cor:crabpot} for the iid Gaussian case (Assumption~\ref{as:ideal3}).

\subsection{Proof of Theorem~\ref{thm:idealinterpolator}, Corollaries~\ref{cor:fundamentalprice} and~\ref{cor:crabpot}}\label{sec:thm1proof}

To prove Theorem~\ref{thm:idealinterpolator}, we first get an exact expression for the ideal interpolator as defined in Definition~\ref{def:ideal_interpolator}.
A simple calculation gives us
\begin{align*}
    \testerr(\alphahat) &= \EE[(\inprod{\avec(\Xvec)}{\alphastar} + W - \inprod{\avec(\Xvec)}{\alphahat})^2] - \sigma^2 \\
    &= \EE[(\inprod{\avec(\Xvec)}{\alphahat - \alphastar})^2] + \EE[W^2] \\
    &= \EE[(\alphahat - \alphastar)^\top \avec(\Xvec) \avec(\Xvec)^\top (\alphahat - \alphastar)] + \sigma^2 - \sigma^2 \\
    &= (\alphahat - \alphastar)^\top \Sigmabold (\alphahat - \alphastar) + \sigma^2 - \sigma^2 \\
    &= \vecnorm{\Sigmabold^{1/2} (\alphahat - \alphastar)}{2}^2 + \sigma^2 - \sigma^2 \\
    &= \vecnorm{\Sigmabold^{1/2} (\alphahat - \alphastar)}{2}^2 .
\end{align*}

Thus, we can equivalently characterize the ideal interpolating solution $\alphahat_{\mathsf{ideal}} := {\arg \min} \text{ } \testerr(\alphahat)$ as:
\begin{align*}
    \alphahat_{\mathsf{ideal}} &:= {\arg \min} \vecnorm{\Sigmabold^{1/2} (\alphabold - \alphastar)}{2} \\
    &\text{ subject to Equation~\eqref{eq:interpolatingsoln} holding.}
\end{align*}

Observe that Equation~\eqref{eq:interpolatingsoln} can be rewritten as 
\begin{align*}
    \Atrain (\alphabold - \alphastar) &= \Wtrain \\
    \implies \Atrain \Sigmabold^{-1/2} \Sigmabold^{1/2} (\alphabold - \alphastar) &= \Wtrain .
\end{align*}

Then we have a closed form expression for the minimum norm solution, denoting $\Btrain = \Atrain \Sigmabold^{-1/2}$:
\begin{align*}
    \Sigmabold^{1/2} (\alphahat_{\mathsf{ideal}} - \alphastar) = \Btrain^\top (\Btrain \Btrain^\top)^{-1} \Wtrain 
\end{align*}

(note that $\Btrain^\dagger := \Btrain^\top (\Btrain \Btrain^\top)^{-1}$ is just the right Moore-Penrose pseudoinverse of $\Btrain$).

Substituting this expression into the test MSE calculation:
\begin{align*}
    \testerr^* &= \testerr(\alphahat_{\mathsf{ideal}}) = \vecnorm{\Btrain^\top (\Btrain \Btrain^\top)^{-1} \Wtrain}{2}^2 \\
    &= \Wtrain^\top (\Btrain \Btrain^\top)^{-1} \Wtrain
\end{align*}

Plugging this into the expression of $\testerr^*$ gives us Equation~\eqref{eq:idealMSE}, thus proving Theorem~\ref{thm:idealinterpolator}.
\qed

To prove Corollaries~\ref{cor:fundamentalprice} and~\ref{cor:crabpot}, we lower bound and upper bound the test MSE of the ideal interpolator in Equation~\eqref{eq:idealMSE}.
We use matrix concentration theory~\cite{vershynin2010introduction} to do this for the Gaussian case (Assumption~\ref{as:ideal3}).
We first state the following concentration result on the non-zero singular values of random matrix $\mathbf{B}$ with entries $B_{ij} \text{ i.i.d } \sim \NORMAL(0,1)$ as is from Vershynin's book~\cite{vershynin2010introduction}.
The original argument is contained in classical work by Davidson and Szarek~\cite{davidson2001local}.

\begin{lemma}[Theorem $5.32$, ~\cite{vershynin2010introduction}]\label{lem:gaussianconcentration}
Let $\sigma_{min}(\cdot)$ and $\sigma_{max}(\cdot)$ denote the minimum and maximum (non-degenerate) singular values of a matrix.
For random matrix $\Bmat \in \reals^{n \times d}$ having entries $B_{ij} \text{ i.i.d. } \sim \NORMAL(0,1)$, we have 
\begin{align}\label{eq:singularvaluegaussian}
\sqrt{d} - \sqrt{n} - t \leq \sigma_{min}(\Bmat) \leq \sigma_{max}(\Bmat) \leq \sqrt{d} + \sqrt{n} + t
\end{align}
with probability at least $(1 - e^{-t^2/2})$.
\end{lemma}

We first use this lemma to prove Corollary~\ref{cor:fundamentalprice}.
We lower bound Equation~\eqref{eq:idealMSE} as
\begin{align*}
    \testerr^* &\geq \frac{\vecnorm{\Wtrain}{2}^2}{\lambda_{max}((\Btrain^\top )^\top\Btrain^\top)} .
\end{align*}

Now, we apply the \textit{upper} bound on the \textit{maximum} singular value of $\Btrain^\top$ as stated in Lemma~\ref{lem:gaussianconcentration}, substituting $t := \sqrt{n}$ to get
\begin{align}\label{eq:st3}
    \testerr^* &\geq \frac{\vecnorm{\Wtrain}{2}^2}{(\sqrt{d} + 2\sqrt{n})^2}
\end{align}
with probability at least $(1 - e^{-n/2})$.
Further, we have $\vecnorm{\Wtrain}{2}^2 = \sum_{i=1}^n W_i^2$ and the lower tail bound on chi-squared random variables~\cite[Chapter $2$]{wainwright2019high} gives us
\begin{align}\label{eq:st4}
    \vecnorm{\Wtrain}{2}^2 \geq n\sigma^2(1 - \delta)
\end{align}
with probability greater than or equal to $(1 - e^{-n \delta^2/8})$ for any $0 < \delta < 1$.
When Equations~\eqref{eq:st3} and~\eqref{eq:st4} both hold, we get the statement of Corollary~\ref{cor:fundamentalprice}.
\qed


Now, we prove Corollary~\ref{cor:crabpot}.
Denoting $\Btrain^T = \begin{bmatrix}
\bvec_1 & \ldots & \bvec_n
\end{bmatrix}$, we can upper bound Equation~\eqref{eq:idealMSE} as
\begin{align*}
    \testerr^* &\leq \frac{\vecnorm{\Wtrain}{2}^2}{\lambda_{min}((\Btrain^\top )^\top\Btrain^\top)} .
\end{align*}
Observe that the matrix $\Btrain$ has entries $B_{ij} \sim \NORMAL(0,1)$ due to the whitening.
Thus, we can apply the \textit{lower} bound on the \textit{minimum} singular value of $\Btrain^\top$ as stated in Lemma~\ref{lem:gaussianconcentration}, substituting $t := \sqrt{n}$ to get
\begin{align}\label{eq:st1}
    \testerr^* &\leq \frac{\vecnorm{\Wtrain}{2}^2}{(\sqrt{d} - 2\sqrt{n})^2} 
\end{align}
with probability at least $(1 - e^{-n/2})$.

Further, we have $\vecnorm{\Wtrain}{2}^2 = \sum_{i=1}^n W_i^2$ and the corresponding upper tail bound on chi-squared random variables\footnote{One could have also used the Hanson-Wright inequality~\cite{hanson1971bound} to get slightly more precise constants, but the tight concentration of the singular values of the random matrix $\Btrain$ implies that only constant factors would be improved.} gives us
\begin{align}\label{eq:st2}
    \vecnorm{\Wtrain}{2}^2 \leq n\sigma^2(1 + \delta)
\end{align}
with probability greater than or equal to $(1 - e^{-n \delta^2/8})$ for any $\delta > 0$.
When Equations~\eqref{eq:st1} and~\eqref{eq:st2} both hold, we get the upper inequality in Equation~\eqref{eq:crabpot}, thus proving Corollary~\ref{cor:crabpot}.

Using the union bound on the probability of non-event of Equations~\eqref{eq:st3},~\eqref{eq:st4},~\eqref{eq:st1} or~\eqref{eq:st2}, gives us the statements of Corollaries~\ref{cor:crabpot} as well as~\ref{cor:fundamentalprice} with probability greater than or equal to $(1 - 2e^{-n/2} - 2e^{-\delta^2 n/8})$.
This provides a characterization of the fundamental price of interpolation for the Gaussian case.
\qed

The proof of Theorem~\ref{thm:idealinterpolator} tells us that the ideal interpolator fits what can be thought of as \textit{a residual signal}, $(\alphabold - \alphastar)$, to noise.
This is well and truly an \textit{ideal} interpolator, as it requires knowledge of the true signal $\alphastar$ to implement -- and thus it should only be thought of as primarily a fundamental limit on generalization.
To understand whether we could achieve this limit, we turn to practically used interpolating schemes, beginning with the $\ell_2$-minimizing interpolator which has seen the most concurrent (empirical and theoretical) attention.

\section{The minimum-$\ell_2$-norm interpolator through the Fourier lens}\label{sec:l2}

For overparameterized linear models, it is natural to consider the minimum-$\ell_2$-norm solution that interpolates the data, i.e. satisfies Equation~\eqref{eq:interpolatingsoln}. 
After all, the proof of Theorem~\ref{thm:idealinterpolator} shows that the least harmful way in which to fit noise is to find the minimum-$\ell_2$-norm interpolator of (effectively) \textit{pure noise} using appropriately whitened features.
Furthermore, the minimum-$\ell_2$-norm interpolator is explicitly characterizable as a linear matrix operator on the output vector $\Ytrain$, and can be easily computed as well.
For example, gradient descent on the overparameterized linear regression problem with appropriate initialization converges to the minimum-$\ell_2$-norm interpolator~\cite{shah2018minimum, woodworth2019kernel}. 

As mentioned in Section~\ref{sec:relatedwork}, a couple of excellent papers~\cite{hastie2019surprises, bartlett2019benign} that center around comprehensive analyses of the $\ell_2$-minimizing interpolator were published concurrently\footnote{Through personal communication with Peter Bartlett in December 2018/January 2019, we became aware of the gist of the investigations underlying~\cite{bartlett2019benign} when we first presented our work.
In the earlier version of this work presented at IEEE ISIT 2019, we cite this communication accordingly.}.
When it comes to whitened or Gaussian features, any results discussed in this section are all reflected, more or less, across~\cite{hastie2019surprises} and~\cite{bartlett2019benign} respectively -- and we recommend reading these papers for the details.
The results, like our proofs of Theorem~\ref{thm:idealinterpolator}, Corollary~\ref{cor:fundamentalprice} and~\ref{cor:crabpot}, use fundamental advances in asymptotic and non-asymptotic random matrix theory.
Here, we provide a brief alternative exposition of the main ideas through a Fourier-theoretic lens on \textit{regularly spaced} data (Example~\ref{eg:fourier}).
Our aim for doing this is two-fold: one, simply to complement these papers; and two, to provide a signal processing oriented perspective on salient properties of the minimum-$\ell_2$-norm interpolator.

\multifigureexterior{fig:fourier_paradigmatic_all}{The converse bounds for interpolation, plotted on a log-log scale for $n = 15$. Here, the median is plotted for clarity where the randomness is due to how the features and training points are drawn. Notice that all the curves are similar in the significantly overparameterized regime, i.e. $d/n \geq 10$.}{
\subfigl{0.49\textwidth}{fig:fourier_paradigmatic}{Median plots.}{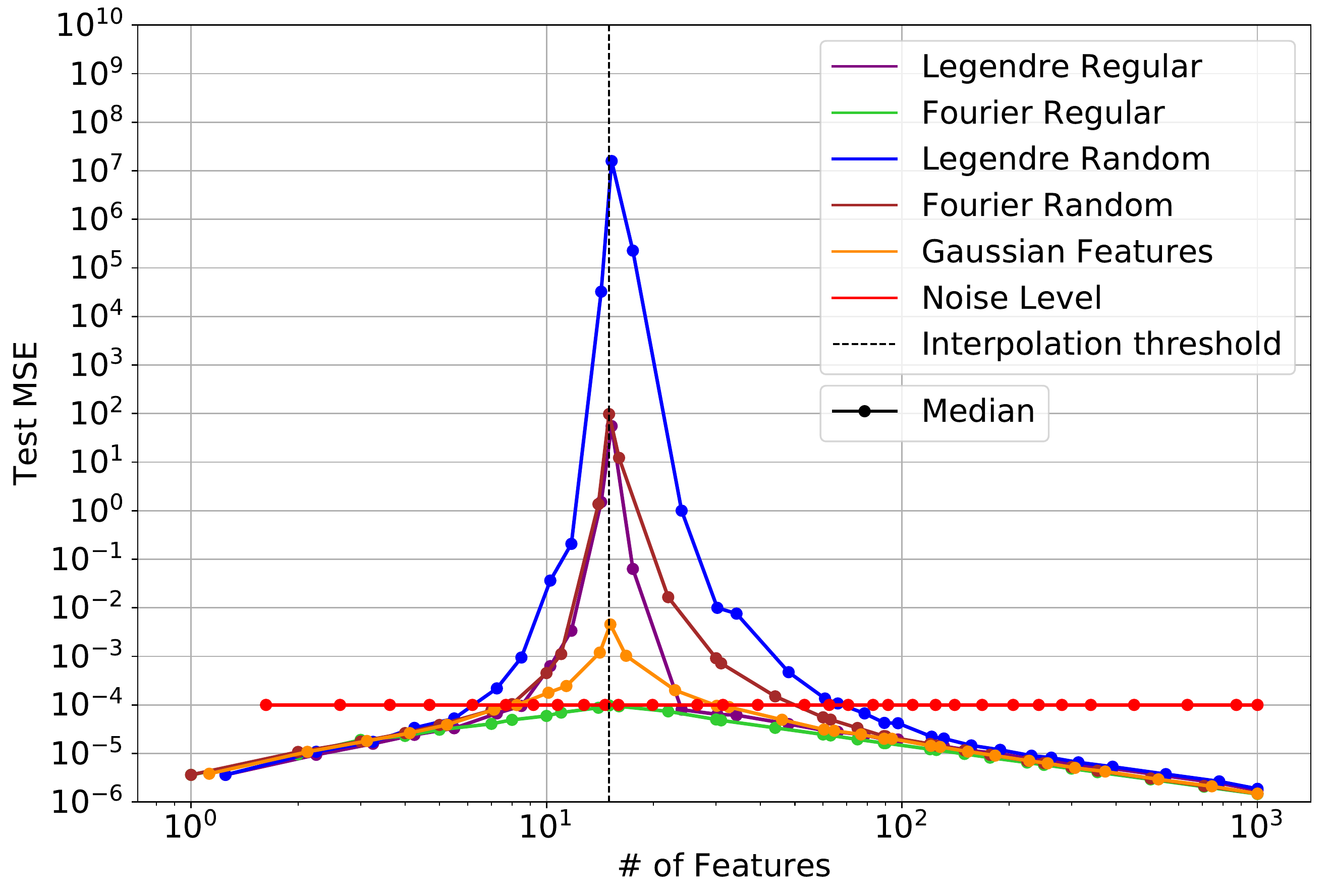}
\subfigl{0.49\textwidth}{fig:fourier_paradigmatic_errorbars}{Median plots with error bars. Notice the variability near the interpolation threshold.}{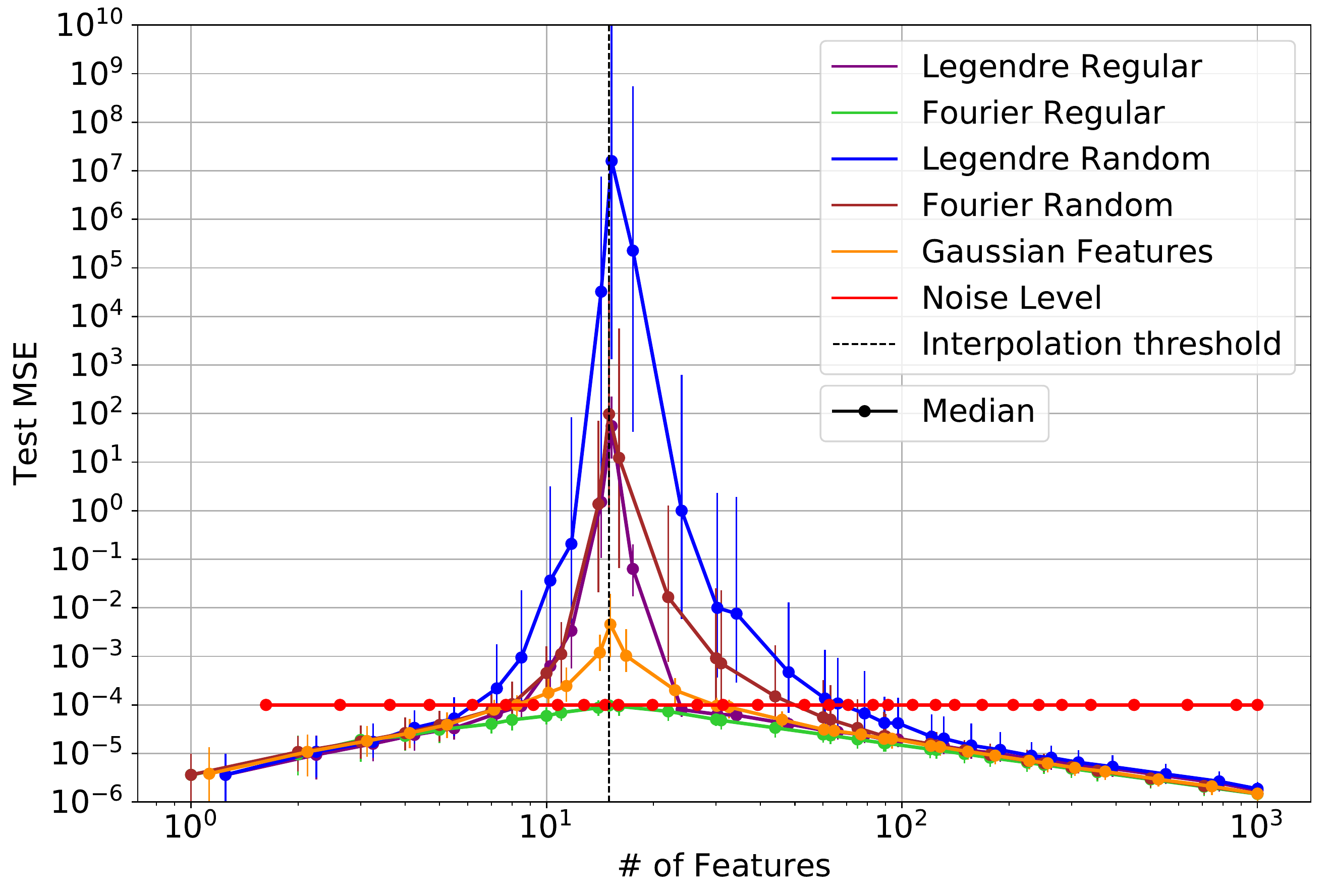}
}

From Corollaries~\ref{cor:fundamentalprice} and~\ref{cor:crabpot}, we saw that the fundamental price of interpolation of noise on test MSE scales as $\Theta\left(\frac{n}{d}\sigma^2\right)$, where $n$ is the number of training samples, $d$ is the number of features, and $\sigma^2$ is the noise variance. 
Looking at Figure~\ref{fig:fourier_paradigmatic_all}, we see that when $d$ is large, this is also the scaling that is achieved by the case of regularly spaced data points with Fourier features. 
This case, as an easy-to-understand paradigmatic example, provides a clear lens into understanding what is happening conceptually \textit{for $\ell_2$-minimizing solutions}\footnote{What we will use is the exact presence of exactly \textit{aliased} higher-frequency features corresponding to any low-frequency features.
This will give extremely clean expressions for the adverse effect that $\ell_2$-minimization has on recovering a signal with low-frequency components.
It is not as simple to understand the effect of sparsity-seeking methods, which require more generic properties like restricted isometry to work (although such properties have also been established for such sub-sampled Fourier matrices under nonuniform random sampling~\cite{rauhut2010compressive}).
}.
It is first useful see how the minimum-$\ell_2$-norm interpolator actually behaves for two contrasting examples.

\begin{example}[Standard Gaussian features, $k = 500$-sparse signal]\label{eg:standardgaussian}
We consider $d$-dimensional iid standard Gaussian features, i.e. Example~\ref{eg:gaussian} with $\Sigmabold = \mathbf{I}_d$.
In other words, the features $\{a(\Xvec)_j\}_{j=1}^d$ are iid and distributed as $\NORMAL(0,1)$.
Let the first $500$ entries of the true signal $\alphastar$ be non-zero and the rest be zero, i.e. $\supp(\alphastar) = [500]$.
We take $n=5000$ measurements, each of which is corrupted by Gaussian noise of variance $\sigma^2 = 0.01$. 
\end{example}

Figure~\ref{fig:RG_sparse_MSE} shows the test MSE of the minimum-$\ell_2$-norm interpolator on Example~\ref{eg:standardgaussian} as a function of the number of features $d$ for various choices of noise variance $\sigma^2$.
We immediately notice that the test MSE is converging to the same level as the test error from simply using a hypothetical $\alphabold = \mathbf{0}$.
This property of generalization error degenerating to that of simply predicting $0$ was also pointed out in \cite{hastie2019surprises} where this level was called out as the ``null risk.''

\begin{example}[Standard Gaussian features plus a constant, constant signal]\label{eg:wiggly}
In this example, we consider $d$-dimensional iid Gaussian features with unit mean and variance equal to $0.01$.
More precisely, the features $\{a(\Xvec)_j\}_{j=1}^d$ are iid and distributed as $\NORMAL(1,0.01)$.
We also assume the generative model for the training data:
\begin{align}\label{eq:wigglemodel}
Y = 1 + W
\end{align}
where as before, $W \sim \NORMAL(0, \sigma^2)$ is the observation noise in the training data, and we pick $\sigma^2 = 0.01$.
Note that in this example the true ``signal" is the constant $1$, which is not \textbf{exactly} expressible as a linear combination of the $d$ Gaussian features.
We take $n = 10$ noisy measurements of this signal.
\end{example}

\begin{figure}[htbp]
  \centering
  \includegraphics[width=0.6\textwidth]{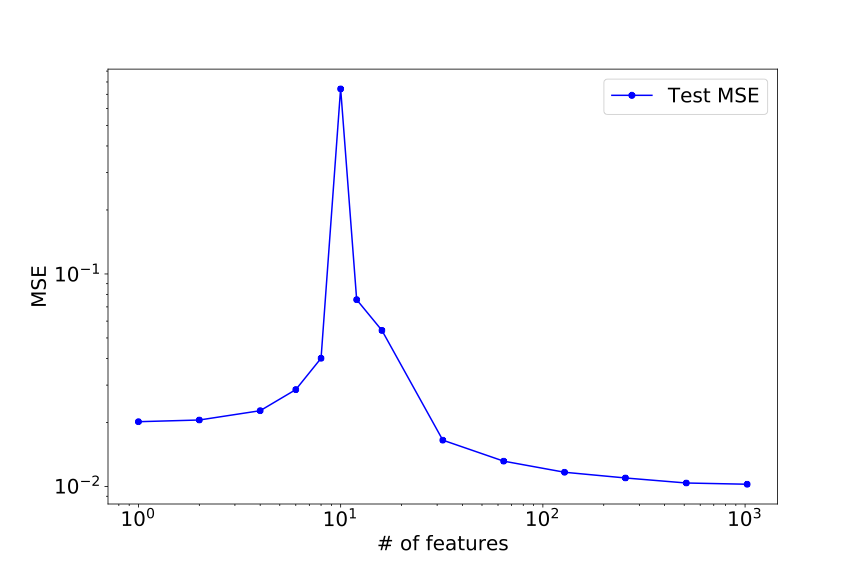}
  \caption{Log-log plot for test MSE for the min 2-norm interpolator (ordinary-least-squares to the left of the peak) vs the number of features for i.i.d. Gaussian features $\sim \mathcal N (1,0.01)$. Notice the clear double descent behavior when $d > n$. Here $n=10$ and the true signal is the constant $1$. This is Example~\ref{eg:wiggly}.} 
  \label{fig:wiggly}
\end{figure}

Example~\ref{eg:wiggly} is directly inspired by feature models\footnote{A more general version of this example would replace the constant $1$ in the means of the features by the relevant realizations of an underlying latent Gaussian random feature vector with the true signal being that latent Gaussian feature vector. 
The qualitative behavior of double-descent will be retained, and we defer a formal discussion of this to a forthcoming paper.
}  in several recent papers~\cite{belkin2019two, belkin2018reconciling,bartlett2019benign}, as well as having philosophical connections to other recent papers~\cite{mei2019generalization, hastie2019surprises}. 
Figure~\ref{fig:wiggly} shows the performance of the minimum-$\ell_2$-norm interpolator on Example~\ref{eg:wiggly} as we increase the number of features. 
Here, we clearly see the double-descent behavior as the test MSE of the minimum-$\ell_2$-norm interpolator decreases with increased overparameterization. 
Note from Equation~\eqref{eq:wigglemodel} that the true signal is not exactly representable by a linear combination of the random Gaussian features, and in fact Example~\ref{eg:wiggly} is an instance of the linear model being misspecified for any (finite) number of features $d$.
Improved approximability from adding more features to the model partially explains the double descent behavior, but it is not the main reason.
It turns out that we would still see the double-descent behavior with the minimum-$\ell_2$-norm interpolator if we added another feature that was always the constant $1$ (in which case the example would fall under the $(1,\sigma^2)$-sparse linear model; see Definition~\ref{def:sparse}.)

The minimum-$\ell_2$-norm interpolator generalizes well for Example~\ref{eg:wiggly}, showing double descent behavior -- but extremely poorly for Example~\ref{eg:standardgaussian}.
Why does one case fail while the other case works?  
Bartlett, Long, Lugosi and Tsigler~\cite{bartlett2019benign} give an account of what is happening directly using the language of random matrix theory.
Their paper defines a distinct pair of ``effective ranks" using the spectrum of the covariance matrix $\Sigmabold$ to state their necessary and sufficient conditions for this interpolator generalizing well.
Classic core concepts in signal processing provide an alternative lens to understand these conditions. 
To use the Fourier-theoretic lens, we will naturally map the number of \textit{regularly spaced} training samples $n$ to what is called the \textit{sampling rate} in signal processing, and the number of features $d$ to what is called the \textit{bandwidth} in signal processing. 
What we call overparameterization is what is called potential \textit{undersampling} in the signal processing literature.



\subsection{Aliasing --- the core issue in overparameterized models}

We have mapped overparameterization to undersampling of a true signal.
The fundamental issue with undersampling of a signal is one of identifiability: infinitely many solutions, each of which correspond to different signal functions, all happen to agree with each other on the $n$ regularly spaced data points.
These different signal functions, of course, disagree everywhere else on the function domain, so the true signal function is not truly reconstructed by most of them.
This results in increased test MSE when such an incorrect function is used for prediction.
Such functions that are different, but agree on the sampled points, are commonly called \textit{aliases} of each other in signal processing language.
Exact aliases naturally appear when the features are Fourier, as we see in the below example.
\begin{example}[Fourier features with a constant signal]\label{eg:fourieralias}
Denote $i := \sqrt{-1}$ as the imaginary number.
Consider the Fourier features as defined in complex form in Example~\ref{eg:fourier} and regularly spaced input on the interval $[0,1)$, i.e. $x_j = \frac{j-1}{n}$ for all $j \in [n]$.

Suppose the true signal is equal to $1$ everywhere and the sampling model in the absence of noise is 
	\begin{align}\label{eq:constantsignal}
	Y_j = 1  \text{ for all } j \in [n] ,
	\end{align}

The estimator has to interpolate this data with some linear combination of Fourier features\footnote{Why are we using complex features for our example instead of the real sines and cosines? Just because keeping track of which feature is an alias of which other feature is less notationally heavy for the complex case. The essential behavior would be identical if we just considered sines and cosines.} $f_k(x) = e^{i 2 \pi k x}$ for $k = 0 \ldots d$. 

A trivial signal function that agrees with Equation~\eqref{eq:constantsignal} at all the data points is the first (constant) Fourier feature: $f_0(x) = e^{i 2 \pi (0) x} = 1$.
It is, however, not the only one.
The complex feature $f_n(x) = e^{i 2 \pi (n) x}$ will agree with $f_0(x) = 1$ on \textit{all} the regularly spaced points $\{x_j\}_{j=1}^n$ by the cyclic property of complex Fourier features (i.e. we have $e^{i 2\pi (n) \frac{j-1}{n}} = e^{i 2 \pi (j-1) } = 1 = f_0(x)$).
This is similarly true for features $f_{ln}(x)$ for all $l = 1, 2, \ldots, \frac{d}{n} - 1$, and we thus have $\frac{d}{n}-1$ \textbf{exact aliases}\footnote{These aliases are essentially higher frequency sinusoids that look the same as the low frequency one when regularly sampled at the rate $n$. This is the classic ``movie of a fan under a strobelight'' visualization where a fan looks like it is stationary instead of moving at a fast speed!} of the true signal function $f_0(x)$ on the regularly spaced data points.
\end{example}

The above property is not unique to the constant function $f_0(x)$: for any true signal function that contains the complex sinusoid of frequency $k^* \in [n]$, i.e. $f_{k^*}(x) = e^{i 2 \pi (k^*) x}$, the one-complete-cycle signal function $f_{k^* + n} = e^{i 2\pi (k^* + n) x}$ again agrees on the $n$ regularly spaced data points, and for this signal function we again have the $\frac{d}{n}$ exact aliases $f_{k^* + ln}$ for all $l = 1, 2, \ldots, \frac{d}{n} - 1$.

The presence of these aliases will naturally affect signal reconstruction.
Before discussing this issue, we show an \textit{advantage} in having aliases: they naturally \textit{absorb} the noise that can harm generalization.
Critically, as defined in Example~\ref{eg:fourier}, the Fourier features are orthonormal to each other in complex function space (where the integral that defines the complex inner product is taken with respect to the uniform measure on $[0,1)$).
If we used only the first $n$ Fourier features (i.e. $f_k(x)$ upto $k = n - 1$) to fit an $n$-dimensional pure noise vector (as described by $\{Y_j = W_j\}_{j=1}^n$), the coefficients of the $n$-dimensional fit, i.e. $\{\widehat{\alpha}_k\}_{k=1}^n$, would directly correspond to the appropriate discrete Fourier 
transform (DFT)\footnote{The convention for defining the DFT depends on the chosen normalization. The symmetric/unitary DFT can be viewed as choosing the orthonormal basis vectors given by $\frac{1}{\sqrt{n}}e^{2\pi (k) ix}$ evaluated at $n$ regularly spaced points from $[0,1)$. The classic DFT is defined by a basis with a different scaling --- namely $\frac{1}{n}$ instead of $\frac{1}{\sqrt{n}}$. This results in the classic DFT having a factor of $\frac{1}{n}$ in the inverse DFT. We choose the convention for the DFT which normalizes \textit{in the opposite direction}. The basis vectors are just the un-normalized $e^{2\pi (k) ix}$ evaluated at $n$ regularly spaced points from $[0,1)$. We do not want any scaling factors in the relevant inverse DFT because we want to get the coefficients of the Fourier features.} of the $n$-dimensional noise vector.
By the appropriate\footnote{The reader can verify that the normalization convention we have chosen for the DFT implies $\vecnorm{\alphahat}{2}^2 = \frac{1}{n} \vecnorm{\Wtrain}{2}^2$.}  Parseval's relation in signal processing, the expected total energy in the feature domain, i.e. $\EE\left[\vecnorm{\alphahat}{2}^2\right]$ would be $\sigma^2$, and moreover (due to the isotropy of white/independent Gaussian noise), this energy would be equally spread across all $n$ Fourier features in expectation.
That is, for every $k \in [n]$ we would have $\EE\left[|\widehat{\alpha}_k|^2\right] = \frac{\sigma^2}{n}$. 

Now, consider what happens when we include the $\frac{d}{n} -1$ higher-frequency aliases corresponding to each lower frequency component $k \in [n]$.
This gives us $d$ Fourier features in total, and we now consider the minimum-$\ell_2$-norm interpolator of noise using all $d$ features.
The following is what will happen:
\begin{enumerate}
\item In an effort to minimize $\ell_2$-norm, the coefficient (absolute) values will be equally \textit{divided}\footnote{The reader who is familiar with wireless communication theory will be able to relate this to the concept of coherent combining gain.} among the $\frac{d}{n}$ aliased features for every realization of the noise samples, i.e. $|\widehat{\alpha}_{k + ln}| = |\widehat{\alpha}_k| \text{ for all } l \in \{1,2,\ldots,\frac{d}{n} - 1\}$ and for all $k \in [n]$.
\item For each $k \in [n]$, the expected \textit{total} contribution from the low frequency feature $k$ and its aliases is now reduced to $\left(\frac{1}{\left(\frac{d}{n}\right)}\right) \cdot \frac{\sigma^2}{n} = \frac{\sigma^2}{d}$.
This results in total $\EE\left[\vecnorm{\alphahat}{2}^2\right] = \frac{n}{d} \sigma^2$.
\end{enumerate}

For this case, we have zero signal and thus the test MSE for the (whitened) Fourier features is exactly $\vecnorm{\alphahat}{2}^2$.
Thus, we have \textit{exactly} recovered the $\Theta\left(\sigma^2\frac{n}{d}\right)$ scaling for the ideal MSE that we \textit{bounded} in Corollaries~\ref{cor:fundamentalprice} and~\ref{cor:crabpot}.
The aliases, when used with minimum-$\ell_2$-interpolation of noise, are literally dissipating noise energy, thus directly reducing its potentially harmful effect on generalization in the average\footnote{It is also clear that the average case might be very different than the worst case --- a phenomenon intimately connected to the fundamental issue of adversarial examples on neural networks that empirically generalize well.} case.

\subsubsection{Avoiding signal ``bleed"}

\multifigureexterior{fig:signal_bleed}{Depiction of signal components ``bleeding" out into spurious features as a result of using the minimum-$\ell_2$-norm interpolator. The ``bleeding" has two effects: lower ``survival" of signal in the original true features, and higher ``contamination" by spurious features.}{
\subfigl{0.5\textwidth}{fig:fourier_bleed}{Illustration of the ``bleed".}{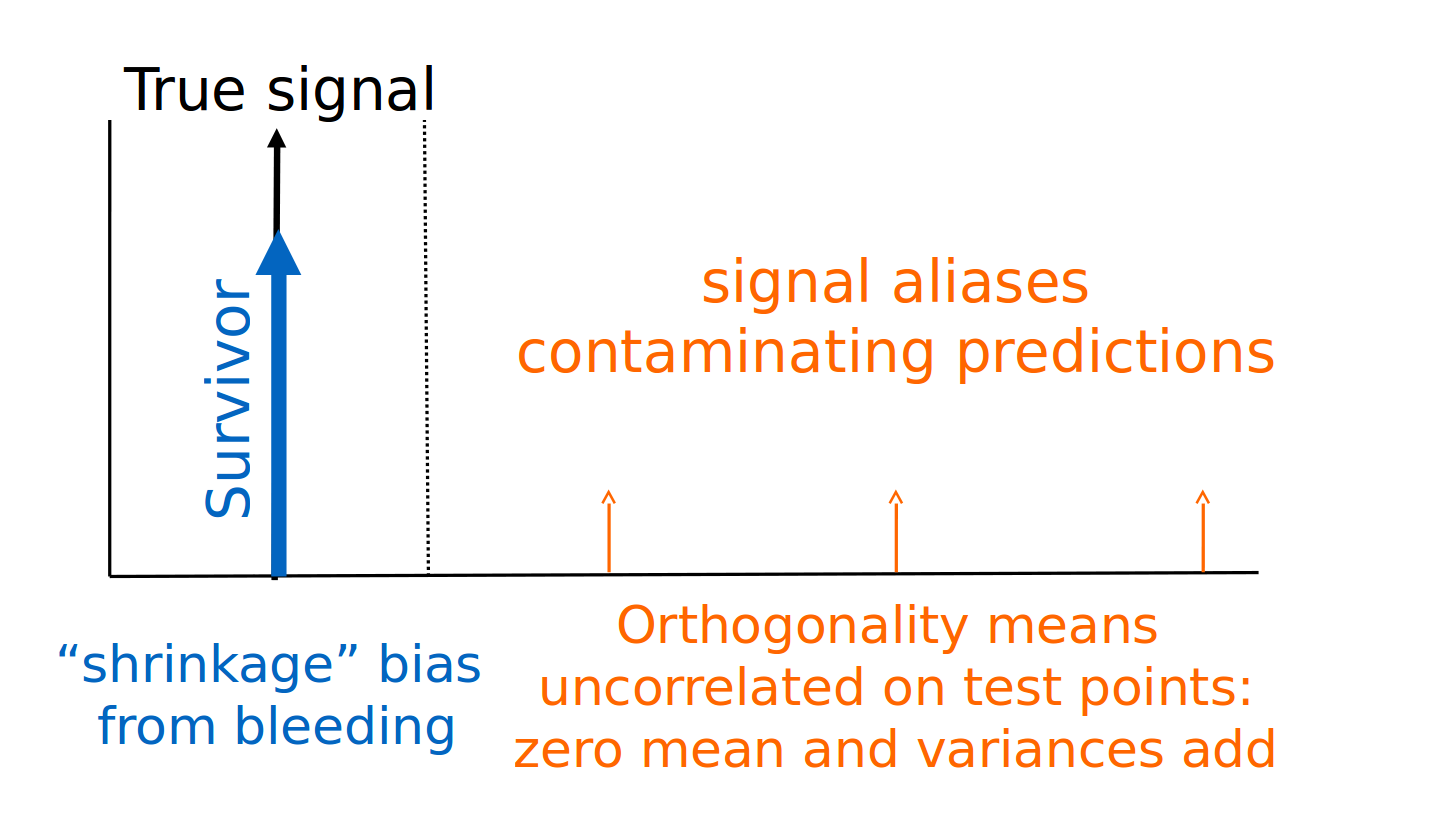}
\subfigl{0.4\textwidth}{fig:gaussian_bleed}{Plot of estimated signal components of minimum-$\ell_2$-interpolator for iid Gaussian features. Here, $n = 5000$, $d = 30000$ and the true signal $\alphastar$ has non-zero entries only in the first $500$ features.}{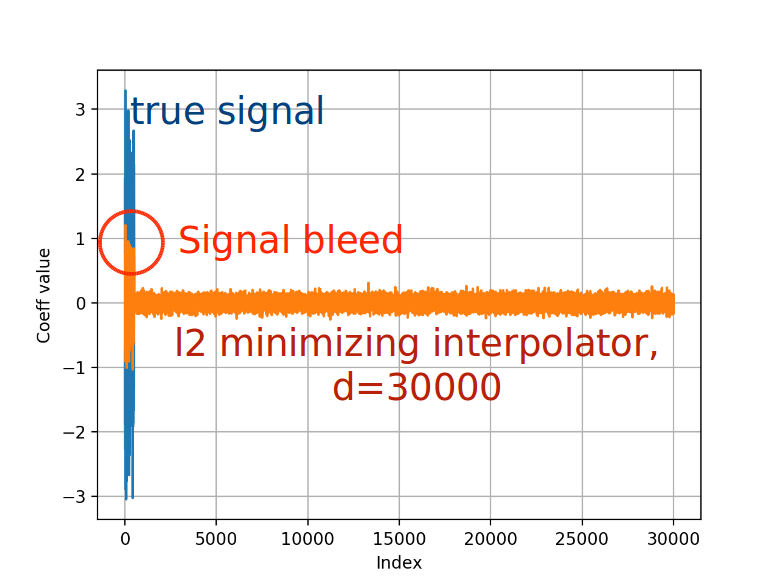}
} 
The problem with $\ell_2$-minimizing interpolation is that the above effect of absorbing and dissipating training label energy is generic, whether those labels are signal or noise.
Whereas being able to absorb and dissipate training harmful noise energy is a good thing, the same exact effect is harmful when there is true signal.
Suppose, as in Equation~\eqref{eq:constantsignal}, that the true signal was a constant (thus the only true frequency component is $k = 0$). 
A simple calculation shows that the estimated coefficient of the true function would also be attenuated in exactly the same way, and the absolute value of the coefficient corresponding to the constant feature (i.e. $|\widehat{\alpha}_0| = \frac{n}{d}$) decays to $0$ as $d \to \infty$.
True signal energy, which should ideally be concentrated in the constant feature, is literally \textit{bleeding} into its aliases in the inference by the $\ell_2$-norm-minimizing interpolating solution. 
This is what we are seeing in the scaling of the test MSE of the $\ell_2$-minimizing interpolator for iid Gaussian features (Example~\ref{eg:standardgaussian}, shown in Figure~\ref{fig:RG_sparse_MSE}) as well as the convergence of the test MSE of the minimum-$\ell_2$-norm interpolator to the ``null risk" proved by Hastie, Tibshirani, Rosset and Montanari~\cite{hastie2019surprises}.
An illustration of this bleeding effect is provided in Figure~\ref{fig:fourier_bleed}, and the realization of this effect on actually recovered signal components for Example~\ref{eg:standardgaussian} is shown in Figure~\ref{fig:gaussian_bleed}.

The asymptotic bleeding of signal is a generic issue for whitened features more generally (see~\cite[Lemma $2$]{hastie2019surprises}).
This may seem to paint a hopeless picture for the $\ell_2$-minimizing interpolator even in the absence of noise -- how, then, can it ever work?
The key is that we can rescale, and mix, the underlying whitened features to give rise to a \textit{transformed feature family}, with respect to which we seek an interpolating solution that minimizes the $\ell_2$-norm of the coefficients \textit{corresponding to these transformed features.}
The test MSE of such an interpolator will of course be different from the minimum-$\ell_2$-norm interpolator using whitened features.
The effective difference arises only through the effective rescaling of the whitened features through this transformation: the manifestation of the rescaling can be explicit (the features $\{a(\Xvec)_j\}_{j=1}^d$ can be visibly scaled by weights $w_j := \sqrt{\lambda_j}$ for some $\lambda_j > 0$) or implicit (the eigenvalues of the covariance matrix\footnote{Bartlett, Long, Lugosi and Tsigler~\cite{bartlett2019benign} present their results through this implicit viewpoint, but their analysis essentially reduces to the explicit viewpoint after a transformation in the underlying geometry.} $\Sigmabold$ corresponding to the transformed features $a(\Xvec)_j$ correspond to the squared weights $\lambda_j$).

Consider Example~\ref{eg:fourieralias} and rescaling $w_k > 0$ for Fourier feature $f_k$.
For a set of coefficients $\alphabold \in \reals^d$ corresponding to the original whitened features $\{f_k\}_{k=0}^{d-1}$, we denote the corresponding coefficients for the rescaled features by $\betabold \in \reals^d$.
Then, the minimum-$\ell_2$-norm interpolator with respect to the rescaled feature family is as below, for any output $\{Y_j\}_{j=1}^n$:
\begin{align*}
\betahat &:= {\arg \min} \vecnorm{\betabold}{2} \text{ subject to } \\
\sum_{k=0}^{d-1} \beta_k w_k f_k(x_j) &= Y_j \text{ for all } j \in [n] .
\end{align*}

This would recover equivalent coefficients $\alphahat$ for the minimum-\textit{weighted}-$\ell_2$-norm interpolator of the data with weight $\frac{1}{w_k}$ corresponding to feature $k$, as below:
\begin{align}\label{eq:hilbertnorm}
\alphahat &:= {\arg \min} \sum_{k=0}^{d-1} \frac{\alpha_k^2}{w_k^2} \text{ subject to } \\
\sum_{k=0}^{d-1} \alpha_k f_k(x_j) &= Y_j \text{ for all } j \in [n] .
\end{align}

Now, consider the case of a constant signal without noise, i.e. $Y_j = 1$ for all $j \in [n]$. 
We already saw that the true signal function, which is $f_0(x)$, satisfies $f(x_j) = 1$ for all $j \in [n]$, as does each of its $\left(\frac{d}{n} - 1\right)$ aliases $\{f_{ln}(x)\}$ for $l = 1,1,\ldots, \frac{d}{n} - 1$.
Thus, the coefficients of $\alphahat$ will be the unique linear combination of the aliases, with coefficients represented by $\{\widehat{\alpha}_{ln}\}$, that minimizes the \textit{re-weighted} $\ell_2$-norm subject to the sum of such coefficients being exactly equal to $1$ (to satisfy the interpolation constraint).
The special case of whitened features corresponds to $w_k = 1$ for all $k \in [d]$, and this intuitively results all aliases contributing equally to the recovered signal function.
What happens with non-uniform weights: in particular, what happens when $w_k$ decreases as a function of frequency $k$?
Intuitively, the weighted-$\ell_2$-norm objective implies that higher-frequency aliases are \textit{penalized more}, and thus a solution would favor smaller coefficients $\widehat{\alpha}_{ln}$ for higher integral values of $l$.
In fact, Appendix~\ref{app:fourier} shows by the principle of matched filtering that the $\ell_2$-minimizing coefficients are precisely

\begin{align}\label{eq:l2minimizingbeta}
\widehat{\alpha}_{ln} = \frac{w^2_{ln}}{V} \text{ where } V := \sum_{l=0}^{d/n - 1} w_{ln}^2 \text{ for all } l = 0,1,\ldots,\frac{d}{n} - 1 .
\end{align}

Since the true \textit{constant} signal is represented by coefficients $\alpha^*_0 = 1$ and zero everywhere else, we are particularly interested in the absolute value of $\widehat{\alpha}_0$: how much of the \textit{true signal component} have we preserved?
Then, the simple explicit calculation in Appendix~\ref{app:fourier} shows that this ``survival factor" is essentially\footnote{
This survival factor can also be understood as the outcome of a competition between two functions. The true signal $f_0$ that has squared weight $w_0^2$, and the most attractive orthogonal alias whose squared weight is $\sum_{l=1}^{\frac{d}{n}} w_{ln}^2$. The minimum 2-norm interpolator will pick a convex combination of the two by minimizing $\frac{\gamma^2}{w_{0}^2} + \frac{(1-\gamma)^2}{\sum_{l=1}^{\frac{d}{n}} w_{ln}^2}$ where $\gamma$ is the survival factor of the true feature. This is minimized by the answer given here for $\gamma = \mathsf{SU}$.}
\begin{align}\label{eq:survival}
\mathsf{SU} := \frac{\widehat{\alpha}_0}{\alpha^*_0} = \frac{w_0^2}{\sum_{l=0}^{\frac{d}{n}-1} w_{ln}^2} .
\end{align}

The \textit{inverse} of the survival factor $\mathsf{SU}$, after substituting $\lambda_k := w_k^2$, is very closely related to the first ``effective rank" condition introduced by Bartlett, Long, Lugosi and Tsigler to characterize the \textit{bias} of the minimum-$\ell_2$-norm interpolator in \cite{bartlett2019benign}.
Clearly, the survival factor intimately depends on the {\em relative weights} placed on different frequencies, how many frequencies there are in consideration, and how many perfect aliases there are (the number of aliases is inversely proportional to the number of training samples $n$).  It is illustrative to rewrite the survival factor $\mathsf{SU}$ as 
\begin{align} \label{eq:survivalBode}
\mathsf{SU} := \frac{1}{1 + \frac{\sum_{l=1}^{\frac{d}{n}-1} w_{ln}^2}{w_0^2}}.
\end{align}

Equation~\eqref{eq:survivalBode} is in a form reminiscent of the classic signal-processing ``one-pole-filter transfer function''. 
What matters is the relative weight of the favored feature $w_0$ to the combined weight of its competing aliases. 
As long as it is relatively high, i.e. $w_0^2 \gg \sum_{l=1}^{\frac{d}{n}-1} w_{ln}^2$, the true signal will survive. 
So in particular, if the weights are such that the sum $\sum_{l=1}^{\frac{d}{n}-1} w_{ln}^2$ converges even as the number of features grows, the true signal will at least partially survive even as $\frac{d}{n} \to \infty$. 
Meanwhile, if the sum $\sum_{l=1}^{\frac{d}{n}-1} w_{ln}^2$ diverges {\bf and does so faster than $w_0^2$}, the signal energy will completely \textit{bleed} out into the aliases (as happens for the whitened case $w_k = 1$ for all $k$).

This need for the relative weight on the true features to be high enough relative to their aliases is something that must hold true before any training data has even been collected.  In other words, the ability of the 2-norm minimizing interpolator to recover signal is fundamentally restricted. There needs to be a low-dimensional subspace (low frequency signals in our example) that is heavily favored in the weighting, and moreover the true signal needs to be well represented by this subspace. The weights essentially encode an \textit{explicit strong prior}\footnote{This is in stark contrast to feature selection operators like the Lasso, which select features in a \textit{data-dependent} manner.} that favors low-frequency features.


We can now start to understand the discrepancy between Examples~\ref{eg:standardgaussian} and~\ref{eg:wiggly}.
There is no prior effect favoring in any way the first $500$ features for Example~\ref{eg:standardgaussian}.
However, by their very nature the features used in Example~\ref{eg:wiggly} heavily (implicitly, when the eigenvalue decomposition of $\Sigmabold$ is considered\footnote{In fact, this very case is evaluated in~\cite[Corollary $1$]{hastie2019surprises}.}) favor the constant feature that best explains the data. 
This is because the maximal eigenvector of $\Sigma$ is a ``virtual feature" that is an average of the $d$ explicit features, i.e. its entries are iid $\NORMAL(1, \frac{0.01}{d})$.
This better and better approximates the \textit{constant} feature, the true signal, as $d$ increases -- and this improved approximability is the primary explanation for the double descent behavior observed in Figure~\ref{fig:wiggly}.

\multifigureexterior{fig:Spiked_l2}{Effect of different priors on weighted $\ell_2$ norm interpolation with $n=500, d = 11000$ when the true signal is the sign function.  The learned function approximates the true signal well when we have a strong prior on the low frequency features but suffers from signal bleed as weaken the prior.}{
\subfig{0.4\textwidth}{Strong prior- feature weights}{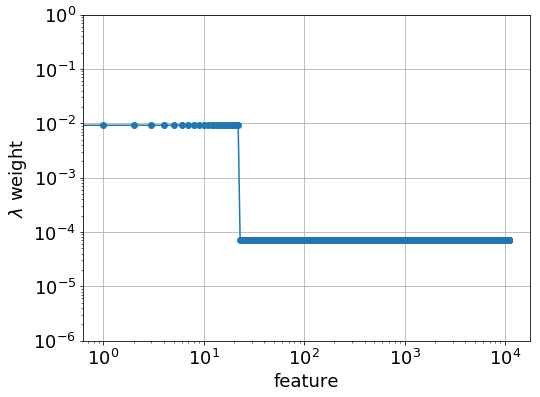}
\subfig{0.55\textwidth}{Strong prior- learned function}{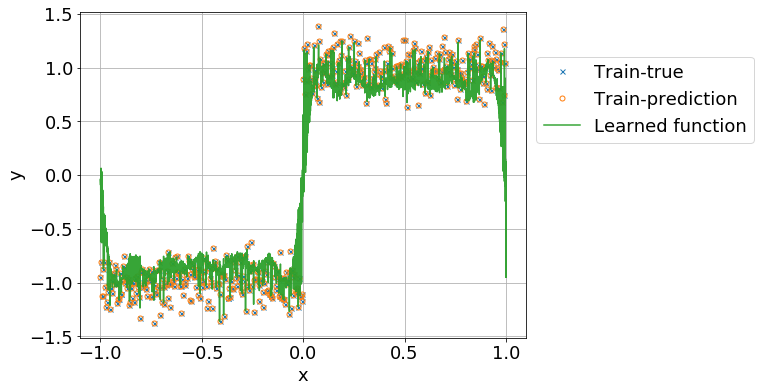}

\subfig{0.4\textwidth}{Medium prior- feature weights}{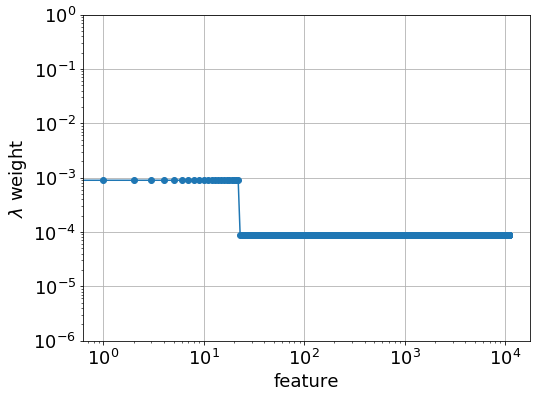}
\subfig{0.55\textwidth}{Medium prior- learned function}{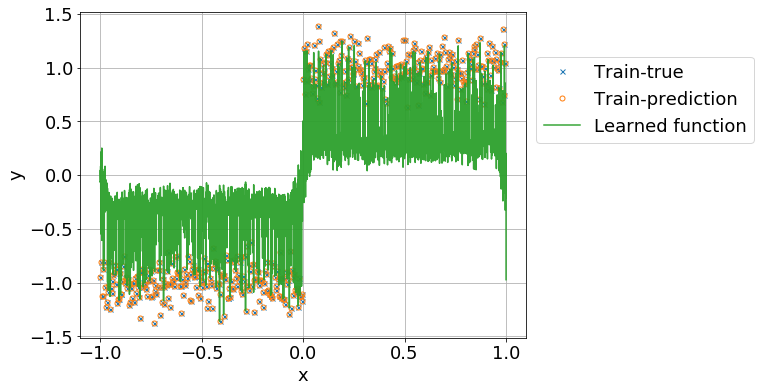}

\subfig{0.4\textwidth}{Weak prior- feature weights}{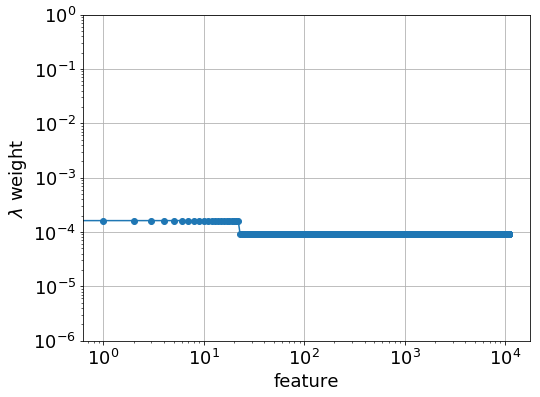}
\subfig{0.55\textwidth}{Weak prior - learned function }{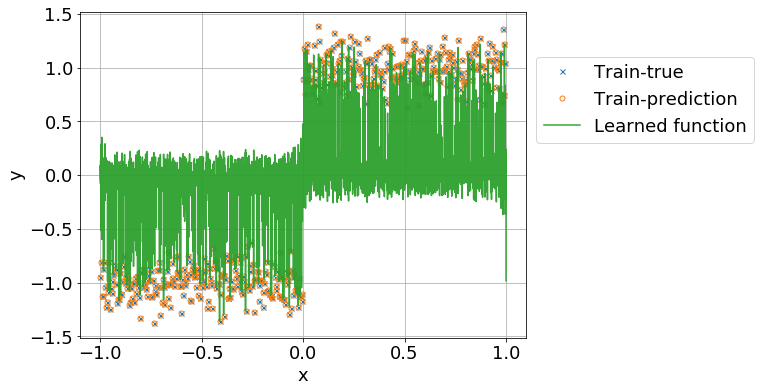}
} 

In Figure~\ref{fig:Spiked_l2}, we illustrate how changing the level of the prior weights impacts interpolative solutions using Fourier features for the simple case of a sign function. Here, there is noise in the training data, but the results would look similar even if there were no training noise --- the prior weights are primarily fighting the tendency of the interpolator to bleed signal.

\subsubsection{Avoiding signal contamination}



We have seen that a sufficiently strong prior in a low-dimensional subspace of features avoids the problem of asymptotically bleeding too much of the signal away --- as long as the true signal is largely within that subspace. But what happens when some of the true signal is bled away? How does this impact prediction beyond shrinking the true coefficients? Furthermore, the issue of signal bleed does not by itself answer the question of consistency, particularly with the additional presence of noise. How does the strong prior affect fitting of noise -- is it still effectively absorbed by the aliases, as we saw when the features were whitened?
sTo properly understand this, we need to introduce the idea of ``signal contamination.''


Consider Example~\ref{eg:fourieralias} now with the \textit{constant-signal-plus-noise} generative model for data:
\begin{align} \label{eq:noisesample}
Y_j = 1 + W_j \text{ for all } j \in [n] .
\end{align}

The output energy (signal as well as noise) bleeds away from the true signal component corresponding to Fourier feature $0$ -- but because we are exactly interpolating the output data, the energy has to go somewhere.
As a result, all energy that is bled from the true feature will go into the aliased features $\{f_{ln}\}_{l=1}^{d/n - 1}$.
Each of these features contributes uncorrelated zero-mean unit-variance errors on a test point, scaled by the recovered coefficients $\{\widehat{\alpha}_{ln}\}$. 
Because they are uncorrelated, their variances add and we can thus define the contamination factor
\begin{align*}
C := \sqrt{\sum_{l=1}^{d/n - 1} \widehat{\alpha}_{ln}^2} .
\end{align*}

Even if there were no noise, the test MSE would be at least $C^2$. Consequently, it is important to verify that $C \to 0$ as $(d,n) \to \infty$.


A straightforward calculation (details in Appendix~\ref{app:fourier}), again through matched-filtering, reveals that the absolute value of the coefficient on aliased feature $ln$ is directly proportional to the weight $w_{ln}$ and the original true signal strength. Thus contamination (measured as the standard-deviation, rather than the variance in order to have common units), like signal survival, is actually a factor
\begin{align}\label{eq:contaminationfinal}
C = \frac{\sqrt{\sum_{l=1}^{d/n - 1} w_{ln}^2}}{w_0^2 + \sum_{l=1}^{d/n - 1} w_{ln}^2}
\end{align}
for the minimum-weighted-$\ell_2$-norm interpolator corresponding to weights $\{w_k\}_{k=1}^d$.
Substituting $\lambda_k := w_k^2$ results in an error scaling that is very reminiscent of Bartlett, Long, Lugosi and Tsigler's second effective-rank condition.
Thus, we see that the two notions of effective ranks\cite{bartlett2019benign} correspond to these factors of survival and contamination, which Bartlett, Long, Lugosi and Tsigler sharply characterize for Gaussian features using random matrix theory.
The effective ``low-frequency features" there represent directions corresponding to the dominant eigenvalues of the covariance matrix $\Sigmabold$.

There is a tradeoff: while we saw earlier that the weight distribution described here should somewhat favor low-frequency features (dominant eigenvalue directions), it cannot put too \textit{little weight} on higher-frequency features either. 
If that happens, the bleeding prevention conditions can be met for a true signal that is in the appropriate low-dimensional subspace. 
But the noise will give rise to nonvanishing contamination and the variance of the prediction error will not go to $0$ as $(n,d) \to \infty$ -- then, the minimum-$\ell_2$-norm interpolator is inconsistent.
Equation~\eqref{eq:contaminationfinal} tells us that the contamination $C$ will be sufficiently small to ensure consistency \textit{iff}:
\begin{enumerate}
\item The weights $\{w_{ln}\}_{l \geq 1}$ decay \textit{slowly enough} so that the sum of squared alias-weights $\sum_{l=1}^{d/n-1} w_{ln}^2$ \textit{diverges}.
This means that there is sufficient \textit{effective} overparameterization to ensure harmless noise fitting.
\item If the sum of squared alias-weights $\sum_{l=1}^{d/n-1} w_{ln}^2$ does not diverge, the term $w_0^2$ must dominate this sum in the dominator. 
Then, we also need $w_0^2 \gg \sqrt{\sum_{l=1}^{d/n-1} w_{ln}^2}$ so that the denominator dominates the numerator.
\end{enumerate}
Clearly, avoiding non-zero contamination is its own condition, which is \textit{not} directly implied by avoiding bleeding. 

To get consistency, it must be the case that the contamination goes to zero with increasing $n,d$ for everywhere that has true signal as well as an asymptotically complete fraction of the other frequencies. 
If contamination doesn't go to zero where the signal is, the test predictions will experience a kind of non-vanishing self-interference from the true signal. If it doesn't go to zero for most of where the noise is, then that noise in the training samples will still manifest as variance in predictions. 

It is instructive to ask whether the above tradeoff in maximizing signal ``survival'' and minimizing signal ``contamination'' manifests as a clean bias-variance tradeoff~\cite{belkin2018reconciling}.
The issue is that the contamination can arise through signal and/or noise energy.
The fraction of contamination that comes from true signal is mathematically a kind of variance that behaves like traditional bias --- it is an approximation error that the inference algorithm makes even when there is no noise. 
The fraction of contamination that comes from noise is indeed a kind of variance that behaves like traditional variance --- it would disappear if there were no noise in training data. 

\subsubsection{A filtering perspective on interpolation}

\multifigureexterior{fig:strong_prior}{Weighted $\ell_2$ norm interpolation for regularly spaced Fourier features with a strong prior on low frequency features. }{
\subfigl{0.45\textwidth}{fig:strong_prior_kernel}{ Pulse shaping kernel for $n = 50, d = 354$.}{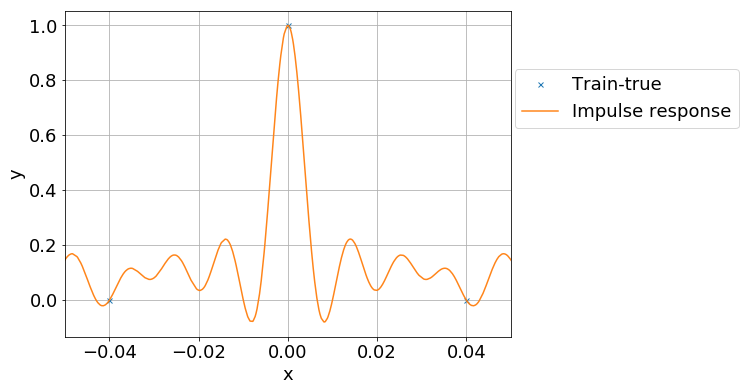}
\subfigl{0.45\textwidth}{fig:strong_prior_survival_contamination}{Flter view on the ``survival'' and contamination for $n = 500, d = 11000$.}{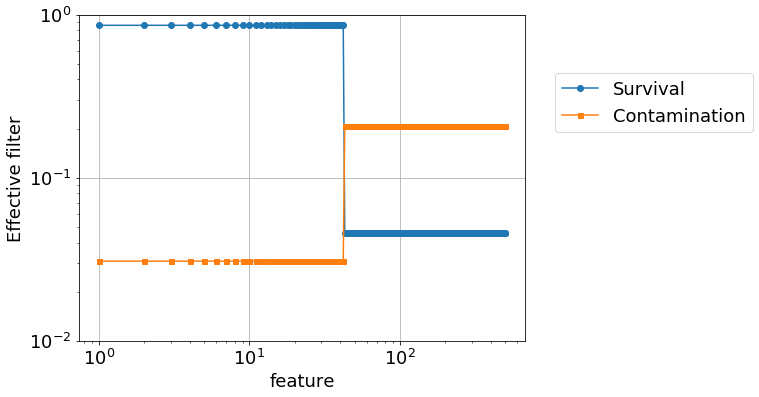}
} 

Returning to the case of Fourier features with regularly spaced training points, we can see that given the weightings $w_i$ on all the features, we can break the features into cohorts of perfect aliases. All the features are orthogonal (vis-a-vis the test distribution) and because of the regular sampling, each cohort is orthogonal to every other cohort even when restricted to the $n$ sample points. Consequently, we can understand the bleeding within each of the cohorts separately. Moreover, if we assume that the true signal is going to be low-frequency\footnote{This is just for simplicity of exposition and matching the standard machine learning default assumption that all things being equal, we prefer a smoother function to a less smooth function. If the weighting were different, then we could just as well redo this story looking at the highest-weight member of the alias cohort.}, then we can think about how much the lowest frequency representative of each cohort bleeds. This can be expressed in terms of the survival $0 \leq \mathsf{SU}(k^*) \leq 1$ for that low-frequency feature $k^*$ when using the weighted minimum 2-norm interpolator.  These $\{\mathsf{SU}(k^*)\}_{k=0}^{d-1}$ together can be viewed as a filter. This filter tells us how much the act of sampling and estimating attenuates each frequency in the true signal. This attenuation is clearly a kind of ``shrinkage.''

With the filtering perspective, we can immediately see that for the minimum-weighted-$\ell_2$-norm interpolator to succeed in recovering the signal, the true signal needs to be well approximated by the low-frequency features $\{k^*\}$ for which  $\mathsf{SU}(k^*) \approxeq 1$ -- otherwise the true pattern will be substantially bled away. 
We also see that to be able to substantially absorb/dissipate the noise energy (which is going to be spread equally across these $n$ cohorts by the isotropic property of white Gaussian noise), it must be the case that most of the survival coefficients $\{\mathsf{SU}(k^*)\}_{k=0}^{d-1}$ are quite small --- most of the noise energy needs to be bled away. 
As we tend $(n,d)$ to infinity, we can quantify the required conditions for consistency.
In the ``continuous time" setting, as $n$ is increasing, the continuous-time frequency (that corresponds to the ``fastest" feature) is growing with $n$. 
So, as long as the maximal value of this frequency $k^*$ for which the signal would survive (i.e. $\mathsf{SU}(k^*) \approxeq 1$) grows \textit{sub-linearly}\footnote{If this is reminiscent of the conditions discussed when one considers Nadaraya-Watson kernel estimation in nonparametric statistics, this is no coincidence as \cite{belkin2018does} points out clearly.} in $n$, {\bf and} the set of frequencies for which the signal would ``bleed out" (i.e. $\mathbf{SU}(k^*) \to 0$) is asymptotically $n - o(n)$ frequencies, there is hope of both recovering a low-frequency signal as well as absorbing noise. 

On one hand, if $\omega(n)$ of the $\mathsf{SU}$s stay boundedly above $0$, then those dimensions of the white noise will clearly not be attenuated as desired, and will show up in our test predictions as a classical kind of prediction variance that is not going to zero. 
On the other hand, if the true signal is not eventually expressible by low-frequency features whose ``survival" coefficients approach $1$, then there is asymptotically non-zero bias in the prediction. 

A further nice aspect of the filtering perspective is that it also lets us immediately see that since the relevant Moore-Penrose pseudo-inverse is a linear operator, we can also view it in ``time domain.'' 
In machine learning parlance, we could call this the ``kernel trick", by which the prediction rule has a direct (and in this case linear) dependence on the labels for the training points. 
In a traditional signal processing, or wireless communications, perspective, this arises from pulse-shaping filters, or interpolating kernels. A particular set of weights induces both a ``survival" filter and an explicit time-domain interpolation function. This is illustrated in Figure~\ref{fig:strong_prior} for a situation in which we put a substantial prior weight on the low-frequency features. Notice that the low-frequency features survive, and have very little contamination. Meanwhile, the higher-frequecies are attenuated, and though their energy is divided across even higher frequency aliases, the net contamination is also small. The time-domain interpolating kernel looks almost like a classical low-pass-filter, except that it passes through zero at the training point intervals to maintain strict interpolation.

\subsubsection{A short comment on Tikhonov regularization and $\ell_2$-minimizing interpolation}

Hastie, Tibshirani, Rosset and Montanari~\cite{hastie2019surprises} describe the test MSE of ridge-regularized solutions (with the optimal level of regularization, denoted by $\lambda$) as strictly better than that of the $\ell_2$-minimizing interpolator.
The discussion above of putting weights on features while we minimize the norm should clearly be evocative of Tikhonov regularization. It turns out that an elementary classical calculation is useful to help us understand what is happening in the new language of bleeding and contamination. 

Consider the classical perspective on Tikonov regularization, where here, we will use $\lambda^2 \Gamma$ as the Tikhonov regularizing matrix to separately call out the ridge-like part $\lambda^2 > 0$ which controls the overall strength of the regularizer and the non-uniformity of the feature weighting that $\Gamma$ represents. As is conventional, let us assume that $\Gamma$ is positive-definite and has a $d\times d$ invertible square-root $\Gamma^{\frac{1}{2}}$ so that $(\Gamma^{\frac{1}{2}})^\top \Gamma^{\frac{1}{2}} = \Gamma$. Then, we know that 
\begin{align}
 &= \argmin_{\alphabold} \|\Atrain \alphabold - \Ytrain\|^2 + \lambda^2 \alphabold^\top \Gamma \alphabold \\
 &= \Gamma^{-\frac{1}{2}} \argmin_{\widetilde{\alpha} } \|\Atrain \Gamma^{-\frac{1}{2}} \widetilde{\alpha} - \Ytrain\|^2 + \lambda^2 \widetilde{\alpha}^\top \widetilde{\alpha} \\
 &= [\Gamma^{-\frac{1}{2}}, 0] \argmin_{\begin{bmatrix}\widetilde{\alpha} \\ v\end{bmatrix} s.t. [\Atrain \Gamma^{-\frac{1}{2}}, \lambda I] \begin{bmatrix}\widetilde{\alpha} \\ v\end{bmatrix} = \Ytrain} \| \begin{bmatrix}\widetilde{\alpha} \\ v\end{bmatrix}\|.
\end{align}

In other words, all Tikhonov regularization can be viewed as being a minimum 2-norm interpolating solution for a remixed set of original features combined with the addition of $n$ more special ridge-features that just correspond to a scaled identity --- one special feature for each training point. These special ``ridge-features'' do not predict anything at test-time, but at training time, they do add aliases whose effective expense is controlled by $\frac{1}{\lambda}$. The bigger $\lambda$ is, the cheaper these aliases become, and the more that they bleed signal energy away from other features during minimum 2-norm ``interpolative'' estimation. Meanwhile, the postmultiplication of $\Atrain$ by $\Gamma^{-\frac{1}{2}}$ corresponds to a transformation of the originally given feature family to one that has essentially been premultiplied by $(\Gamma^{-\frac{1}{2}})^\top$, causing the new transformed feature family to have covariance $(\Gamma^{-\frac{1}{2}})^\top \Sigmabold \Gamma^{-\frac{1}{2}}$. This $\Gamma$ transformation can cause a change in the underlying eigenvalues that is essentially a reweighting. 

Consequently, we can understand the Tikhonov-part as essentially being a reweighting of the features. Such a reweighting, by favoring the true parts of the signal and reducing the relative attractiveness of natural aliases, can help control the bleeding if aligned to where the signal actually is. The impact on contamination is indirect and through the same mechanism. Such reweightings can conceivably cause the given feature family's natural alias structure to be better able to dissipate the noise in unfavored directions. However, the ridge-part is adding additional ``aphysical fake features'' that are contamination-free by their nature, though they may cause increased bleeding. The contamination-free nature of the ridge-features is coming from the same reason that they cause the prediction to no longer interpolate the training data. Within the context of interpolation, the role of the ridge part is to be a more attractive destination for bled energy than any natural false feature directions while not being attractive at all relative to true feature directions. 


This is what tells us immediately that the ridge-part alone cannot prevent attenuation/shrinkage (what we call ``bleed'') of signal\footnote{And this is the classical justification for regularizers like Lasso, which stands for ``least absolute \textit{shrinkage} selection operator"!}, it can mitigate the additional adverse effect of the attenuated energy manifesting in other features, which we designate as signal contamination. Adding a ridge regularizer cannot help if the main problem is bleeding.
This avoidance of contamination, that we have seen could arise either from bled signal or noise, provides a simple explanation for why ridge regularizing is better then $\ell_2$-minimizing interpolation~\cite{hastie2019surprises}.
It is also a nice \textit{high-level} explanation for the double-descent behavior of the \textit{Tikhonov regularizer} in~\cite{mei2019generalization}.

\section{Interpolation in the noisy sparse linear model}\label{sec:sparse}

\multifigureexterior{fig:RG_sparse_MSE}{Test MSE of a variety of sparse recovery methods for Gaussian data sampled from $\NORMAL(0,1)$. Here, $n = 5000$ and $k = 500$ and noise $W \sim \NORMAL(0, \sigma^2)$ for different choices of variance $\sigma^2$.}{
\subfig{0.45\textwidth}{$\sigma^2 = 10^{-4}$.}{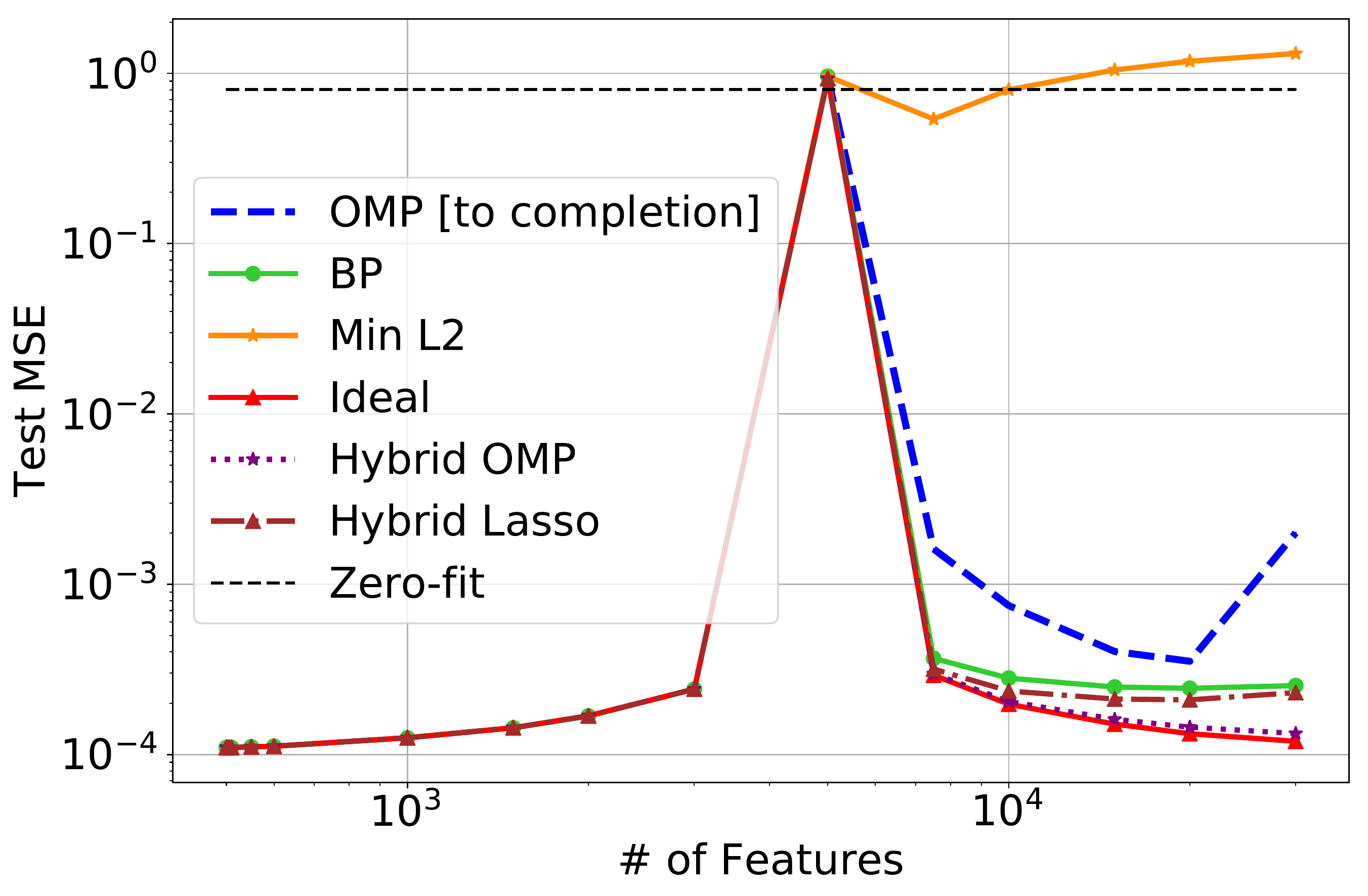}
\subfig{0.45\textwidth}{$\sigma^2 = 10^{-2}$.}{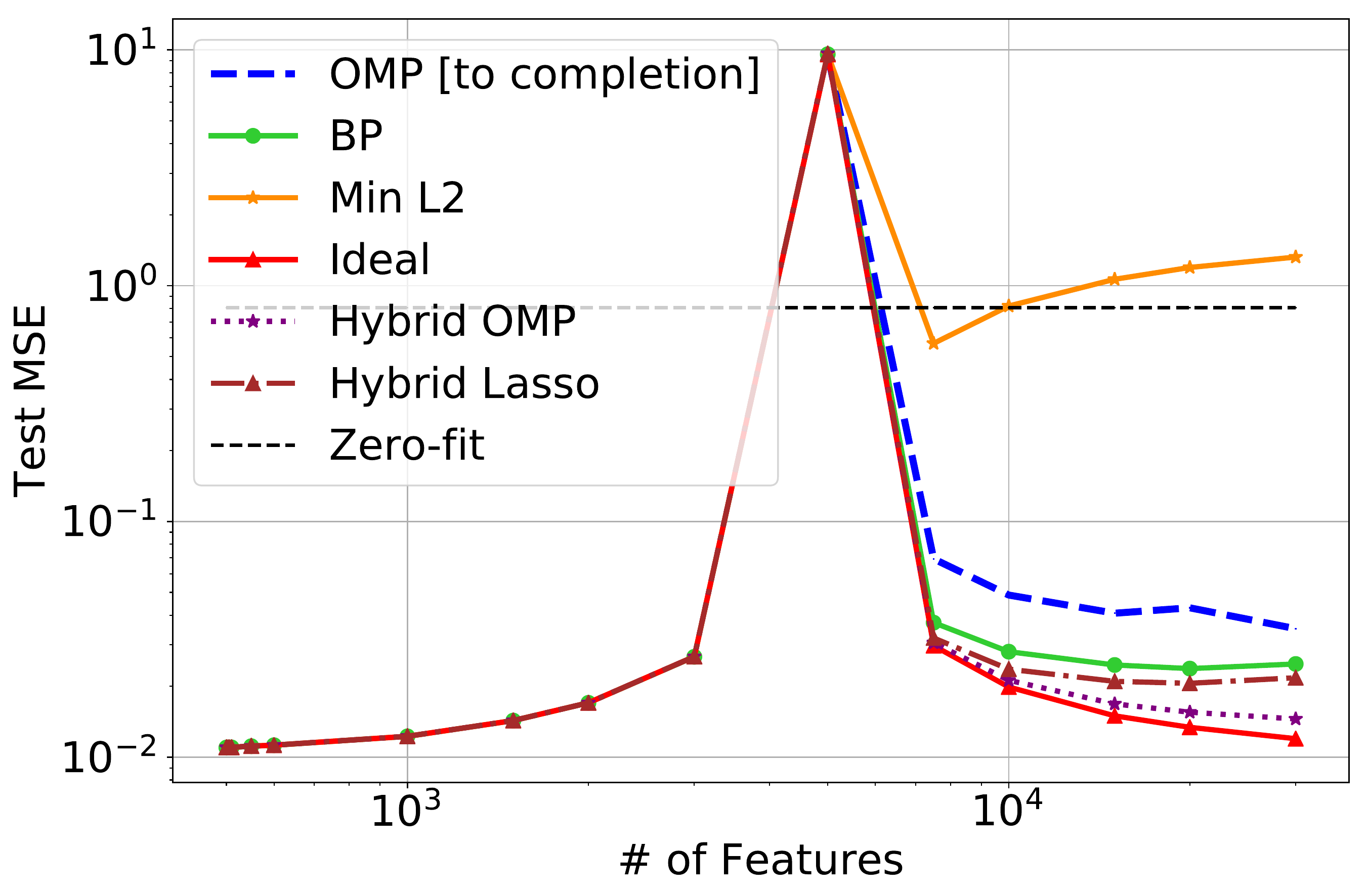}
\subfig{0.45\textwidth}{$\sigma^2 = 10^{-1}$.}{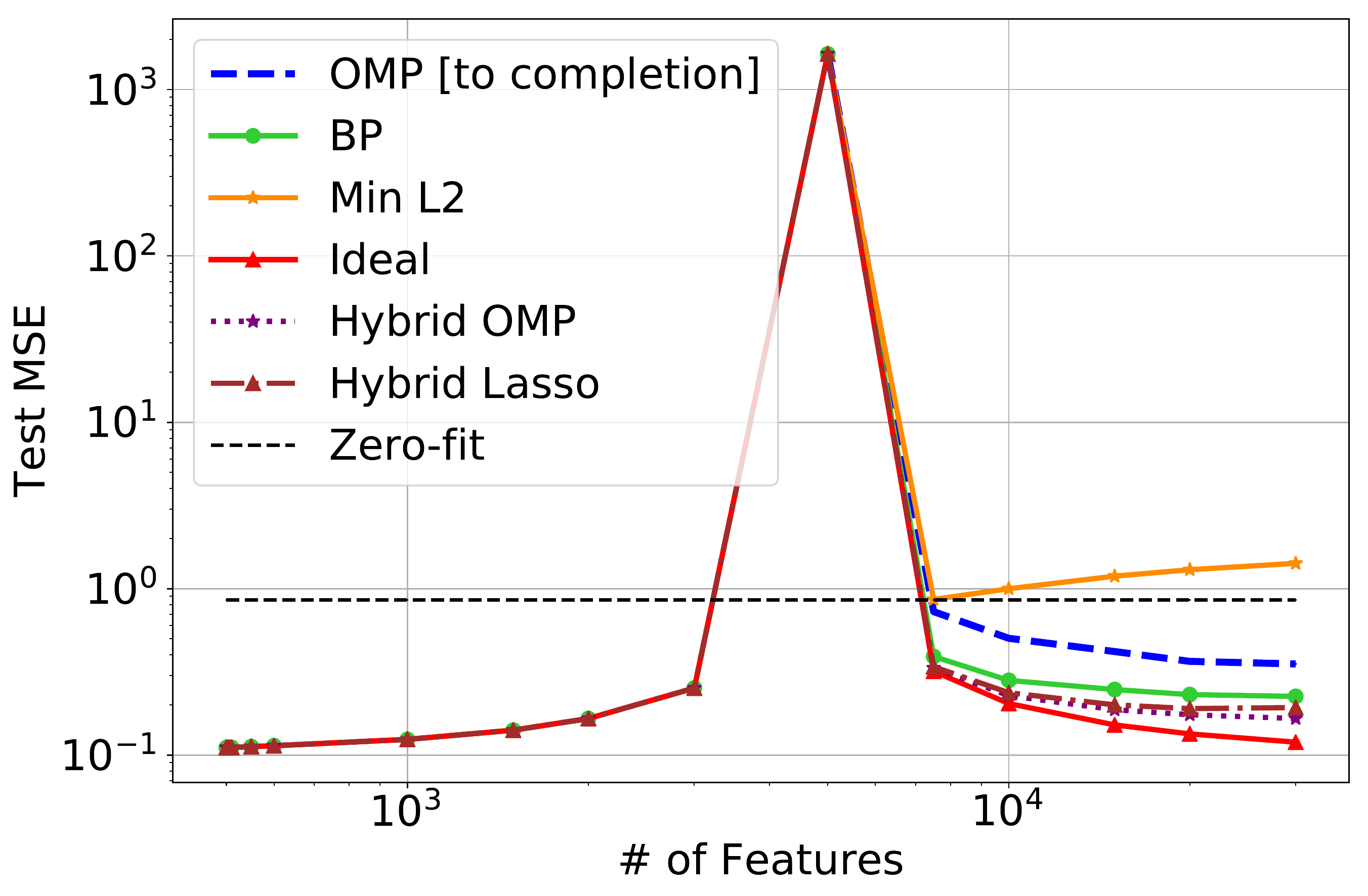}
} 

From the results obtained so far, we can think of any successful interpolator as recovering the signal in the noiseless high-dimensional regime, and then fitting noise harmlessly as in Corollary~\ref{cor:crabpot}.
We just saw that the $\ell_2$-interpolator is generically extremely poor at identifying the correct signal $\alphastar$ even in the noiseless setting -- this was fundamentally due to the issue of the signal \textit{``bleeding out"} across too many features, leading to practically a \textit{zero-fit} as the overparameterization went to infinity.
If the signal $\alphastar$ is dense, we know that there is information-theoretically no hope of identifying it~\cite{wainwright2009information,aeron2010information}.
On the other hand, if there is underlying sparsity in the signal we know that recovery is possible with estimators that are fundamentally \textit{sparsity-seeking}, i.e. estimators obtained through procedures that induce parsimony in the coefficients of $\alphahat$.
Traditionally, the estimators used do not interpolate \textit{in the noisy regime} -- they regularize to essentially denoise the output signal\footnote{(Common examples are the Lasso and orthogonal matching pursuit (OMP) with an early stopping rule.)
}.
In other words, fitting noise has always been thought of as harmful.
We now ask whether this is actually the case by analyzing estimators that \textit{interpolate in the presence of noise}.
First, we define the sparse linear model.

\begin{definition}[Noisy sparse linear model]\label{def:sparse}
For any $k \geq 0, \sigma > 0$, the $(k,\sigma)$-sparse whitened linear model describes output that is generated as
\begin{align*}
Y = \inprod{\avec(\Xvec)}{\alphastar} + W
\end{align*}

where $W \sim \NORMAL(0,\sigma^2)$ and $\alphastar$ is $k$-sparse, i.e. $\vecnorm{\alphastar}{0} \leq k$.
We denote the support of $\alphastar$ by $S^* := \mathsf{supp}(\alphastar)$ and note that $|S^*| \leq k$.
For ease of exposition, we also assume $\EE\left[\avec(\Xvec) \avec(\Xvec)^\top \right] = \mathbf{I}_d$, i.e. the input features are whitened.
\end{definition}


\subsection{Sparsity-seeking methods run to completion}\label{sec:l1}

A starting choice for a reasonable practical interpolator in the sparse regime might be an estimator meant for the \textit{noiseless} sparse linear model to fit the signal as well as noise. Two examples of such estimators are below:
\begin{enumerate}
    \item Orthogonal matching pursuit (OMP) to completion.
    \item Basis pursuit (BP): ${\arg \min} \vecnorm{\alphabold}{1} \text{ subject to Equation~\eqref{eq:interpolatingsoln}.}$
\end{enumerate}

The above methods are extensively analyzed for noiseless sparse linear models~\cite{chen2001atomic}, i.e. the $(k,0)$-sparse linear model as defined in Definition~\ref{def:sparse}.
While they are also known to successfully recover signal in the \textit{high-signal}/\textit{zero-noise} limit, less is known about their performance in the presence of substantial noise.
The test MSE of these interpolators as a function of $d$ in this case is shown in Figure~\ref{fig:RG_sparse_MSE}.
We observe that OMP run to completion and BP are significantly better than the minimum $\ell_2$ norm interpolator.
In fact, they appear only an extra noise multiple of test MSE over the ideal interpolator.
The vast improvement of these estimators over the minimum $\ell_2$-norm interpolator is not surprising on some level, as at the very least these estimators will preserve signal.
The estimator that uses BP in its first step, and then thresholds the coefficients that are below (half the) minimum absolute value of the non-zero coefficients of $\alphastar$, has been shown to recover the sign pattern, or the true support, of $\alphastar$ exactly~\cite{saligrama2011thresholded}.
Indeed, this implies that after the first step, the true signal components (corresponding to entries in $\supp(\alphastar)$) are guaranteed to be preserved by BP even in the presence of low-enough levels of noise.
This means that the signal is preserved, or ``survives" to use our language, with these sparsity-seeking interpolators in the presence of noise.
Similarly, one can show that the first $k$ steps of OMP run to completion will indeed select the correct features corresponding to $\supp(\alphastar)$ with high probability; for more details about OMP, see Appendix~\ref{app:sparsity}.


\multifigureexterior{fig:sparse_purenoise}{Test MSE of sparsity-seeking methods run to completion for Gaussian features sampled from $\NORMAL(0,1)$ and zero signal (pure noise). For every sample point, we have noise $W \sim \NORMAL(0, \sigma^2)$ for $\sigma^2 = 10^{-2}$.}{
\subfigl{0.5\textwidth}{fig:sparse_purenoise_1}{Fixed $n = 1000$, test MSE as a function of $d \geq n$.}{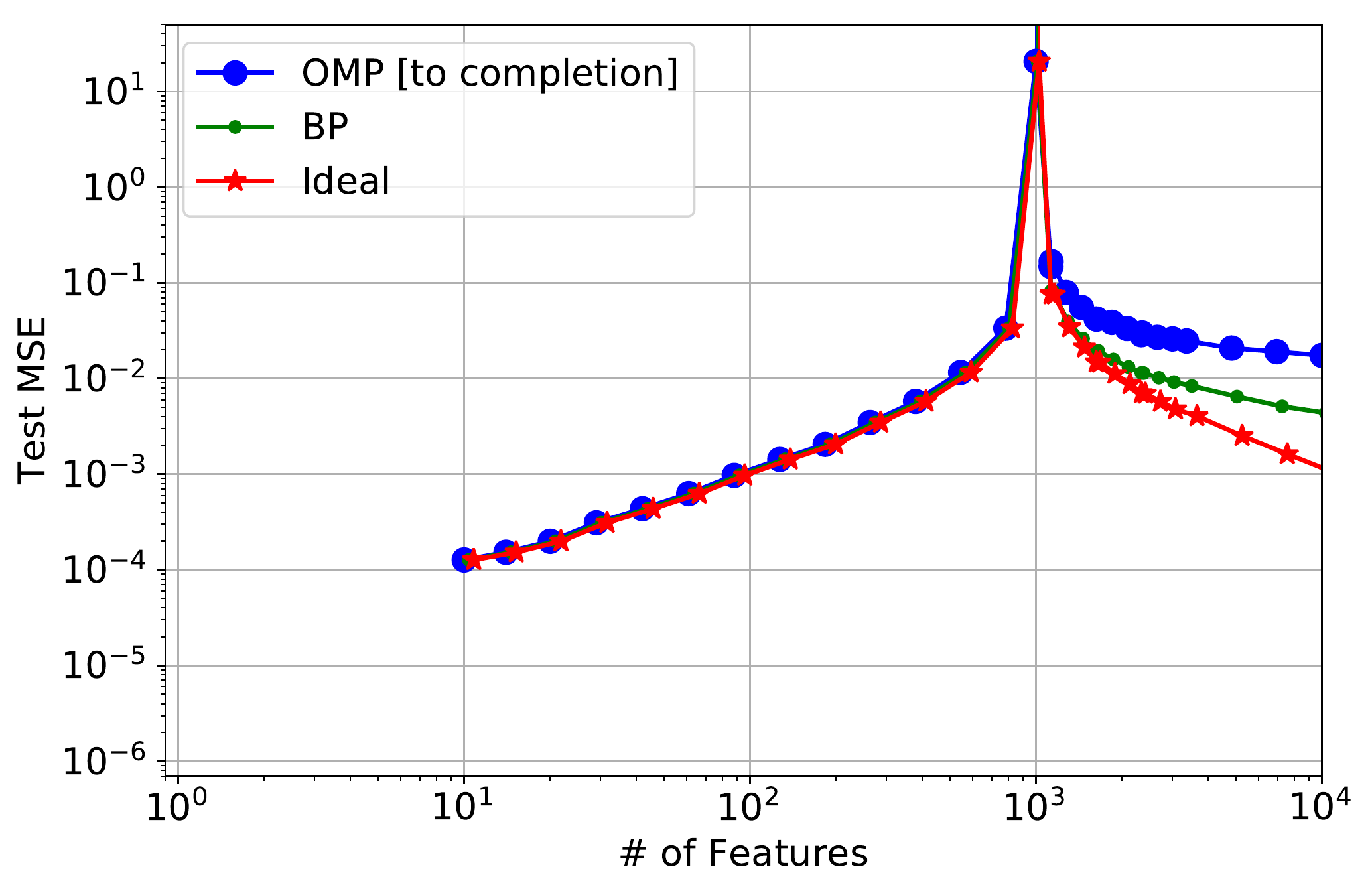}
\subfigl{0.5\textwidth}{fig:sparse_purenoise_4}{Fixed $n = 50$, test MSE as a function of $d \geq n$.}{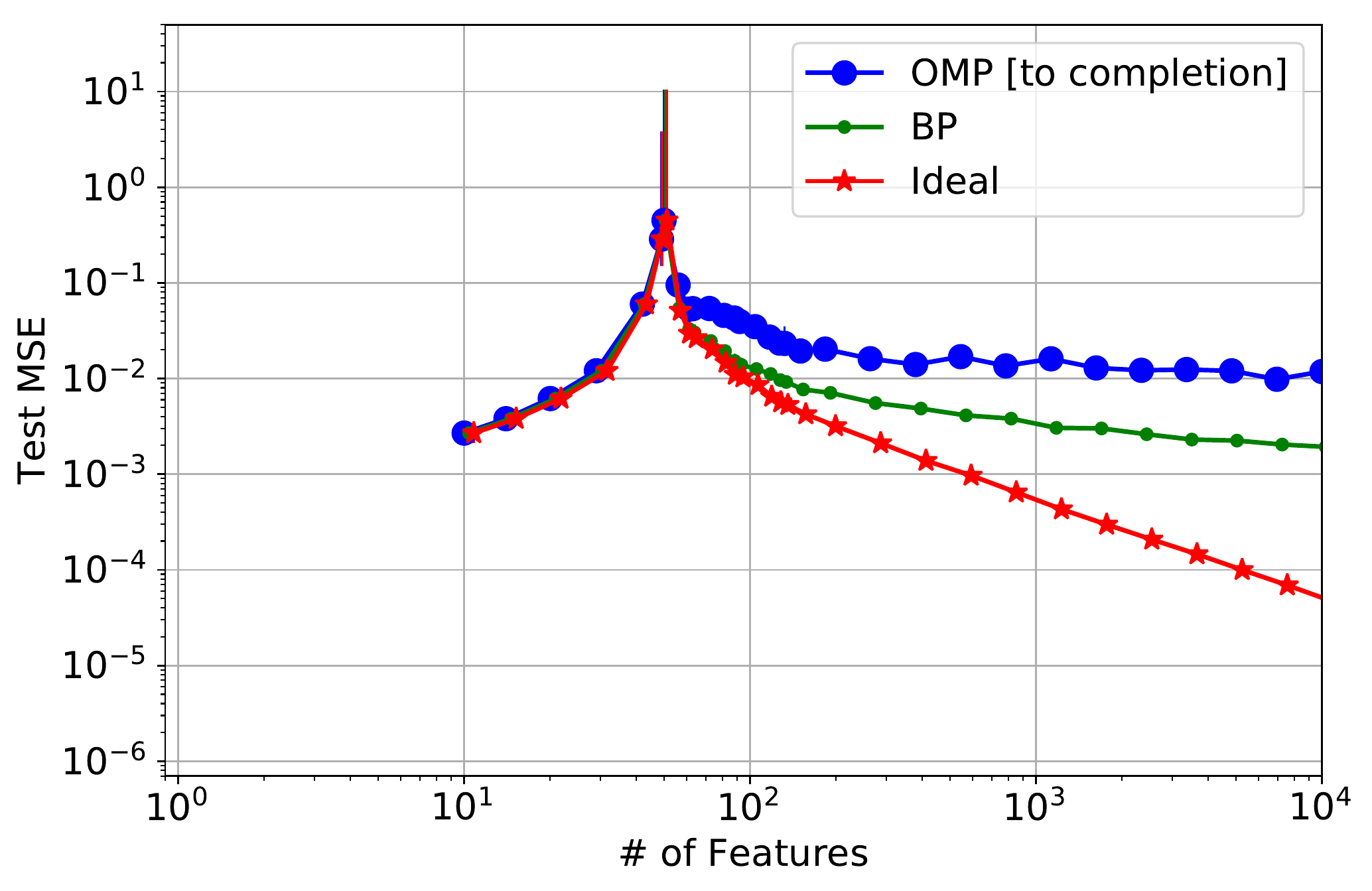}
\subfigl{0.5\textwidth}{fig:sparse_purenoise_2}{Fixed growth $d = n^2$, test MSE as a function of $n \geq 10$}{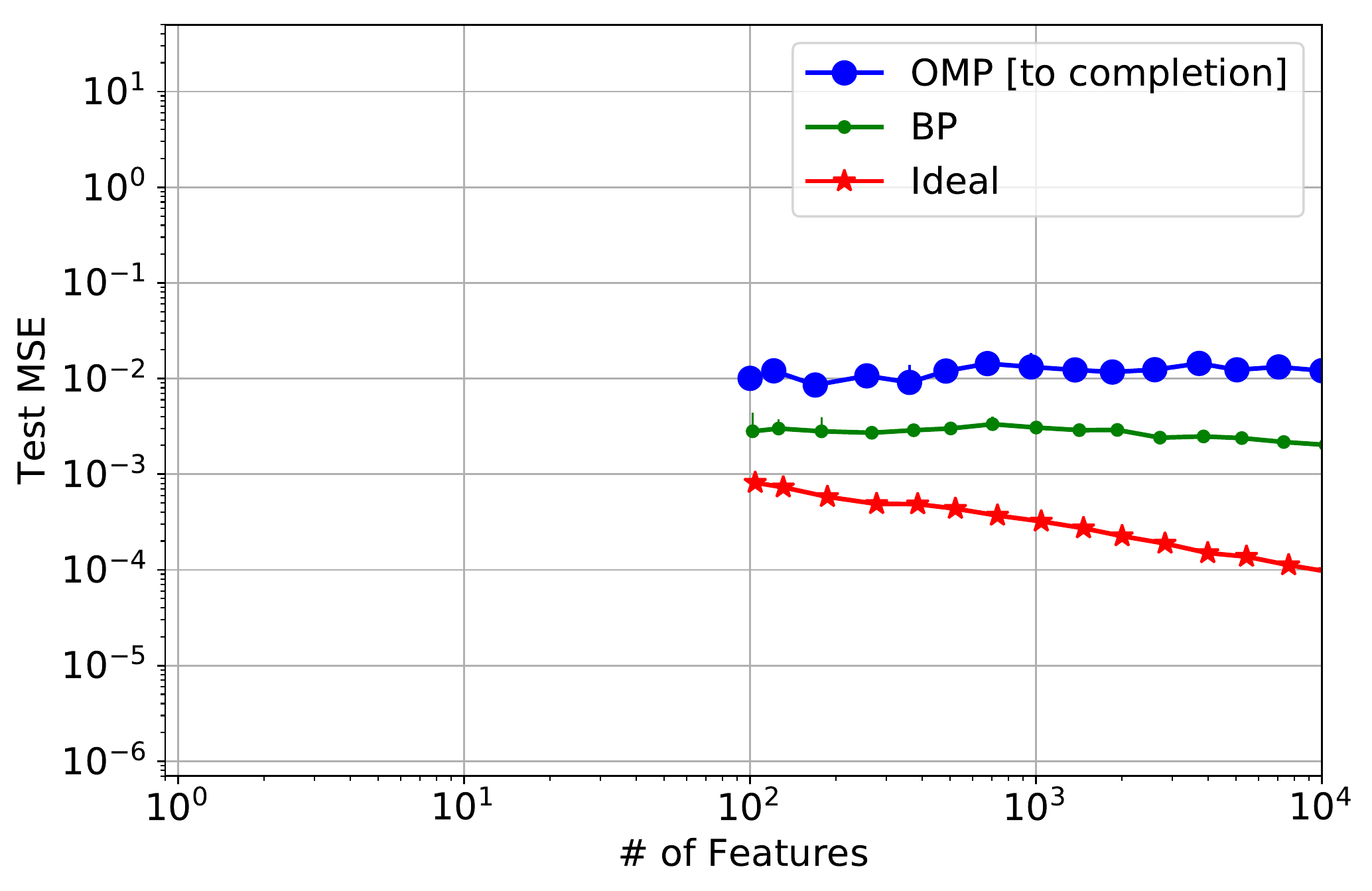}
\subfigl{0.5\textwidth}{fig:sparse_purenoise_3}{Fixed growth $d = e^n$, test MSE as a function of $n \geq 2$.}{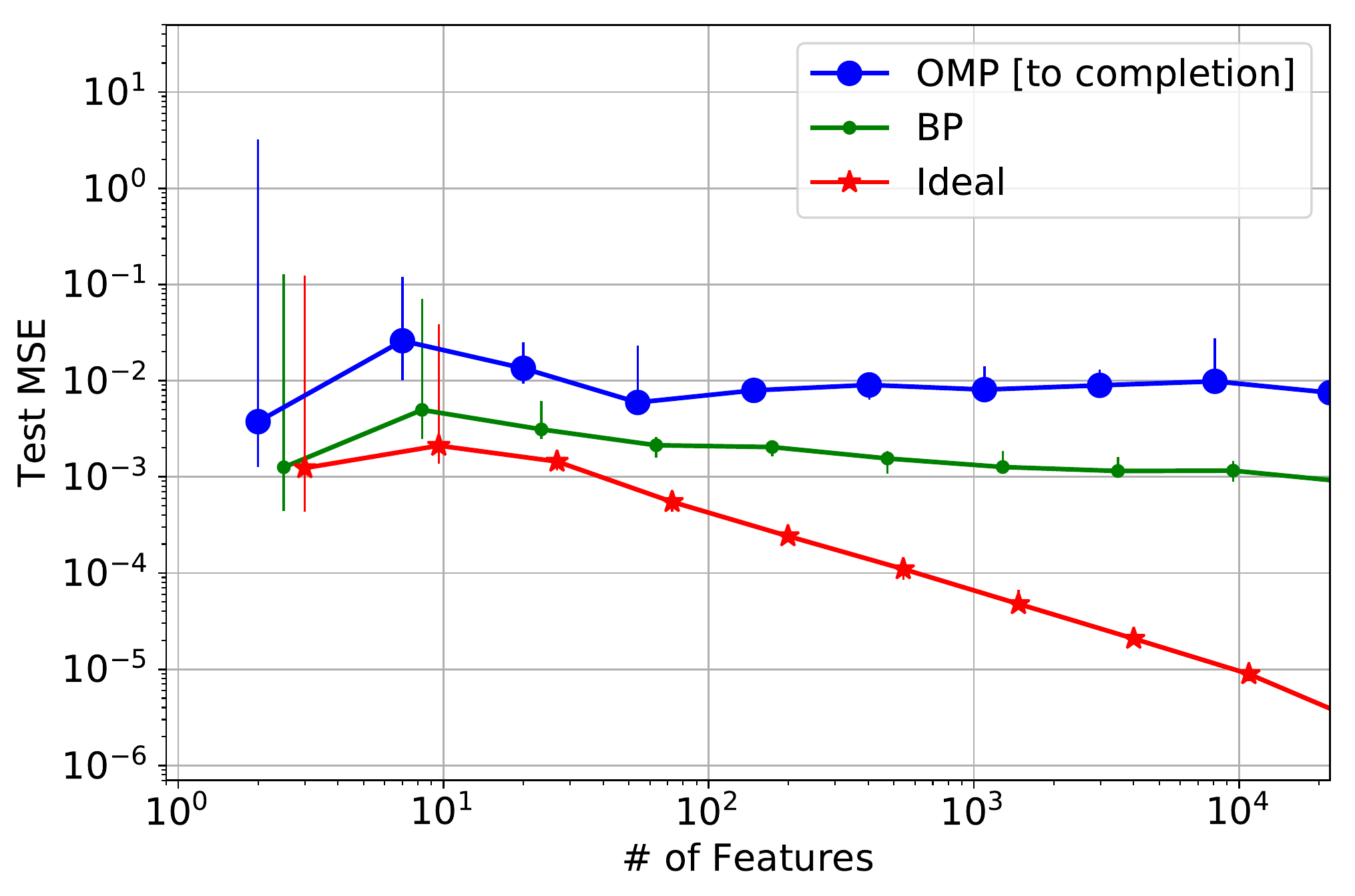}
} 

What is not immediately clear about BP and OMP run to completion is their effect on fitting the noise itself -- does it overfit terribly, or harmlessly (like in Corollary~\ref{cor:crabpot})?
This is directly connected to the ``signal contamination" factor discussed in Section~\ref{sec:l2}, and is not directly answered by existing analysis.
For example,~\cite{saligrama2011thresholded} only guarantees that spurious coefficients are at most half the minimum \textit{non-zero} entry of $\alphastar$, which is generally a constant.
To get a clearer picture of what may happen to purely sparsity-seeking interpolators, we isolate the effect of sparsity-seeking interpolation on noise.
Consider the special case of \textit{zero signal}, i.e. let $\alphastar = \mathbf{0}$.
In this case, for whitened feature families the test MSE of an estimator $\alphahat$ is simply $\vecnorm{\alphahat}{2}^2$, and the estimator $\alphahat$ is in fact fitting pure noise.
Figure~\ref{fig:sparse_purenoise} shows the scaling of the test MSE in the zero-signal case as a function of $n$ and $d$ for OMP and BP.
The ideal test MSE, i.e., the fundamental price of interpolation of noise, is also plotted for reference.
We make these plots for three high-dimensional regimes:
\begin{enumerate}
\item $n$ fixed, and $d \geq n$ growing. In this regime, we ideally want the test MSE to decay to $0$ as $d \to \infty$.
This regime is plotted in Figure~\ref{fig:sparse_purenoise_1} for $n = 1000$, and~\ref{fig:sparse_purenoise_4} for $n = 50$.
\item $d = n^2$ and $n$ growing. We ideally want \textit{consistency} in the sense that we want the test MSE to decay to $0$ as $(n,d) \to \infty$.
This regime is plotted in Figure~\ref{fig:sparse_purenoise_2}.
\item $d = e^n$ and $n$ growing. As before, we want the test MSE to to decay to $0$ as $(n,d) \to \infty$.
This regime is plotted in Figure~\ref{fig:sparse_purenoise_3}.
\end{enumerate}

In all the regimes, Figure~\ref{fig:sparse_purenoise} shows us that the test MSE of the sparsity-seeking interpolators does slightly decrease with $(n,d)$, but extremely slowly and negligibly in comparison to the ideal test MSE.
The issue is that these interpolators are fundamentally parsimonious: they use $\mathcal{O}(n)$ features to fit the noise.
While this property was desirable for signal recovery, it is not that desirable for fitting noise as harmlessly as possible.
We define such an ``overly parsimonious interpolating operator" broadly below.
\begin{definition}\label{def:parsimonious}
For a fixed training data matrix $\Atrain$, consider any interpolating operator $\alphahat := \widehat{\alpha}(\mathbf{Y})$.
Let $S_{\mathsf{top},n} := \{s_j\}_{j=1}^n \subset [d]$ represent the indices for the top $n$ absolute values of coefficients of $\alphahat$.
Then, define truncated vector $\alphahat_{\mathsf{trunc}}$ such that
\begin{align}\label{eq:trunc}
(\alphahat_{\mathsf{trunc}})_j = \begin{cases}
\alphahat_j \text{ if } j \in S_{\mathsf{top},n} \\
0 \text{ otherwise. } 
\end{cases}
\end{align}
Then, for some constant $\beta \in (0,1]$, a $\beta$-\textbf{parsimonious interpolating operator} satisfies
\begin{align}\label{eq:parsimonious}
\vecnorm{\Atrain \alphahat_{\mathsf{trunc}}}{2}^2 \geq \beta \vecnorm{\mathbf{Y}}{2}^2
\end{align}

for all vectors $\mathbf{Y} \in \reals^n$.
For the special case $\beta = 1$, we will call this a \textbf{hard-sparse interpolating operator}.
\end{definition}

Definition~\ref{def:parsimonious} applies to interpolating operators that are a function of the output $\Ytrain$.
For a given data matrix, a parsimonious interpolating operator selects exactly $n$ features that capture a constant fraction of the energy of \textit{any output}.
(The locations of the $n$ features that are selected will, of course, depend on the output.)
For the special case of $\beta = 1$, we will see that \textit{hard-sparse interpolators} include OMP and the minimum-$\ell_1$-norm interpolator (BP).
Because a non-vanishing fraction (represented by the parameter $\beta > 0$) of the noise energy is concentrated in only $n$ features, the effective overparameterization benefit of ideal interpolation of noise (that we saw in Corollary~\ref{cor:crabpot}) can no longer be realized.

We state the main result of this section for a random whitened feature matrix $\Atrain$ satisfying one out of sub-Gaussianity (Assumption~\ref{as:ideal2}) or Gaussianity (Assumption~\ref{as:ideal3}).

\begin{theorem}\label{thm:parsimoniousnoisefit}
Consider any interpolating solution $\alphahat$ obtained by a $\beta$-parsimonious interpolating operator (as defined in Definition~\ref{def:parsimonious}) with constant $\beta \in (0,1]$.
Then, when applied to any random whitened feature matrix $\Atrain$ satisfying:
\begin{enumerate}
    \item Gaussianity (Assumption~\ref{as:ideal3}), there exists an instance of the $(k, \sigma^2)$-sparse linear model for any $k \geq 0$ for which the test MSE
    \begin{align}\label{eq:parsimoniousnoisefit_gaussian}
    \testerr(\alphahat) \geq \frac{\beta \sigma^2 (1 - \delta)}{4\ln \left(\frac{d}{n}\right)} 
    \end{align}
    for any $\delta > 0$, with probability at least $(1 - e^{-n \ln \left(\frac{d}{n}\right)} - e^{-n \delta^2/2})$ over realizations of feature matrix $\Atrain$ and noise $\Wtrain$.
    \item sub-Gaussianity of rows (Assumption~\ref{as:ideal2}) with parameter $K > 0$, there exists an instance of the $(k, \sigma^2)$-sparse linear model for which the test MSE
    \begin{align}\label{eq:parsimoniousnoisefit_subgaussian}
    \testerr(\alphahat) \geq \frac{\beta \sigma^2 (1 - \delta)}{C''_K \ln \left(\frac{d}{n}\right)} 
    \end{align}
    for any $\delta > 0$, with probability at least $(1 - e^{-n \ln \left(\frac{d}{n}\right)} - e^{-n \delta^2/2})$ over realizations of feature matrix $\Atrain$ and noise $\Wtrain$.
    Here, constant $C''_K > 0$ depends only on the upper bound on the sub-Gaussian parameter, $K$.
\end{enumerate}
\end{theorem}

Theorem~\ref{thm:parsimoniousnoisefit} should be thought of as a negative result for the applicability of parsimonious interpolators meant for the noiseless setting in additionally fitting noise, even when the setting is heavily overparameterized -- that is, we are no longer enjoying harmless interpolation of noise.
Consider the following scalings of $d$ with respect to $n$:
\begin{enumerate}
\item $d = \gamma n$ for some constant $\gamma > 1$. 
In this case, Equations~\eqref{eq:parsimoniousnoisefit_gaussian} and~\eqref{eq:parsimoniousnoisefit_subgaussian} give us $\vecnorm{\alphahat}{2}^2 = \omega(\sigma^2)$, and the test MSE does not go to $0$ as $n \to \infty$. 
\item $d = n^q$ for some $q > 1$.
In this case, Equations~\eqref{eq:parsimoniousnoisefit_gaussian} and~\eqref{eq:parsimoniousnoisefit_subgaussian} give us $\vecnorm{\alphahat}{2}^2 = \omega\left(\frac{\sigma^2}{(q-1) \ln n}\right)$ which goes to $0$ as $n \to \infty$, but at an extremely slow \textit{logarithmic} rate.
\item $d = e^{\gamma n}$ for some $\gamma > 0$.
This is an extremely overparameterized regime.
In this case, Equations~\eqref{eq:parsimoniousnoisefit_gaussian} and~\eqref{eq:parsimoniousnoisefit_subgaussian} give us $\vecnorm{\alphahat}{2}^2 = \omega\left(\frac{\sigma^2}{\gamma n - \ln n}\right)$ which is a much faster rate.
However, in this exponentially overparameterized regime it is well known that successful signal recovery is impossible even in the absence of noise.
\end{enumerate}
Putting these conclusions together, Theorem~\ref{thm:parsimoniousnoisefit} suggests that while consistency of parsimonious interpolators might be possible in polynomially high-dimensional regimes -- it would be at an extremely slow logarithmic rate.

Theorem~\ref{thm:parsimoniousnoisefit} holds for a broad class of sparsity-seeking interpolators that successfully recover signal in the absence of noise (in the polynomially high-dimensional regime).
The following results hold for OMP run to completion, and BP -- both of which satisfy $1$-parsimonious interpolation on any output.

\begin{corollary}\label{cor:omptocompletion}
When the random whitened feature matrix $\Atrain$ satisfies one out of Gaussianity (Assumptions~\ref{as:ideal3}) or sub-Gaussianity (Assumption~\ref{as:ideal2}), the interpolator formed by OMP run to completion incurs test MSE
\begin{align*}
\testerr(\alphahat_{\mathsf{OMP}}) \geq \frac{\sigma^2 (1 - \delta)}{C\ln \left(\frac{d}{n}\right)}
\end{align*}
for some constant $C > 0$ with high probability for at least one instance of the $(k, \sigma)$-sparse linear model.
\end{corollary}

Corollary~\ref{cor:omptocompletion} is trivial because, by nature, OMP run to completion selects exactly $n$ features and stops (see Appendix~\ref{app:sparsity} for details.)
Thus, the OMP interpolator is always $n$-hard-sparse, thus $1$-parsimonious according to Definition~\ref{def:parsimonious}.

\begin{corollary}\label{cor:bp}
When the random whitened feature matrix $\Atrain$ satisfies one out of Gaussianity (Assumptions~\ref{as:ideal3}) or sub-Gaussianity (Assumption~\ref{as:ideal2}), the minimum-$\ell_1$-norm interpolator (BP) incurs test MSE
\begin{align*}
\testerr(\alphahat_{\mathsf{BP}}) \geq \frac{\sigma^2 (1 - \delta)}{C \ln \left(\frac{d}{n}\right)}
\end{align*}
for some constant $C > 0$ with high probability for at least one instance of the $(k, \sigma)$-sparse linear model.
\end{corollary}

Corollary~\ref{cor:bp} is proved in Appendix~\ref{app:bp}. The result follows trivially from Theorem~\ref{thm:parsimoniousnoisefit} once we know that the minimum-$\ell_1$-norm interpolator also selects $n$ features.
While this is less immediately obvious than OMP, it is a classical result\footnote{Also referenced in Chen, Donoho and Saunder's classic survey~\cite{chen2001atomic}. The name ``basis pursuit'' is actually a give-away for this property, as we know that a basis in $\reals^n$ cannot contain more than $n$ elements.}, that follows from the linear program (LP) formulation of BP, that the $\ell_1$-minimizing interpolator selects at most $n$ features, and will generically select \textit{exactly} $n$ features.
We reproduce this argument in full in Appendix~\ref{app:bp} for completeness.

We close this section with the proof of Theorem~\ref{thm:parsimoniousnoisefit}.

\begin{proof}
For any $k > 0$, we consider the instance of the $(k,\sigma)$-sparse linear model for which there is zero signal, i.e. $\alphastar = \mathbf{0}$.
In this case, the output is pure noise, i.e. $\Ytrain = \Wtrain$ and the interpolator operates on pure noise as $\alphahat := \alpha(\Ytrain) = \alpha(\Wtrain)$.
Further, recall that the test MSE on whitened features for any estimator $\alphahat$ is defined as 
\begin{align*}
\testerr(\alphahat) := \vecnorm{\alphahat - \alphastar}{2}^2 = \vecnorm{\alphahat}{2}^2 .
\end{align*}

Recall that $\alphahat_{\mathsf{trunc}}$ is the truncated version of the interpolator $\alphahat$ as defined in Equation~\eqref{eq:trunc}.
Let $S = \mathsf{supp}(\alphahat_{\mathsf{trunc}})$ denote the support of the truncated vector $\alphahat_{\mathsf{trunc}}$.
Recall, that by definition, $|S| = n$. 
More generally, the actual composition of the $n$ elements in $S$ will depend both on the realizations of the random matrix $\Atrain$ and the noise $\Wtrain$.

Let $\Atrain(S)$ be the $n \times n$ matrix with columns sub-sampled from the set $S$ of size $n$, and denote $\Wtrain' := \Atrain \alphahat_{\mathsf{trunc}} = \Atrain(S) \alphahat_{\mathsf{trunc}}$.
Assuming that $\Atrain(S)$ is invertible (which is always almost surely true for a random matrix), we note that
\begin{align*}
\alphahat_{\mathsf{trunc}} = \Atrain(S)^{-1} \Wtrain' ,
\end{align*}

and so we have
\begin{align*}
\vecnorm{\alphahat_{\mathsf{trunc}}}{2}^2 &= (\Wtrain')^\top (\Atrain(S)^\top)^{-1} \Atrain(S)^{-1} (\Wtrain') \\
&= (\Wtrain')^\top (\Atrain(S) \Atrain(S)^\top)^{-1} (\Wtrain') \\
&\geq \vecnorm{\Wtrain'}{2}^2\lambda_{min}\left(\Atrain(S)\Atrain(S)^\top\right)^{-1} \\
&= \frac{\vecnorm{\Wtrain'}{2}^2}{\lambda_{max}\left(\Atrain(S)\Atrain(S)^\top\right)} \\
&= \frac{\vecnorm{\Wtrain'}{2}^2}{\lambda_{max}\left(\Atrain(S)^\top\Atrain(S)\right)} \\
&\geq \frac{\vecnorm{\Wtrain'}{2}^2}{\max_{S' \subset [d], |S'| = n} \lambda_{max}\left(\Atrain(S')^\top\Atrain(S')\right)}
\end{align*}

Thus, we need to prove a lower bound on the maximal eigenvalue $\lambda_{max}\left(\Atrain(S')^\top\Atrain(S')\right)$ that will hold, point-wise, for all $S' \subset [d], |S'| = n$.
We state and prove the following intermediate lemma.
\begin{lemma}\label{lem:rvuniform}
For matrix $\Atrain$ satisfying:
\begin{enumerate}
    \item Assumption~\ref{as:ideal3} (Gaussian features), we have
    \begin{align*}
    \max_{S' \subset [d]: |S'| = n} \lambda_{max}\left(\Atrain(S')^\top\Atrain(S')\right) \leq n \left(2 + 2\sqrt{\ln \left(\frac{d}{n}\right)}\right)^2
    \end{align*}
    with probability greater than or equal to $(1 - e^{-n \ln \left(\frac{d}{n}\right)})$.
    \item Assumption~\ref{as:ideal2} (Sub-Gaussianity), we have
    \begin{align*}
    \max_{S' \subset [d]: |S'| = n} \lambda_{max}\left(\Atrain(S')^\top\Atrain(S')\right) \leq n \left((C_K + 1) + \frac{\sqrt{2\ln \left(\frac{d}{n}\right)}}{c_K}\right)^2
    \end{align*}
    with probability greater than or equal to $(1 - e^{-n \ln \left(\frac{d}{n}\right)})$, where parameter $C_K, c_K > 0$ are \textbf{positive constants} that depend on the sub-Gaussian parameter $K$.
\end{enumerate}
\end{lemma}

Notice that Lemma~\ref{lem:rvuniform} directly implies the statement of Theorem~\ref{thm:parsimoniousnoisefit}.
This is because with probability greater than or equal to $(1 - e^{-n \ln \left(\frac{d}{n}\right)})$, we then have for any subset $|S'| = n$,
\begin{align*}
\vecnorm{\alphahat_{\mathsf{trunc}}}{2}^2 &\geq \frac{\vecnorm{\Wtrain'}{2}^2}{\lambda_{max}\left(\Atrain(S')^\top\Atrain(S')\right)} \\
&\geq \frac{\vecnorm{\Wtrain'}{2}^2}{n \left(C + \sqrt{C'\ln \left(\frac{d}{n}\right)}\right)^2} \\
&\geq \frac{\vecnorm{\Wtrain'}{2}^2}{C''n \ln \left(\frac{d}{n}\right)}
\end{align*}

for positive constants $C,C',C''$ (we have dropped the subscript dependence on $K$ for the sub-Gaussian case for convenience).
Now, recall that $\Wtrain' = \Atrain \alphahat_{\mathsf{trunc}}$, and from Equation~\eqref{eq:parsimonious} in Definition~\ref{def:parsimonious}, we get
\begin{align*}
\vecnorm{\Wtrain'}{2}^2 = \vecnorm{\Atrain \alphahat_{\mathsf{trunc}}}{2}^2 \geq \beta \vecnorm{\Wtrain}{2}^2
\end{align*}

Also recall from the lower tail bound on chi-squared random variables that $\vecnorm{\Wtrain}{2}^2 \geq n\sigma^2(1- \delta)$ with probability greater than or equal to $(1 -e^{-\frac{n \delta^2}{8}})$.
Putting these facts together, we get
\begin{align*}
\vecnorm{\alphahat_{\mathsf{trunc}}}{2}^2 \geq \frac{\beta \sigma^2 (1 - \delta)}{C'' \ln \left(\frac{d}{n}\right)} .
\end{align*}

Finally, we note that truncation to $0$ only decreases the $\ell_2$-norm, and thus $\vecnorm{\alphahat}{2}^2 \geq \vecnorm{\alphahat_{\mathsf{trunc}}}{2}^2$.
This gives us
\begin{align*}
\testerr(\alphahat) = \vecnorm{\alphahat}{2}^2 \geq \frac{\beta \sigma^2 (1 - \delta)}{C'' \ln \left(\frac{d}{n}\right)} ,
\end{align*}

which is precisely the statement in Theorem~\ref{thm:parsimoniousnoisefit}.
Note that the overall probability of this statement is greater than or equal to $\left(1 - e^{-\frac{n \delta^2}{8}} - e^{-n \ln \left(\frac{d}{n}\right)}\right)$.

It only remains to prove Lemma~\ref{lem:rvuniform}, which we do below.
\begin{proof}
Under Gaussianity (Assumption~\ref{as:ideal3}), we use Lemma~\ref{lem:gaussianconcentration} which we introduced in the proof of Corollary~\ref{cor:fundamentalprice}.
For every $S' \subset [d], |S'| = n$, a direct substitution of of Lemma~\ref{lem:gaussianconcentration} applied to the \textit{maximum singular value} of matrix $\Atrain(S')$ with $t := t_{0,1} = \sqrt{4n\left(\ln \left(\frac{d}{n}\right) + 1\right)}$ gives us
\begin{align*}
\Pr\left[\lambda_{max}\left(\Atrain(S')^\top\Atrain(S')\right) > \left(2\sqrt{n} + t_{0,1}\right)^2\right] \leq e^{-\frac{t_{0,1}^2}{2}} = e^{-2n\ln \left(\frac{d}{n}\right)} .
\end{align*}

Similarly, under Assumption~\ref{as:ideal2}, we use Lemma~\ref{lem:matrixconcentration} stated in Appendix~\ref{app:technical}.
For every $S' \subset [d], |S'| = n$, a direct substitution of of Lemma~\ref{lem:matrixconcentration} applied to the \textit{maximum singular value} of matrix $\Atrain(S')$ with $t = t_{0,2} := \sqrt{\frac{2n\left(\ln \left(\frac{d}{n}\right) + 1\right)}{c_K}}$ gives us
\begin{align*}
\Pr\left[\lambda_{max}\left(\Atrain(S')^\top\Atrain(S')\right) > \left(\sqrt{n}(1 + C_K) + t_0\right)^2\right] \leq e^{-c_K t_0^2} = e^{-2n\ln \left(\frac{d}{n}\right)} .
\end{align*}

For convenience, we unify the rest of the argument for both assumptions, denoting $C := \{1, C_K\}$, $c := \{\frac{1}{2}, c_K\}$ and $t_0 := \{t_{0,1}, t_{0,2}\}$ for the Gaussian case and sub-Gaussian case respectively.
Applying the union bound for all $S' \subset [d], |S'| = n$ gives us
\begin{align*}
&\Pr\left[\max_{|S'| = n} \lambda_{max}\left(\Atrain(S')^\top\Atrain(S')\right) > \left(\sqrt{n}(1 + C) + t_0\right)^2\right] \\
&\leq \sum_{S' \subset [d]: |S'| = n} \Pr\left[ \lambda_{max}\left(\Atrain(S')^\top\Atrain(S')\right) > \left(\sqrt{n}(1 + C) + t_0\right)^2\right] \\
&\leq {d \choose n} \cdot e^{-c t_0^2} \\
&\stackrel{(\mathsf{i})}{<} \left(\frac{ed}{n}\right)^n \cdot e^{-c t_0^2}\\
&= e^{n \left(\ln \left(\frac{d}{n}\right) + 1\right) - c t_0^2} \\
&= e^{-n \left(\ln \left(\frac{d}{n}\right) + 1\right)}
\end{align*}
where we have substituted $t_0 := \sqrt{\frac{2n\left(\ln \left(\frac{d}{n}\right) + 1\right)}{c}}$ and inequality $(\mathsf{i})$ follows from Fact~\ref{fact:binomialcoeffbound} in Appendix~\ref{app:facts}.
This completes the proof for both the Gaussian case and the sub-Gaussian case.
\end{proof}

Notice that the union-bound used here is quite likely a bit loose. But we know that the maximum of independent random variables does not behave far from what the union-bound predicts, and so tightening here would require exploiting non-independence.
\end{proof}

\subsection{Order-optimal hybrid methods}\label{sec:hybrid}

Can we construct an optimal interpolation scheme in the $(k,\sigma)$-noisy sparse linear model for any $(k, \sigma)$? 
We define optimality in the sense of order-optimality, i.e. the optimal scaling for test MSE as a function of $(d,n)$ that holds for all instances\footnote{Statisticans commonly call this \textit{minimax-optimality}.} in the $(k,\sigma)$-noisy sparse linear model.
Clearly, any order-optimal interpolator needs to satisfy the following two conditions:
\begin{enumerate}
\item The test MSE due to the presence of signal should be proportional to $\sigma^2 \cdot \frac{k \ln \left(\frac{d}{k}\right)}{n}$.
This is the best possible scaling in test MSE that \textit{any} estimator can achieve in the noisy sparse linear model, whether or not such an estimator interpolates.
\item The test MSE due to the presence of noise, even when the signal is absent, should be proportional to $\sigma^2 \cdot \frac{n}{d}$.
As we saw in Corollaries~\ref{cor:fundamentalprice} and~\ref{cor:crabpot}, this is the fundamental price of interpolation of noise.
\end{enumerate}

Let us consider the requirements of an interpolator to meet the above two conditions for order-optimality.
For one, it should be able to uniquely identify the signal in the absence of noise; two, it needs to fit noise harmlessly across $\omega(d)$ features.
We saw in Section~\ref{sec:l2} that the minimum-$\ell_2$-norm interpolator fails at signal identifiability because it tends to bleed signal out across multiple features, rather than concentrating it in a few entries.
On the other hand, in Section~\ref{sec:l1} we saw that purely sparsity-seeking interpolators are poor choices to fit noise harmlessly because of their parsimony -- they only spread noise among $n$ features.
Thus, the goals for signal identifiability and harmless noise-fitting are fundamentally at odds with one another in the overparameterized regime with underlying sparsity.

This tradeoff suggests that the ``optimal" interpolator should use different procedures for fitting the part of the output that is signal, and the part of the output that is noise.
It suggests the application of a hybrid scheme, described in the $k$-sparse regime below.
Recall that we have assumed whitened feature families, i.e. $\Sigmabold = \mathbf{I}_d$, for ease of exposition\footnote{This is primarily to be able to state our corollaries for sparse signal recovery out-of-the-box.
However, more general versions of these results will hold for unwhitened feature families as well.
}.

\begin{definition}\label{def:hybrid}
Corresponding to any estimator $\alphahat_1$ of $\alphastar$ for the $(k,\sigma$)-whitened sparse linear model (note that this estimator \textbf{need not interpolate}), we can define a two-step \textbf{hybrid interpolator} as below:
\begin{enumerate}
    \item Compute the residual $\Wtrain' = \Ytrain - \Atrain \alphahat_1$.
    \item Return $\alphahat_{\mathsf{hybrid}} := {\arg \min} \vecnorm{\alphabold - \alphahat_1}{2}\text{ subject to $\alphabold$ satisfying Equation~\eqref{eq:interpolatingsoln} }$.
    Clearly, an equivalent characterization is $\alphahat_{\mathsf{hybrid}} = \alphahat_1 + \Deltahat$, where $\Deltahat := {\arg \min} \vecnorm{\Deltabold}{2} \text{ subject to } \Atrain \Deltabold = \Wtrain'$.
\end{enumerate}
\end{definition}

Observe that the feasibility constraint is just a rewriting of Equation~\eqref{eq:interpolatingsoln}, ensuring that the estimator $\alphahat_{\mathsf{hybrid}}$ interpolates the data.
The guarantee that such a hybrid scheme achieves on test MSE is stated in the following proposition.

\begin{proposition}\label{prop:sparseinterpolators}
Denote the \textbf{estimation error} and \textbf{prediction error} of the estimator $\alphahat_1$ by
\begin{align*}
\mathcal{E}_{\mathsf{est}}(\alphahat_1) &= \vecnorm{\alphahat_1 - \alphastar}{2}^2 \\
\mathcal{E}_{\mathsf{pred}}(\alphahat_1) &= \frac{1}{n}\vecnorm{\Atrain(\alphahat_1 - \alphastar)}{2}^2 .
\end{align*}
Then, for \textbf{any} estimator $\alphahat_1$, the hybrid estimator $\alphahat_{\mathsf{hybrid}}$ has test MSE
\begin{align}\label{eq:hybridinterpolator}
    \testerr(\alphahat_{\mathsf{hybrid}}) \leq \mathcal{E}_{\mathsf{est}}(\alphahat_1) + \frac{2\vecnorm{\Wtrain}{2}^2 + 2 n \mathcal{E}_{\mathsf{pred}}(\alphahat_1)}{\lambda_{min}(\Atrain \Atrain^\top)^{-1}} + \sigma^2
\end{align}
\end{proposition}

Proposition~\ref{prop:sparseinterpolators} shows a natural split in the error of such hybrid interpolators in terms of two quantities: the error that arises from signal recovery, and the error that arises from fitting noise.
The proof of Proposition~\ref{prop:sparseinterpolators} follows simply from the insights already presented and is given in Appendix~\ref{app:sparsity} for completeness.

Using the ideas from the proof of Corollary~\ref{cor:crabpot}, one can then show that the performance of a hybrid interpolator fundamentally depends on the ability of the sparse signal estimator to estimate the signal, plus the excess error incurred by the ideal interpolator of pure noise, i.e. the ideal test MSE arising from fitting noise.
For the rest of this section, we focus on the special case where the entries of $\Atrain$ are drawn from the standard normal distribution, i.e. $\Atrain$ satisfies Assumption~\ref{as:ideal3} together with whiteness.
The literature on sparse signal recovery is rich, and we can interpret the guarantee provided by Proposition~\ref{prop:sparseinterpolators} for a variety of choices of the first-step estimator $\alphahat_1$.
A summary of these results is contained in Table~\ref{tab:corollaries}.

\begin{table}[]
\centering
\begin{tabular}{ccc}
Estimator $\alphahat_1$ & Side information & Test MSE of hybrid interpolator  \\
\hline
Optimal (e.g. SLOPE)  & $\sigma^2$ & $\Oh\left(\sigma^2 \cdot \frac{k \ln \left(\frac{d}{k}\right)}{n} + \sigma^2 \cdot \frac{n}{d}\right) + \sigma^2$ \\                                   \\
Lagrangian Lasso  & $\sigma^2$ & $\Oh\left(\sigma^2 \cdot \frac{k \ln d}{n} + \sigma^2 \cdot \frac{n}{d}\right) + \sigma^2$ \\ \\
OMP upto $k$ steps & $k$  &  $\Oh\left(\sigma^2 \cdot \frac{k \ln d}{n} + \sigma^2 \cdot \frac{n}{d}\right) + \sigma^2$ \\ \\
OMP upto $k_0 > k$ steps & Upper bound on $k$ &  $\Oh\left(\sigma^2 \cdot \frac{k_0 \ln d}{n} + \sigma^2 \cdot \frac{n}{d}\right) + \sigma^2$ \\ \\
Square-root-Lasso & None &  $\Oh\left(\sigma^2 \cdot \frac{k \ln d}{n} + \sigma^2 \cdot \frac{n}{d}\right) + \sigma^2$ \\ \\
\end{tabular}
\caption{Corollaries of Proposition~\ref{prop:sparseinterpolators} for a variety of initial estimators for the $(k,\sigma)$-sparse linear model.}
\label{tab:corollaries}
\end{table}

First, we consider estimators $\alphahat_1$ that are optimal in their scaling with respect to $(k,\sigma^2,n,d)$ in the sparse regime.
From here on, we will call these order-optimal estimators.
\begin{corollary}\label{cor:minimaxoptimal}
Consider the feature matrix $\Atrain$ with iid standard Gaussian entries, and any estimator $\alphahat_1$ that is order-optimal in both estimation error and prediction error, i.e. there exist universal constants $C, C' > 0$ such that
\begin{subequations}
\begin{align}
\mathcal{E}_{\mathsf{est}}(\alphahat_1) &\leq \frac{Ck \ln \left(\frac{d}{k}\right)}{n} \label{eq:minimaxest} \\
\mathcal{E}_{\mathsf{pred}}(\alphahat_1)&\leq \frac{C'k \ln \left(\frac{d}{k}\right)}{n} \label{eq:minimaxpred} .
\end{align}
\end{subequations}

with high probability (over the randomness of both the noise $\Wtrain$ and the randomness in the whitened feature matrix $\Atrain$).
Then, provided that $d > 4n$, the hybrid estimator $\alphahat_{\mathsf{hybrid}}$ that is based on estimator $\alphahat_1$ gives us test MSE
\begin{align}\label{eq:minimaxoptimal}
\testerr(\alphahat_{\mathsf{hybrid}}) &\leq C''\sigma^2\frac{k \ln \left(\frac{d}{k}\right)}{n} + C'''\sigma^2\frac{n}{(\sqrt{d} - 2\sqrt{n})^2} + \sigma^2 
\end{align}
for universal constants $C'', C''' > 0$.
Such a hybrid interpolator is \textbf{order-optimal among all interpolating solutions.}
\end{corollary}

To see that the rate in Equation~\eqref{eq:minimaxoptimal} is order-optimal, note that any interpolating solution would need to incur error at least $\sigma^2 \frac{n}{(\sqrt{d} + 2\sqrt{n})^2}$ due to the effect of fitting pure noise (from the lower bound part of Corollary~\ref{cor:crabpot}).
On the other hand, we know that \textit{any} sparse-signal estimator, regardless of whether it interpolates or not, would need to incur error at least $\sigma^2\frac{k \ln \left(\frac{d}{k}\right)}{n}$.
Thus, the test MSE of any interpolating solution on at least one instance in the sparse regime would have to be at least
\begin{align*}
\sigma^2 \max\left\{\frac{n}{(\sqrt{d} + 2\sqrt{n})^2},   \frac{k \ln \left(\frac{d}{k}\right)}{n} \right\} \geq \frac{\sigma^2}{2}\left(\frac{n}{(\sqrt{d} + 2\sqrt{n})^2} + \frac{k \ln \left(\frac{d}{k}\right)}{n}\right)
\end{align*}

which exactly matches the rate in Equation~\eqref{eq:minimaxoptimal} upto constants.

One example of such an order-optimal estimator is the SLOPE estimator which was recently analyzed~\cite{bellec2018slope}.
If one is willing to tolerate a slightly slower rate of $\Oh\left(\sigma^2 \cdot \frac{k\ln d}{n}\right)$ for signal recovery, several other estimators can be used.
We state the recovery guarantee (informally) for two choices below.
Formal statements are proved in Appendix~\ref{app:sparsity}.

\begin{corollary}\label{cor:lassoomp}
Consider the estimators $\alphahat_1 := \alphahat_{\mathsf{Las.},1}$ and $\alphahat_1 := \alphahat_{\mathsf{OMP},1}$ obtained by Lasso and OMP respectively for suitable choices of regularizer and stopping rule\footnote{Details in Appendix~\ref{app:sparsity}.} respectively.
Assuming lower bounds on number of samples $n$ and signal strength respectively, the hybrid estimator based on either of these estimators then gives us test MSE
\begin{align*}
\testerr(\alphahat_{\mathsf{hybrid}}) &\leq C''\sigma^2\frac{k \ln d}{n} + C'''\sigma^2\frac{n}{d} + \sigma^2 
\end{align*}
for universal constants $C'', C''' > 0$.
\end{corollary}

Corollaries~\ref{cor:minimaxoptimal} and~\ref{cor:lassoomp} are proved in Appendix~\ref{app:sparsity} with explicit analysis of Lasso and OMP for completeness.
The signal recovery guarantees are slightly different for the Lagrangian Lasso and OMP - for the former, we restate a direct bound~\cite{bickel2009simultaneous} on the $\ell_2$-norm of the signal recovery error, which scales as given in the informal equation above.
For the latter, our error bound derives from a much stronger variable selection guarantee~\cite{cai2011orthogonal} which implicitly requires the SNR to scale as $\omega(\ln d)$.
For both, it is worth paying attention to the conditions strictly \textit{required} for successful recovery -- for the Lagrangian Lasso, the condition manifests as a lower bound on the number of samples $n$ as a function of $(k,d)$.
For OMP, the condition manifests as a lower bound on the absolute values of non-zero coefficients $\{\alpha^*_j\}, j \in \supp(\alphastar)$.
For further details on these conditions, see Appendix~\ref{app:sparsity}.

The examples of estimators provided so far for the first step of the hybrid interpolator require knowledge of the noise variance $\sigma^2$ -- they use it either to regularize appropriately, or to define an appropriate stopping rule.
While one could always estimate this parameter through cross-validation, there do exist successful estimators that do not have access to this side information.
One example is provided below.

\begin{corollary}\label{cor:sqrootlasso}
Consider the estimator $\alphahat_1 := \alphahat_{\mathsf{Sq.rt.Las.},1}$ which is based on the \textbf{square-root-Lasso} for suitable choice of regularizer \textbf{that does not depend on the noise variance $\sigma^2$} (details in Appendix~\ref{app:sparsity}).
The hybrid estimator then gives us test MSE
\begin{align*}
\testerr(\alphahat_{\mathsf{hybrid}}) &\leq C''\sigma^2\frac{k \ln d}{n} + C'''\sigma^2\frac{n}{d} + \sigma^2 
\end{align*}
for universal constants $C'', C''' > 0$.
\end{corollary}

Regardless of which sparse signal estimator is used in the first step of the hybrid interpolator (Definition~\ref{def:hybrid}), the qualitative story is the same: there is a tradeoff in how to overparameterize, i.e. how to set $d >> n$.
As we increase the number of features in the family, the error arising from fitting noise goes down as $1/d$ - but these ``fake features" also have the potential to be falsely discovered\footnote{This same idea inspired recent work on explicitly constructing fake ``knockoff" features to draw away energy from features that have the potential to be falsely discovered~\cite{barber2015controlling}.}, thus driving up the signal recovery error rate logarithmically in $d$.
This logarithmic-linear tradeoff still ensures that the best test error is achieved when we sizeably overparameterize, even if not at $d = \infty$ like if we were only fitting noise.
Figure~\ref{fig:RG_sparse_MSE} considers an example with iid Gaussian design and true sparsity level $500$ for various ranges of noise variance: $\sigma^2 = 10^{-4}, 10^{-2}$ and $10^{-1}$.
For all these ranges, we observe that hybrid interpolators' test MSE closely track the ideal MSE.

\section{Conclusions for high-dimensional generative model}
Our results, together with other work in the last year~\cite{hastie2019surprises,belkin2019two,bartlett2019benign,bibas2019new}, provide a significant understanding of the ramifications of selecting interpolating solutions of \textit{noisy} data generated from a high-dimensional, or \textit{overparameterized} linear model, i.e. where the number of features used far exceeds the number of samples of training data.
Key takeaways are summarized below:

\begin{enumerate}
\item While ``denoising" the output is always strictly preferred to simply interpolating both signal and noise, the additional price of this interpolation on noise can be minimal when the regime is extremely overparameterized.
The price decays to $0$ as the level of overparameterization goes to infinity. 
We can now intuitively see that this is a consequence of the ability of aliased features to absorb and dissipate training noise energy --- figuratively spreading it out over a much larger bandwidth. 
\item Realizing this potential for ``harmless interpolation of noise'' critically depends on the number of such aliased features, as shown for Gaussian feature families in~\cite{bartlett2019benign} for the minimum $\ell_2$-norm interpolator, and this in fact holds completely generally even for ideal interpolation. Overparameterization is \textit{necessary}, for harmless interpolation.
\item The $\ell_2$-minimizing interpolator of noisy data with whitened features successfully absorbs noise energy for maximal harmlessness, but also dissipates signal across multiple aliases, as also shown in~\cite{hastie2019surprises}. 
To prevent this ``bleeding of signal", there needs to be a strong implicit bias towards a small, \textit{a-priori known} set of pre-determined feature directions.
Moreover, as mentioned in~\cite{bartlett2019benign}, to be successful, we should only care about measuring generalization error along these pre-determined feature directions --- i.e.~the true signal needs to be in this known subspace to be reliably learned by the $\ell_2$-min interpolator.
\item The $\ell_1$-minimizing interpolator of \textit{noisy} data suffers from the reverse problem: it more successfully preserves signal (when indeed the true signal is sparsely representable), but does not absorb training noise harmlessly.
\item Interpolators in the sparse linear model satisfying the diametrically opposite goals of \textit{sparse signal preservation} and \textit{noise energy absorption} exist, and can be understood as having a two-step, ``hybrid" nature. They first fit the signal as best they can and then interpolate everything in a way that spreads noise out. This gives asymptotic rates that are, in an order sense, the best that could be hoped for.
\end{enumerate}

These takeaways paint a more measured picture for the practical use of interpolation in \textit{noisy} high-dimensional linear regression.
While it is intellectually interesting that interpolation need not be harmful, it is practically always suboptimal to regularization with denoising.
This was pointed out in~\cite[Sections $6$ and $7$]{hastie2019surprises} for the $\ell_2$-minimizing interpolator \textit{vis-a-vis} Tikhonov regularization at the optimal level (and this optimal level can be estimated using cross-validation). Our discussion of Tikhonov and ridge regularization shows that this is because ridge regression can be interpreted as a further level of overparameterization, except using features that manifest as $0$ during test-time. The
best possible interpolating estimators are an additional $\Oh\left(\sigma^2 \frac{n}{d}\right)$ worse in test MSE than their optimal denoising counterparts. 

Moreover, because of the fundamental tradeoff between signal preservation and noise absorption, constructed consistent estimators in a two-step procedure by building on optimal regularizers to further fit noise -- a second step that indeed appears rather artificial and unnecessary. 
A conceivable benefit to interpolating solutions was their lack of dependence on prior knowledge of the noise variance $\sigma^2$; however, regularizing estimators are also now known to work in the absence of this knowledge -- either by modifying the optimization objective to make it SNR-invariant~\cite{belloni2011square}, or by first estimating the SNR~\cite{verzelen2018adaptive}. 

\subsection{Future Directions}

The benefits of overparameterization, however, deserve considerable future attention and scrutiny -- regardless of whether the solutions interpolate or regularize.
The missing ingredient in our discussion above is the potential for \textit{model misspecification}, i.e. data that can only be approximated by the high-dimensional linear model. (The example~\ref{eg:wiggly} in Section~\ref{sec:l2} is just a caricature in that direction.)
In the empirical double descent papers~\cite{geiger2018jamming,belkin2019reconciling}, the second descent of the test MSE in the overparameterized regime is \textit{strictly better} than the first descent in the underparameterized regime.
From our concrete understanding of the correctly specified high-dimensional regime, it is clear that this is only possible in the presence of an approximation-theoretic benefit of adding more features into the model.
Analyzing this benefit is very challenging generally, but in the special case when the features are independent,~\cite[Theorem $5$]{hastie2019surprises} provides asymptotic minimal scalings for the approximation-theoretic benefit under which a double descent curve can be obtained with the $\ell_2$-minimizing interpolator.
Even here, the global minimum of the empirical risk is \textit{not} at infinite overparameterization.
Very recetly, Mei and Montanari~\cite{mei2019generalization} provided a promising result: they theoretically recovered the second descent of the double-descent curve for the Tikhonov-regularized solution (note that this is \textit{not} an interpolator) on random Fourier features, which are clearly not independent.
More work remains to be done to characterize necessary \textit{and} sufficient conditions for a double descent in the empirical risk -- such conditions include the generative model for the data, the choice of featurization, the choice of estimator, and whether this estimator should even interpolate.

Finally, as also pointed out by~\cite{bartlett2019benign}, we are far from understanding the ramifications of overparameterization on generalization for the original case study that motivated recent interest in the overparameterized regime: deep neural networks with many more parameters than training points~\cite{neyshabur2014search,zhang2016understanding}.
In most practically used neural networks, the overparameterization includes both\textit{depth} rather than width, resulting in the presence of substantial non-linearity -- how this affects generalization, even in this intuitive and simple picture of signal preservation and noise absorption -- remains unclear.
As a starting point to investigate this case,~\cite[Section $8$]{hastie2019surprises} provides a model that interpolates between full linearity and non-linearity.

Quantifiable ramifications of interpolation and overparameterization on non-quadratic loss functions (such as those that appear in classification problems) would also be very interesting to understand, but the investigation here suggests that similar effects (bleeding, noise absorption, etc.) should exist there as well since the core intuition underlying them is that they are consequences of aliasing. The signal-processing perspective here actually suggests a pair of families of meta-conjectures that can help drive forward our understanding. On the one hand, we can use the lens of regularly-sampled training data and Fourier features to understand phenomena of interest whenever we seek to understand the significantly overparameterized regime in the context of inference algorithms that are minimum 2-norm in nature. The resulting conjectures can then be validated in Gaussian models and beyond. But perhaps even more interesting is the reverse direction --- our lens suggests that the many phenomena and techniques developed over the decades in the regularly-sampled Fourier world should have counterparts in the world of overparameterized learning.

\section*{Acknowledgments}

We thank Peter L Bartlett for insightful initial discussions, and Aaditya Ramdas for thoughtful questions about an earlier version of this work.
We would like to acknowledge the staff of EECS16A and EECS16B at Berkeley for in part inspiring the initial ideas, as well as the staff of EECS189.

We thank the anonymous reviewers of IEEE ISIT 2019 for useful feedback, and both the information theory community and the community at the Simons Institute Foundations of Deep Learning program, Summer 2019, for several stimulating discussions that improved the presentation of this paper -- especially Andrew Barron, Mikhail Belkin, Shai Ben-David, Meir Feder, Suriya Gunasekar, Daniel Hsu, Tara Javidi, Ashwin Pananjady, Shlomo Shamai, Nathan Srebro, and Ram Zamir.

Last but not least, we acknowledge the support of the ML4Wireless center member companies and NSF grants AST-144078 and ECCS-1343398.

\bibliographystyle{IEEEtran}
\bibliography{references.bib}

\begin{thebibliography}{10}
\providecommand{\url}[1]{#1}
\csname url@samestyle\endcsname
\providecommand{\newblock}{\relax}
\providecommand{\bibinfo}[2]{#2}
\providecommand{\BIBentrySTDinterwordspacing}{\spaceskip=0pt\relax}
\providecommand{\BIBentryALTinterwordstretchfactor}{4}
\providecommand{\BIBentryALTinterwordspacing}{\spaceskip=\fontdimen2\font plus
\BIBentryALTinterwordstretchfactor\fontdimen3\font minus
  \fontdimen4\font\relax}
\providecommand{\BIBforeignlanguage}[2]{{%
\expandafter\ifx\csname l@#1\endcsname\relax
\typeout{** WARNING: IEEEtran.bst: No hyphenation pattern has been}%
\typeout{** loaded for the language `#1'. Using the pattern for}%
\typeout{** the default language instead.}%
\else
\language=\csname l@#1\endcsname
\fi
#2}}
\providecommand{\BIBdecl}{\relax}
\BIBdecl

\bibitem{friedman2001elements}
J.~Friedman, T.~Hastie, and R.~Tibshirani, \emph{The {E}lements of
  {S}tatistical {L}earning}, vol.~1, no.~10.

\bibitem{vapnik1999overview}
V.~N. Vapnik, ``An overview of statistical learning theory,'' \emph{IEEE
  transactions on {N}eural {N}etworks}, vol.~10, no.~5, pp. 988--999, 1999.

\bibitem{bartlett2005local}
P.~L. Bartlett, O.~Bousquet, S.~Mendelson \emph{et~al.}, ``Local rademacher
  complexities,'' \emph{The Annals of Statistics}, vol.~33, no.~4, pp.
  1497--1537, 2005.

\bibitem{zhang2016understanding}
C.~Zhang, S.~Bengio, M.~Hardt, B.~Recht, and O.~Vinyals, ``Understanding deep
  learning requires rethinking generalization,'' \emph{arXiv preprint
  arXiv:1611.03530}, 2016.

\bibitem{miller2002subset}
A.~Miller, \emph{Subset selection in regression}.\hskip 1em plus 0.5em minus
  0.4em\relax Chapman and Hall/CRC, 2002.

\bibitem{hastie2019surprises}
T.~Hastie, A.~Montanari, S.~Rosset, and R.~J. Tibshirani, ``Surprises in
  high-dimensional ridgeless least squares interpolation,'' \emph{arXiv
  preprint arXiv:1903.08560}, 2019.

\bibitem{belkin2019two}
M.~Belkin, D.~Hsu, and J.~Xu, ``Two models of double descent for weak
  features,'' \emph{arXiv preprint arXiv:1903.07571}, 2019.

\bibitem{bartlett2019benign}
P.~L. Bartlett, P.~M. Long, G.~Lugosi, and A.~Tsigler, ``Benign {O}verfitting
  in {L}inear {R}egression,'' \emph{arXiv preprint arXiv:1906.11300}, 2019.

\bibitem{bibas2019new}
K.~Bibas, Y.~Fogel, and M.~Feder, ``A new look at an old problem: A universal
  learning approach to linear regression,'' \emph{arXiv preprint
  arXiv:1905.04708}, 2019.

\bibitem{mitra2019understanding}
P.~P. Mitra, ``Understanding overfitting peaks in generalization error:
  Analytical risk curves for $ l\_2 $ and $ l\_1 $ penalized interpolation,''
  \emph{arXiv preprint arXiv:1906.03667}, 2019.

\bibitem{neyshabur2014search}
B.~Neyshabur, R.~Tomioka, and N.~Srebro, ``In search of the real inductive
  bias: On the role of implicit regularization in deep learning,'' \emph{arXiv
  preprint arXiv:1412.6614}, 2014.

\bibitem{allen2019convergence}
Z.~Allen-Zhu, Y.~Li, and Z.~Song, ``A convergence theory for deep learning via
  over-parameterization,'' in \emph{International Conference on Machine
  Learning}, 2019, pp. 242--252.

\bibitem{allen2018learning}
Z.~Allen-Zhu, Y.~Li, and Y.~Liang, ``Learning and generalization in
  overparameterized neural networks, going beyond two layers,'' \emph{arXiv
  preprint arXiv:1811.04918}, 2018.

\bibitem{azizan2019stochastic}
N.~Azizan, S.~Lale, and B.~Hassibi, ``Stochastic mirror descent on
  overparameterized nonlinear models: Convergence, implicit regularization, and
  generalization,'' \emph{arXiv preprint arXiv:1906.03830}, 2019.

\bibitem{chizat2018global}
L.~Chizat and F.~Bach, ``On the global convergence of gradient descent for
  over-parameterized models using optimal transport,'' in \emph{Advances in
  neural information processing systems}, 2018, pp. 3036--3046.

\bibitem{du2018gradient}
S.~S. Du, X.~Zhai, B.~Poczos, and A.~Singh, ``Gradient descent provably
  optimizes over-parameterized neural networks,'' \emph{arXiv preprint
  arXiv:1810.02054}, 2018.

\bibitem{mei2018mean}
S.~Mei, A.~Montanari, and P.-M. Nguyen, ``A mean field view of the landscape of
  two-layer neural networks,'' \emph{Proceedings of the National Academy of
  Sciences}, vol. 115, no.~33, pp. E7665--E7671, 2018.

\bibitem{soltanolkotabi2018theoretical}
M.~Soltanolkotabi, A.~Javanmard, and J.~D. Lee, ``Theoretical insights into the
  optimization landscape of over-parameterized shallow neural networks,''
  \emph{IEEE Transactions on Information Theory}, 2018.

\bibitem{soudry2018implicit}
D.~Soudry, E.~Hoffer, M.~S. Nacson, S.~Gunasekar, and N.~Srebro, ``The implicit
  bias of gradient descent on separable data,'' \emph{The Journal of Machine
  Learning Research}, vol.~19, no.~1, pp. 2822--2878, 2018.

\bibitem{gunasekar2018characterizing}
S.~Gunasekar, J.~Lee, D.~Soudry, and N.~Srebro, ``Characterizing implicit bias
  in terms of optimization geometry,'' \emph{arXiv preprint arXiv:1802.08246},
  2018.

\bibitem{nacson2019convergence}
M.~S. Nacson, J.~Lee, S.~Gunasekar, P.~H.~P. Savarese, N.~Srebro, and
  D.~Soudry, ``Convergence of gradient descent on separable data,'' in
  \emph{The 22nd International Conference on Artificial Intelligence and
  Statistics}, 2019, pp. 3420--3428.

\bibitem{woodworth2019kernel}
B.~Woodworth, S.~Gunasekar, J.~Lee, D.~Soudry, and N.~Srebro, ``Kernel and deep
  regimes in overparametrized models,'' \emph{arXiv preprint arXiv:1906.05827},
  2019.

\bibitem{wilson2017marginal}
A.~C. Wilson, R.~Roelofs, M.~Stern, N.~Srebro, and B.~Recht, ``The marginal
  value of adaptive gradient methods in machine learning,'' in \emph{Advances
  in Neural Information Processing Systems}, 2017, pp. 4148--4158.

\bibitem{shah2018minimum}
V.~Shah, A.~Kyrillidis, and S.~Sanghavi, ``Minimum norm solutions do not always
  generalize well for over-parameterized problems,'' \emph{arXiv preprint
  arXiv:1811.07055}, 2018.

\bibitem{neyshabur2015norm}
B.~Neyshabur, R.~Tomioka, and N.~Srebro, ``Norm-based capacity control in
  neural networks,'' in \emph{Conference on Learning Theory}, 2015, pp.
  1376--1401.

\bibitem{bartlett2017spectrally}
P.~L. Bartlett, D.~J. Foster, and M.~J. Telgarsky, ``Spectrally-normalized
  margin bounds for neural networks,'' in \emph{Advances in Neural Information
  Processing Systems}, 2017, pp. 6240--6249.

\bibitem{golowich2017size}
N.~Golowich, A.~Rakhlin, and O.~Shamir, ``Size-independent sample complexity of
  neural networks,'' \emph{arXiv preprint arXiv:1712.06541}, 2017.

\bibitem{neyshabur2017exploring}
B.~Neyshabur, S.~Bhojanapalli, D.~McAllester, and N.~Srebro, ``Exploring
  generalization in deep learning,'' in \emph{Advances in Neural Information
  Processing Systems}, 2017, pp. 5947--5956.

\bibitem{schapire1998boosting}
R.~E. Schapire, Y.~Freund, P.~Bartlett, W.~S. Lee \emph{et~al.}, ``Boosting the
  margin: A new explanation for the effectiveness of voting methods,''
  \emph{The annals of statistics}, vol.~26, no.~5, pp. 1651--1686, 1998.

\bibitem{wei2018margin}
C.~Wei, J.~D. Lee, Q.~Liu, and T.~Ma, ``On the margin theory of feedforward
  neural networks,'' \emph{arXiv preprint arXiv:1810.05369}, 2018.

\bibitem{wyner2017explaining}
A.~J. Wyner, M.~Olson, J.~Bleich, and D.~Mease, ``Explaining the success of
  adaboost and random forests as interpolating classifiers,'' \emph{The Journal
  of Machine Learning Research}, vol.~18, no.~1, pp. 1558--1590, 2017.

\bibitem{belkin2018understand}
M.~Belkin, S.~Ma, and S.~Mandal, ``To understand deep learning we need to
  understand kernel learning,'' in \emph{International Conference on Machine
  Learning}, 2018, pp. 540--548.

\bibitem{belkin2018overfitting}
M.~Belkin, D.~J. Hsu, and P.~Mitra, ``Overfitting or perfect fitting? risk
  bounds for classification and regression rules that interpolate,'' in
  \emph{Advances in Neural Information Processing Systems}, 2018, pp.
  2300--2311.

\bibitem{belkin2019does}
M.~Belkin, A.~Rakhlin, and A.~B. Tsybakov, ``Does data interpolation contradict
  statistical optimality?'' in \emph{The 22nd International Conference on
  Artificial Intelligence and Statistics}, 2019, pp. 1611--1619.

\bibitem{liang2018just}
T.~Liang and A.~Rakhlin, ``Just interpolate: Kernel ''ridgeless" regression can
  generalize,'' \emph{arXiv preprint arXiv:1808.00387}, 2018.

\bibitem{rakhlin2018consistency}
A.~Rakhlin and X.~Zhai, ``Consistency of interpolation with laplace kernels is
  a high-dimensional phenomenon,'' \emph{arXiv preprint arXiv:1812.11167},
  2018.

\bibitem{geiger2018jamming}
M.~Geiger, S.~Spigler, S.~d'Ascoli, L.~Sagun, M.~Baity-Jesi, G.~Biroli, and
  M.~Wyart, ``The jamming transition as a paradigm to understand the loss
  landscape of deep neural networks,'' \emph{arXiv preprint arXiv:1809.09349},
  2018.

\bibitem{belkin2019reconciling}
M.~Belkin, D.~Hsu, S.~Ma, and S.~Mandal, ``Reconciling modern machine-learning
  practice and the classical bias--variance trade-off,'' \emph{Proceedings of
  the National Academy of Sciences}, vol. 116, no.~32, pp. 15\,849--15\,854,
  2019.

\bibitem{rahimi2008random}
A.~Rahimi and B.~Recht, ``Random features for large-scale kernel machines,'' in
  \emph{Advances in neural information processing systems}, 2008, pp.
  1177--1184.

\bibitem{hsu2012random}
D.~Hsu, S.~M. Kakade, and T.~Zhang, ``Random design analysis of ridge
  regression,'' in \emph{Conference on Learning Theory}, 2012, pp. 9--1.

\bibitem{wainwright2009information}
M.~J. Wainwright, ``Information-theoretic limits on sparsity recovery in the
  high-dimensional and noisy setting,'' \emph{IEEE Transactions on Information
  Theory}, vol.~55, no.~12, pp. 5728--5741, 2009.

\bibitem{aeron2010information}
S.~Aeron, V.~Saligrama, and M.~Zhao, ``Information theoretic bounds for
  compressed sensing,'' \emph{IEEE Transactions on Information Theory},
  vol.~56, no.~10, pp. 5111--5130, 2010.

\bibitem{pati1993orthogonal}
Y.~C. Pati, R.~Rezaiifar, and P.~S. Krishnaprasad, ``Orthogonal matching
  pursuit: Recursive function approximation with applications to wavelet
  decomposition,'' in \emph{Signals, Systems and Computers, 1993. 1993
  Conference Record of The Twenty-Seventh Asilomar Conference on}.\hskip 1em
  plus 0.5em minus 0.4em\relax IEEE, 1993, pp. 40--44.

\bibitem{donoho2012sparse}
D.~L. Donoho, Y.~Tsaig, I.~Drori, and J.-L. Starck, ``Sparse solution of
  underdetermined systems of linear equations by stagewise orthogonal matching
  pursuit,'' \emph{IEEE Transactions on Information Theory}, vol.~58, no.~2,
  pp. 1094--1121, 2012.

\bibitem{chen2001atomic}
S.~S. Chen, D.~L. Donoho, and M.~A. Saunders, ``Atomic decomposition by basis
  pursuit,'' \emph{SIAM review}, vol.~43, no.~1, pp. 129--159, 2001.

\bibitem{mei2019generalization}
S.~Mei and A.~Montanari, ``The generalization error of random features
  regression: {P}recise asymptotics and double descent curve,'' \emph{arXiv
  preprint arXiv:1908.05355}, 2019.

\bibitem{vershynin2010introduction}
R.~Vershynin, ``Introduction to the non-asymptotic analysis of random
  matrices,'' \emph{arXiv preprint arXiv:1011.3027}, 2010.

\bibitem{davidson2001local}
K.~R. Davidson and S.~J. Szarek, ``Local operator theory, random matrices and
  {B}anach spaces,'' \emph{Handbook of the geometry of Banach spaces}, vol.~1,
  no. 317-366, p. 131, 2001.

\bibitem{wainwright2019high}
M.~J. Wainwright, \emph{High-dimensional statistics: {A} non-asymptotic
  viewpoint}.\hskip 1em plus 0.5em minus 0.4em\relax Cambridge University
  Press, 2019, vol.~48.

\bibitem{hanson1971bound}
D.~L. Hanson and F.~T. Wright, ``A bound on tail probabilities for quadratic
  forms in independent random variables,'' \emph{The Annals of Mathematical
  Statistics}, vol.~42, no.~3, pp. 1079--1083, 1971.

\bibitem{rauhut2010compressive}
H.~Rauhut, ``Compressive sensing and structured random matrices,''
  \emph{Theoretical foundations and numerical methods for sparse recovery},
  vol.~9, pp. 1--92, 2010.

\bibitem{belkin2018reconciling}
M.~Belkin, D.~Hsu, S.~Ma, and S.~Mandal, ``Reconciling modern machine learning
  and the bias-variance trade-off,'' \emph{arXiv preprint arXiv:1812.11118},
  2018.

\bibitem{belkin2018does}
M.~Belkin, A.~Rakhlin, and A.~B. Tsybakov, ``Does data interpolation contradict
  statistical optimality?'' \emph{arXiv preprint arXiv:1806.09471}, 2018.

\bibitem{saligrama2011thresholded}
V.~Saligrama and M.~Zhao, ``Thresholded basis pursuit: Lp algorithm for
  order-wise optimal support recovery for sparse and approximately sparse
  signals from noisy random measurements,'' \emph{IEEE Transactions on
  Information Theory}, vol.~57, no.~3, pp. 1567--1586, 2011.

\bibitem{bellec2018slope}
P.~C. Bellec, G.~Lecu{\'e}, A.~B. Tsybakov \emph{et~al.}, ``Slope meets
  {L}asso: improved oracle bounds and optimality,'' \emph{The Annals of
  Statistics}, vol.~46, no.~6B, pp. 3603--3642, 2018.

\bibitem{bickel2009simultaneous}
P.~J. Bickel, Y.~Ritov, A.~B. Tsybakov \emph{et~al.}, ``Simultaneous analysis
  of {L}asso and {D}antzig selector,'' \emph{The Annals of Statistics},
  vol.~37, no.~4, pp. 1705--1732, 2009.

\bibitem{cai2011orthogonal}
T.~T. Cai and L.~Wang, ``Orthogonal matching pursuit for sparse signal recovery
  with noise,'' \emph{IEEE Transactions on Information theory}, vol.~57, no.~7,
  pp. 4680--4688, 2011.

\bibitem{barber2015controlling}
R.~F. Barber, E.~J. Cand{\`e}s \emph{et~al.}, ``Controlling the false discovery
  rate via knockoffs,'' \emph{The Annals of Statistics}, vol.~43, no.~5, pp.
  2055--2085, 2015.

\bibitem{belloni2011square}
A.~Belloni, V.~Chernozhukov, and L.~Wang, ``Square-root lasso: pivotal recovery
  of sparse signals via conic programming,'' \emph{Biometrika}, vol.~98, no.~4,
  pp. 791--806, 2011.

\bibitem{verzelen2018adaptive}
N.~Verzelen, E.~Gassiat \emph{et~al.}, ``Adaptive estimation of
  high-dimensional signal-to-noise ratios,'' \emph{Bernoulli}, vol.~24, no.~4B,
  pp. 3683--3710, 2018.

\bibitem{Bertsimas:1997:ILO:548834}
D.~Bertsimas and J.~Tsitsiklis, \emph{Introduction to Linear Optimization},
  1st~ed.\hskip 1em plus 0.5em minus 0.4em\relax Athena Scientific, 1997.

\bibitem{doi:10.1080/00029890.1994.11996972}
\BIBentryALTinterwordspacing
C.-K. Li and W.~So, ``Isometries of lp-norm,'' \emph{The American Mathematical
  Monthly}, vol. 101, no.~5, pp. 452--453, 1994. [Online]. Available:
  \url{https://doi.org/10.1080/00029890.1994.11996972}
\BIBentrySTDinterwordspacing

\bibitem{raskutti2010restricted}
G.~Raskutti, M.~J. Wainwright, and B.~Yu, ``Restricted eigenvalue properties
  for correlated {G}aussian designs,'' \emph{Journal of Machine Learning
  Research}, vol.~11, no. Aug, pp. 2241--2259, 2010.

\bibitem{tropp2007signal}
J.~A. Tropp and A.~C. Gilbert, ``Signal recovery from random measurements via
  orthogonal matching pursuit,'' \emph{IEEE Transactions on information
  theory}, vol.~53, no.~12, pp. 4655--4666, 2007.

\bibitem{fletcher2009necessary}
A.~K. Fletcher, S.~Rangan, and V.~K. Goyal, ``Necessary and sufficient
  conditions for sparsity pattern recovery,'' \emph{IEEE Transactions on
  Information Theory}, vol.~55, no.~12, pp. 5758--5772, 2009.

\bibitem{edelman1988eigenvalues}
A.~Edelman, ``Eigenvalues and condition numbers of random matrices,''
  \emph{SIAM Journal on Matrix Analysis and Applications}, vol.~9, no.~4, pp.
  543--560, 1988.

\bibitem{szarek1991condition}
S.~J. Szarek, ``Condition numbers of random matrices,'' \emph{Journal of
  Complexity}, vol.~7, no.~2, pp. 131--149, 1991.

\bibitem{rudelson2008littlewood}
M.~Rudelson and R.~Vershynin, ``The {L}ittlewood--{O}fford problem and
  invertibility of random matrices,'' \emph{Advances in Mathematics}, vol. 218,
  no.~2, pp. 600--633, 2008.

\end{thebibliography}

\newpage
\appendix

\section{Supplemental proofs}

We start by formally stating the matrix concentration results that are used in the proof of Corollary~\ref{cor:fundamentalprice} and~\ref{cor:crabpot}.

\subsection{Matrix concentration lemmas}\label{app:technical}

We begin with matrices satisfying sub-Gaussianity of rows (Assumption~\ref{as:ideal2}) and then consider matrices whose rows are heavy-tailed (Assumption~\ref{as:ideal1}).

\subsubsection{Supplemental lemmas for sub-Gaussian case}

We define a sub-Gaussian random variable and a sub-Gaussian random vector\footnote{Common examples of sub-Gaussian random variables are Gaussian, Bernoulli and all bounded random variables.} below.
\begin{definition}\label{def:subgaussian}
A random variable $X$ is sub-Gaussian with parameter at most $K < \infty$ if for all $p \geq 1$, we have
\begin{align*}
    p^{-1/2} \EE[|X|^p]^{1/p} \leq K .
\end{align*}
Further, a random vector $\Xvec$ is sub-Gaussian with parameter at most $K$ if for every (fixed) vector $\vvec$, the random variable $\frac{\inprod{\Xvec}{\vvec}}{\vecnorm{\vvec}{2}}$ is sub-Gaussian with parameter at most $K$.
\end{definition}

We cite the following lemma for sub-Gaussian matrix concentration.
\begin{lemma}[Theorem $5.39$, ~\cite{vershynin2010introduction}]\label{lem:matrixconcentration}
Let $\Bmat$ be a $n \times d$ matrix whose rows $\{\bvec_i\}_{i=1}^n$ are independent sub-Gaussian isotropic random vectors in $\reals^d$.
Then, for every $t > 0$ we have the following inequality:
\begin{align*}
    \sqrt{n} - C_K \sqrt{d} - t \leq \sigma_{min}(\Bmat) \leq \sigma_{max}(\Bmat) \leq \sqrt{n} + C_K \sqrt{d} + t
\end{align*}
with probability at least $1 - 2e^{-c_K t^2}$, where $C_K, c_K$ depend only on the sub-Gaussian parameter $K$ of the columns.
\end{lemma}

To prove Corollary~\ref{cor:fundamentalprice} for matrices $\Btrain$ satisfying the sub-Gaussian Assumption~\ref{as:ideal2}, we apply Lemma~\ref{lem:matrixconcentration} for the matrix $\Bmat := \Btrain$ itself.
We recall that we lower bounded the ideal test MSE as 
\begin{align*}
\testerr^* &\geq \frac{\vecnorm{\Wtrain}{2}^2}{\lambda_{max}(\Btrain \Btrain^\top)} \\
&= \frac{\vecnorm{\Wtrain}{2}^2}{\lambda_{max}(\Btrain^\top \Btrain)} \\
&= \frac{\vecnorm{\Wtrain}{2}^2}{\sigma^2_{max}(\Btrain)}
\end{align*}

and then substituting Lemma~\ref{lem:matrixconcentration} for the quantity $\sigma_{max}(\Btrain)$ with $t := \sqrt{n}$, together with the chi-squared tail bound on $\vecnorm{\Wtrain}{2}^2$ yields 
\begin{align*}
\testerr^* &\geq \frac{n\sigma^2(1 - \delta)}{(\sqrt{n} + (C_K + 1) \sqrt{d})^2} ,
\end{align*}
which is precisely Equation~\eqref{eq:fp_ideal2} with probability at least $(1 - e^{-c_K n} - e^{-n\delta^2/8})$.
This completes the proof of Corollary~\ref{cor:fundamentalprice} for sub-Gaussian feature vectors.
\qed

Before we move on to the heavy-tailed case, we remark on why we were not able to prove a version of Corollary~\ref{cor:crabpot} for the case of the feature \textit{vectors} (rows of $\Btrain$) being sub-Gaussian.
First, recall that to upper bound the ideal test MSE $\testerr^*$, one needs to lower bound the minimum singular value of $\Btrain$ (or equivalently $\Btrain^\top$).
Lemma~\ref{lem:matrixconcentration} provides only a vacuous bound for the minimum singular value, as $n < d \leq C_K \sqrt{d}$ in general.
Vershynin~\cite[Theorem $5.58$]{vershynin2010introduction} does provide a concentration bound for tall matrices with independent, sub-Gaussian columns -- which would be suitable to lower bound the minimum singular value of $\Btrain^\top$.
However, this result uses an (unrealistically) restrictive condition of needing the columns to be exactly normalized as $\vecnorm{\bvec_i}{2} = \sqrt{d}$ almost surely.
Vershynin also provides an explicit example of a random tall matrix with sub-Gaussian columns that violates this condition, for which the minimum singular value does \textit{not} concentrate.
This suggests that sub-Gaussianity of high-dimensional feature \textit{vectors} by itself may be too weak a condition to expect concentration on the \textit{minimum} singular value, altogether a more delicate quantity.

However, we can prove Corollary~\ref{cor:crabpot} for the more special case where $\Btrain$ has independent, sub-Gaussian \textit{entries} of unit variance.
In other words, the random variable $B_{ij}$ is sub-Gaussian and unit variance, and the random variables $\{B_{ij}\}$ are independent.
We cite the following lemma for concentration of the \textit{minimum singular value} of such a random matrix $\Btrain$.

\begin{lemma}[Theorem $5.38$,~\cite{vershynin2010introduction}]\label{lem:matrixconcentration_iid}
For $d \geq n$, let $\Bmat$ be a $d \times n$ (or $n \times d$) random matrix whose entries are independent sub-Gaussian random variables with zero mean, unit variance, and sub-Gaussian parameter at most $K$.
Then, for $\epsilon \geq 0$, we have
\begin{align*}
\Pr\left[\sigma_{min}(\Btrain) \leq \epsilon (\sqrt{d} - \sqrt{n-1}) \right] \leq (C_K \epsilon)^{d - n + 1} + c^d
\end{align*}
where $C_K > 0$ and $c_K \in (0,1)$ depend only on the sub-Gaussian parameter $K$.
\end{lemma}

We apply Lemma~\ref{lem:matrixconcentration_iid} substituting $\epsilon = \frac{1}{2C_K}$.
Then, we get
\begin{align*}
\Pr\left[\sigma_{min}(\Btrain) \leq \frac{1}{2C_K} (\sqrt{d} - \sqrt{n-1}) \right] \leq \left(\frac{1}{2}\right)^{d - n + 1} + c^d \to 0 \text{ as } (d,n) \to \infty .
\end{align*}

Noting as before that
\begin{align*}
\testerr^* &\leq \frac{\vecnorm{\Wtrain}{2}^2}{\sigma^2_{min}(\Btrain)} ,
\end{align*}

and substituting the upper chi-squared tail bound on the quantity $\vecnorm{\Wtrain}{2}^2$ as well as the lower tail bound on $\sigma_{min}(\Btrain)$ above directly gives us the statement of Corollary~\ref{cor:crabpot} for the iid sub-Gaussian case.

\subsubsection{Supplemental lemmas for heavy-tailed case}

We cite the following lemma for heavy-tailed matrix concentration.
\begin{lemma}[Theorem $5.41$, ~\cite{vershynin2010introduction}]\label{lem:heavytailedconcentration}
Let $\Bmat$ be a $n \times d$ matrix whose rows $\{\bvec_i\}_{i=1}^n$ are independent, isotropic random vectors in $\reals^d$.
Let $\vecnorm{\bvec_i}{2} \leq \sqrt{n}$ almost surely for all $i \in [n]$.
Then, for every $t > 0$ we have the following inequality:
\begin{align*}
    \sqrt{n} - t \sqrt{d} \leq \sigma_{min}(\Bmat) \leq \sigma_{max}(\Bmat) \leq \sqrt{n} + t \sqrt{d} 
\end{align*}
with probability at least $1 - 2ne^{-ct^2}$ for some positive constant $c > 0$.
\end{lemma}

We will use Lemma~\ref{lem:heavytailedconcentration} to prove our lower bound for whitened feature matrices $\Btrain$ satisfying Assumption~\ref{as:ideal1}.
Observe that Lemma~\ref{lem:heavytailedconcentration} contains the deviation parameter $t$ as a multiplicative constant as opposed to an additive one, which is typical of concentration results for heavy-tailed distributions.
Accordingly, the eventual scaling of the lower bound on the ideal test MSE, as well the extent of high probability in the result, will be accordingly weaker.

Substituting $t := \sqrt{\frac{2}{c}(\ln n)}$ yields the upper tail inequality
\begin{align*}
\sigma_{max}(\Bmat) \leq \sqrt{n} + \sqrt{\frac{2d}{c}\ln n}
\end{align*}

with probability at least $(1 - \frac{1}{n})$.
In a similar argument to the sub-Gaussian case, we can then lower bound the ideal test MSE as 
\begin{align*}
\testerr^* &\geq \frac{n\sigma^2(1 - \delta)}{(\sqrt{n} + C\sqrt{d \ln n})^2}
\end{align*}

for constant $C > 0$, completing the proof of Corollary~\ref{cor:fundamentalprice} for heavy-tailed random feature vectors.
\qed

As with the sub-Gaussian case, controlling the minimum singular value is much more delicate and in general it is not well-controlled.
Vershynin~\cite[Theorem $5.65$]{vershynin2010introduction} also has a result for tall matrices with independent columns more generally -- but this result imposes even more conditions on top of the normalization requirement $\vecnorm{\bvec_i}{2} = \sqrt{d}$: the columns need to be sufficiently incoherent in a pairwise sense.

\subsection{Proof of Corollary~\ref{cor:bp}}\label{app:bp}
\newcommand{\A} {\Atrain}

\newcommand{\w} {\boldsymbol \alpha}
\newcommand{\W} {W}
\newcommand{\V} {\ensuremath{\mathbf V}}
\newcommand{\y} {\mathbf{y}}
\newcommand{\mc} {\mathcal}
\newcommand{\sw} {\alpha^*}
\newcommand{\tran}{^{\top}}
\newcommand{\AS}{\mathbf{A_S}}
\newcommand{\Q}{\mathbf{Q}}
\newcommand{\Qt}{\mathbf{Q\tran}}
\newcommand{\R}[1]{\mathbb{R}^{#1}}
\newcommand{\event}{\mathcal{E}}
\newcommand{\q}{q}
\newcommand{\abold}{\ensuremath{\mathbf{a}}}
\newcommand{\abs}[1] {| #1 |}
\newcommand{\translate}{\leftrightarrow}
\newcommand{\subt}{\text{s.t.}}
\newcommand{\ones}{\mathbf{1}}
\newcommand{\norm}[1]{\|#1\|}
\newcommand{\inv}[1]{(#1)^{-1}}
\newcommand{\gperm}{\mathcal{P}^n}
\newcommand{\pbp}{\ensuremath{p^*_{BP}}}
\newcommand{\plp}{\ensuremath{p^*_{LP}}}

\newcommand{\indicator}[1]{\mathbb{I}\{#1\}}
\newcommand{\lpu}{\ensuremath{\mathbf{u}}}
\newcommand{\lpv}{\ensuremath{\mathbf{v}}}
\newcommand{\altu}{\ensuremath{\mathbf{\widetilde{u}}}}
\newcommand{\altv}{\ensuremath{\mathbf{\widetilde{v}}}}
\newcommand{\ustar}{\ensuremath{\mathbf{u^*}}}
\newcommand{\uhat}{\ensuremath{\widehat{\mathbf{u}}}}
\newcommand{\uone}{\ensuremath{\lpu_1}}
\newcommand{\utwo}{\ensuremath{\lpu_2}}
\newcommand{\vone}{\ensuremath{\lpv_1}}
\newcommand{\vtwo}{\ensuremath{\lpv_2}}
\newcommand{\alphaone}{\ensuremath{\alphabold_1}}
\newcommand{\alphatwo}{\ensuremath{\alphabold_2}}

\newcommand{\vstar}{\ensuremath{\mathbf{v^*}}}
\newcommand{\vhat}{\ensuremath{\widehat{\mathbf{v}}}}
\newcommand{\cstar}{\ensuremath{\mathbf{c^*}}}
\newcommand{\chat}{\ensuremath{\widehat{\mathbf{c}}}}
\newcommand{\altalpha}{\ensuremath{\tilde{\alphabold}}}
\newcommand{\alphastarpos}{\ensuremath{\alphastar_+}}
\newcommand{\alphastarneg}{\ensuremath{\alphastar_-}}
\newcommand{\Atrainsone}{\ensuremath{\Atrain(S_1)}}
\newcommand{\Atrainstwo}{\ensuremath{\Atrain(S_2)}}
\newcommand{\Atrains}{\ensuremath{\Atrain(S)}}
\newcommand{\veczero}{\ensuremath{\mathbf{0}}}
\newcommand\numeq[1]%
  {\stackrel{\scriptscriptstyle(\mkern-1.5mu#1\mkern-1.5mu)}{=}}

Recall from Section~\ref{sec:l1} that the basis pursuit program (BP) is given by
\begin{align}
    \pbp = \min & \: \norm{\alphabold}_1 \label{prog:bp}\\
    \subt \: &\Atrain \alphabold=\Wtrain  \nonumber .
\end{align}

Consider the following linear program (LP):
\begin{align}
    \plp = \min& \: {\sum_{i=1}^d \lpu_i + \sum_{i=1}^d \lpv_i}  \label{prog:lp}\\
    \subt \: &\begin{bmatrix}\Atrain & -\Atrain \end{bmatrix} \begin{bmatrix}\lpu \\ \lpv\end{bmatrix}=\Wtrain, \nonumber \\
    & \begin{bmatrix}\lpu \\\lpv\end{bmatrix} \geq \veczero . \nonumber 
\end{align}
Note that the concatenated vector $\begin{bmatrix}\lpu^\top & \lpv^\top \end{bmatrix} \in \reals^{2d}$.
Let $\alphahat$ refer to the optimal solution of BP in Equation~\eqref{prog:bp} and let $\begin{bmatrix}\uhat^\top & \vhat^\top\end{bmatrix}$ be the optimal solution of the LP in Equation~\eqref{prog:lp}.
A classical result is that the BP program in Equation~\eqref{prog:bp} and the LP in Equation~\eqref{prog:lp} are equivalent -- not only in the sense that the maximal objective value is the same for both, but also that the optimal solution(s) are equivalent up to a linear transformation.
We state the equivalence as a lemma and provide the entire proof argument for completeness.

\begin{lemma}\label{lemma:lpbpequiv}
The BP program \eqref{prog:bp} and the LP program \eqref{prog:lp} are equivalent in the following senses:
\begin{enumerate}
\item The maximal objectives are equal, i.e. $\pbp = \plp$.
\item Given an optimal solution $(\uhat, \vhat)$ for the LP, we can obtain an optimal solution for the BP program through the linear transformation $\alphahat = \uhat-\vhat$. 
\item Given an optimal solution $\alphahat$ for the the BP, we can obtain an optimal solution for the LP through the transformation
\begin{align*}
\widehat{u}_j &:= \widehat{\alpha}_j \Ind[\widehat{\alpha}_j \geq 0] \\
\widehat{v}_j &:= -\widehat{\alpha}_j \Ind[\widehat{\alpha}_j < 0] .
\end{align*}
\end{enumerate}
\end{lemma}

\begin{proof}
Denote the objective for the BP as $f(\alphabold) = \norm{\alphabold}_1$ and the objective for the LP as $g(\lpu, \lpv) = \sum_{i=1}^d \lpu_i + \sum_{i=1}^d \lpv_i$. \\

The crucial step is to show that a \textit{necessary} condition for a solution $\begin{bmatrix} \lpu^\top & \lpv^\top \end{bmatrix}$ to be optimal is that $\supp(\lpu) \cap \supp(\lpv) = \emptyset$, i.e. that the supports of the vectors $\lpu$ and $\lpv$ are disjoint.
This is stated formally in the lemma below.
\begin{lemma}\label{lemma:support_disjoint}
The supports of any optimal solution $\begin{bmatrix} \uhat^\top & \vhat^\top \end{bmatrix}$ are disjoint, i.e $\supp(\uhat) \cap \supp(\vhat) = \emptyset$.
\end{lemma}

Taking Lemma~\ref{lemma:support_disjoint} to be true for the moment, let us look at what it implies.
For any solution where $\lpu$ and $\lpv$ are disjoint in their support, we can construct vector $\alphabold := \lpu - \lpv$, and we note that
\begin{align*}
f(\alphabold) = \norm{\alphabold}_1 = \sum_{i=1}^d u_i + \sum_{i=1}^d v_i = g(\lpu,\lpv) .
\end{align*}

In the other direction, for any feasible solution $\alphabold$, we can \textit{uniquely} construct vectors $(\lpu, \lpv)$ as
\begin{align*}
u_j &:= \alpha_j \Ind[\alpha_j \geq 0] \text{ for } j \in [d]\\
v_j &:= -\alpha_j \Ind[\alpha_j < 0] \text{ for } j \in [d] ,
\end{align*}
and we will again have 
\begin{align*}
g(\lpu,\lpv) &= \sum_{i \in \supp(\lpu)} \alpha_j \Ind[\alpha_j \geq 0] + \sum_{i \in \supp(\lpv)}  -\alpha_j \Ind[\alpha_j < 0] \\
&= \vecnorm{\alphabold}{1} .
\end{align*}

Thus, we have established a \textit{one-one} equivalence between all feasible solutions for the BP program and all feasible solutions for the LP that satisfy the disjoint-support condition.
By Lemma~\ref{lemma:support_disjoint}, we know that all optimal solutions to the LP need to satisfy this condition.
Thus, we have proved equivalence between the BP program~\eqref{prog:bp} and the LP~\eqref{prog:lp}.



It only remains to prove Lemma~\ref{lemma:support_disjoint}, which we do below.

\begin{proof}
It suffices to show that any solution not satisfying the disjointness property is strictly suboptimal.
Suppose, for a candidate solution $\begin{bmatrix} \uhat^\top & \vhat^\top \end{bmatrix}$, there exists index $j \in [d]$ such that $ \widehat{u}_j > 0 \text{ and } \widehat{v}_j > 0$. Let $\widehat{c}_j = \widehat{u}_j - \widehat{v}_j$. \\ 
If $\widehat{c}_j > 0$, the alternate solution $\begin{bmatrix} \altu^\top & \altv^\top \end{bmatrix}$ with $\widetilde{u}_i = \widehat{u}_i, \widetilde{v}_i= \widehat{v}_i$ for $i \neq j$ and $\widetilde{u}_j = \widehat{c}_j, \widetilde{v}_j = 0$ will be feasible and have lower objective values. 
To see this, observe that 
\begin{align*}
   \vecnorm{\uhat}{1} + \vecnorm{\vhat}{1} - ( \vecnorm{\altu}{1} + \vecnorm{\altv}{1}) &= \abs{\widehat{u}_j} - \abs{\widehat{c}_j} + \abs{\widehat{v}_j} \\
   &= \abs{\widehat{u}_j} + \abs{\widehat{v}_j} - \abs{\widehat{u}_j - \widehat{v}_j} \\
   &> 0,
\end{align*}
where the last step follows since $\widehat{u}_j, \widehat{v}_j > 0$.
On the other hand, if $\widehat{c} < 0$, we can construct the alternative solution the alternate solution $\begin{bmatrix} \altu^\top & \altv^\top \end{bmatrix}$ with $\widetilde{u}_i = \widehat{u}_i, \widetilde{v}_i= \widehat{v}_i$ for $i \neq j$ and $\widetilde{u}_j = 0, \widetilde{v}_j = -\widehat{c}_j$.
Clearly, this solution is feasible and by a similar argument also has lower objective value.
This completes the proof.
\end{proof}

Since we have proved Lemma~\ref{lemma:support_disjoint}, we have completed our proof of equivalence.
\end{proof}

As a corollary of the above equivalence, it is clear that the support of any optimal solution for the BP program, $\alphahat$ is the same as the support of its equivalent solution $\begin{bmatrix} \uhat^\top & \vhat^\top\end{bmatrix}$ that is optimal for the LP.
That is, we have
\begin{align}\label{eq:suppeq}
\abs{\supp(\alphahat)} = \abs{\supp\left(\begin{bmatrix}\uhat^\top & \vhat^\top\end{bmatrix}\right)} = \abs{\supp(\uhat)} + \abs{\supp(\vhat)}.
\end{align}

As a consequence of the BP-LP equivalence, we can show through basic LP theory that there exists an optimal solution for the BP program whose support size is at most $n$.
We state and prove this lemma below.

\begin{lemma}
There exists an optimal solution for the BP program \eqref{prog:bp} for which the support size is at most $n$.\label{lemma:bp1}\\
\end{lemma}
\begin{proof}
First we show that there exists an optimal solution $\begin{bmatrix}\uhat^\top & \vhat^\top \end{bmatrix}$ to the equivalent LP~\eqref{prog:lp} with support size at most $n$. 
A basic feasible solution (BFS) for the LP is of the form $\begin{bmatrix} \lpu^\top & \lpv^\top \end{bmatrix}$ with support size at most $n$. 
It is well-known that for any LP, basic feasible solutions are extreme points of the feasible set~\cite{Bertsimas:1997:ILO:548834}. 
Now, from Bauer's maximum principle, we know that the minimum of a linear objective over a convex compact set is attained some extreme point of the feasible set, i.e. a BFS.
Thus, there exists an extreme point, i.e. BFS, $\begin{bmatrix} \uhat^\top &  \vhat^\top \end{bmatrix}$ that attains the minimum objective value for the LP and is thus an optimal solution. By the definition of BFS, this solution has support size at most $n$. Moreover, by the second statement of Lemma \ref{lemma:lpbpequiv}, $\alphahat = \uhat - \vhat$ is an optimal solution for the BP. 
From Equation~\eqref{eq:suppeq} we have
\begin{align*}
\abs{\supp(\alphahat)} = \abs{\supp\left(\begin{bmatrix}\uhat^\top & \vhat^\top \end{bmatrix}\right)} \leq n .
\end{align*}
This completes the proof. 
\end{proof}

Clearly, there exists \text{an} optimal solution of support size \textit{at most} $n$.
Generically, this optimal solution will be unique, and its support size will be \textit{exactly} $n$.
Showing these two properties is sufficient to show that BP is $1$-parsimonious (as defined in Definition~\ref{def:parsimonious}).
We conclude this section by proving these properties one-by-one for an appropriately non-degenerate matrix $\Atrain$, and Gaussian noise $\Wtrain \sim \NORMAL(\mathbf{0},\sigma^2\mathbf{I}_n)$.
The sense in which we define non-degeneracy of covariate matrix $\Atrain$ is in the following three assumptions, which are extremely weak and will in general hold for any random ensemble with probability $1$.
For any subset $S \subset [d]$, we denote the submatrix of $\Atrain$ corresponding to columns in subset $S$ by $\Atrain(S)$.

\begin{assumption}\label{as:nonzero}
For every subset $S$ of size $n$, the matrix $\Atrain(S) \neq \mathbf{O}$.
\end{assumption}

\begin{assumption}\label{as:invertibility}
For every subset $S$ of size $n$, the matrix $\Atrain(S)$ is invertible.
\end{assumption}

\begin{assumption}\label{as:nogperm}
For any two \textit{distinct} subsets $S_1, S_2$ of size $n$, the matrix $\inv{\Atrain(S_1)}\Atrain(S_2) \notin \gperm$, where $\gperm$refers to the set of \textbf{generalized permutation matrices} of size $n \times n$.
\end{assumption}

Assumption~\ref{as:nogperm} will be used to show uniqueness of the optimal solution to BP, which we state below as a lemma.
\begin{lemma}
The solution $\alphahat$ is unique with probability 1 over Gaussian noise $\Wtrain \sim  \mathcal{N}(0, \sigma^2\mathbf{I}_n)$ when $\Atrain$ satisfies Assumptions~\ref{as:nonzero},~\ref{as:invertibility} and~\ref{as:nogperm}.
\end{lemma}

\begin{proof}
To show uniqueness of the optimal solution to the BP program~\eqref{prog:bp}, it suffices to show that there cannot exist more than one optimal BFS for the equivalent LP~\eqref{prog:lp}.
We will show that the probability that there exists two basic feasible solutions for the LP that attain the same objective value is exactly zero. 
Let $\begin{bmatrix} \uone^\top & \vone^\top \end{bmatrix}$ and $\begin{bmatrix} \utwo^\top & \vtwo^\top \end{bmatrix}$ be two  basic feasible solutions to the LP with supports restricted to $S_1$ and $S_2$ where $S_1, S_2 \subset [d]$ and $S_1 \neq S_2$. 
Let $\alphaone = \uone -\vone $ and $\alphatwo = \utwo - \vtwo$ denote their equivalent solutions for the BP program. 
Since $|S_1| = |S_2| = n$, from Assumption~\ref{as:invertibility} we can write $\alphaone = \inv{\Atrainsone}\Wtrain, \alphatwo = \inv{\Atrainstwo}\Wtrain.$  
Denoting $\V := \inv{\Atrainstwo}\Wtrain$, we have $\V \sim \mathcal{N}(0,\inv{\Atrainstwo} (\inv{\Atrainstwo})\top)$.
Thus, we get
\begin{align*}
    \Pr\left[\norm{\alphaone}_1 = \norm{\alphatwo}_1\right] &= \Pr\left[\norm{\inv{\Atrainsone}\Wtrain}_1 = \norm{\inv{\Atrainstwo}\Wtrain}_1\right]\\
    &= \Pr\left[\norm{\inv{\Atrainsone} \Atrainstwo\V}_1 = \norm{\V}_1\right].
\end{align*}
From the well-known result on the isometry of the $\ell_1$ norm \cite{doi:10.1080/00029890.1994.11996972}, the last equality holds \textit{iff} either $\V = \veczero$ or if $\inv{\Atrainsone} \Atrainstwo \in \gperm$. 
From our assumptions,$\inv{\Atrainsone} \Atrainstwo \notin \gperm$ and the equality can hold \textit{only if} $\V = \veczero$. 
However, by Assumption~\ref{as:nonzero}, the random variable $\V$ has a non-zero covariance matrix, and so $\Pr\left[\norm{\alphaone}_1 = \norm{\alphatwo}_1\right]= 0$. 
Hence, we have $\Pr\left[\norm{\alphaone}_1 = \norm{\alphatwo}_1\right] = 0$ and the probability that two BFS's can both be optimal for a random realization of noise $\Wtrain$ is equal to $0$.
This completes the proof.
\end{proof}

Finally, we show that the unique solution is of support size \textit{exactly} $n$, which establishes $1$-parsimony for BP and allows us to apply Theorem~\ref{thm:parsimoniousnoisefit} for the $\ell_1$-minimizing interpolator.
\begin{lemma}
Under the above assumptions, the support of the \textbf{unique} solution $\alphahat$ is exactly $n$ with probability $1$ over Gaussian noise $\Wtrain \sim  \mathcal{N}(0, \sigma^2 \mathbf{I}_n)$.
\end{lemma}

\begin{proof}
It is sufficient to show that for every $k < n$, $\Pr\left[\abs{\supp(\alphahat)} = k\right] = 0$.
Consider any subset $S \subset [d]$, $\abs{S} = k$.
Denote the columns of $\Atrains$ as $\abold_1, \abold_2, \dots, \abold_k$. Since $k < n$, there exists a vector $\lpu$ \textit{of unit $\ell_2$ norm} such that $\lpu\tran \abold_i = 0$ for all $i \in [k]$.  Let $\event_S$ denote the event that $\Wtrain$ can be expressed as a linear combination of columns of $\Atrains$. If $\event_S$ occurs then $\Wtrain$ must be orthogonal to $\lpu$ i.e. $\Wtrain \tran \lpu = 0$. Thus,

\begin{align*}
\Pr\left[\event_S\right] \leq  \Pr\left[\Wtrain \tran \lpu = 0\right] = 0 ,
\end{align*}
where the last equality follows by noting that $\vecnorm{\lpu}{2} = 1$ and thus $\Wtrain^\top \lpu \sim \NORMAL(0, \sigma^2)$.
This argument clearly holds for any subset $S \subset [d], \abs{S} = k$, and for any $k < n$.
Therefore, the union bound gives us 
\begin{align*}
\Pr\left[\abs{\supp(\alphahat)} < n \right] &\leq \sum_{k=1}^{n-1} \Pr\left[\abs{\supp(\alphahat)} = k \right] \\
&\leq \sum_{k=1}^{n-1} \sum_{S \subset [d]: |S| = k} \Pr\left[\event_S\right] \\
&= \leq \sum_{k=1}^{n-1} \sum_{S \subset [d]: |S| = k} (0) \\
&= 0.
\end{align*}
and since we already know that $|\supp(\alphahat)| \leq n$, we have $\Pr\left[\abs{\supp(\alphahat)} = n\right]  = 1$, thus completing this proof.
\end{proof}

\subsection{Proof of Proposition~\ref{prop:sparseinterpolators}}\label{app:sparsity}

In this section, we provide the proof of Proposition~\ref{prop:sparseinterpolators}; in particular, the statement of Equation~\eqref{eq:hybridinterpolator}.
Recall that we have assumed $\EE\left[\avec(\Xvec)\avec(\Xvec)^\top\right] = \Sigmabold = \mathbf{I}_d$.
Using the triangle inequality, we have
\begin{align*}
    \EE[\testerr(\alphahat_{\mathsf{hybrid}})] &= \vecnorm{\alphahat_{\mathsf{hybrid}} - \alphastar}{2}^2 + \sigma^2 \\
    &\leq \vecnorm{\alphahat_1 - \alphastar}{2}^2 + \vecnorm{\alphahat_{\mathsf{hybrid}} - \alphahat_1}{2}^2  + \sigma^2 .
\end{align*}

From the definition of the estimator $\alphahat_{\mathsf{hybrid}}$ and a similar argument as in the proof of Theorem~\ref{thm:idealinterpolator}, we have $\alphahat_{\mathsf{hybrid}} - \alphahat_1 = \Deltahat = \Atrain^\dagger \Wtrain'$ and thus,
\begin{align*}
    \vecnorm{\alphahat_{\mathsf{hybrid}} - \alphahat_1}{2}^2 \leq \frac{\vecnorm{\Wtrain'}{2}^2}{\lambda_{min}(\Atrain \Atrain^\top)} .
\end{align*}

Recalling that we denoted $\mathcal{E}_{\mathsf{est}}(\alphahat_1) := \vecnorm{\alphahat_1 - \alphastar}{2}^2$, we have
\begin{align*}
\EE[\testerr(\alphahat_{\mathsf{hybrid}})] &\leq \mathcal{E}_{\mathsf{est}}(\alphahat_1) + \frac{\vecnorm{\Wtrain'}{2}^2}{\lambda_{min}(\Atrain \Atrain^\top)} .
\end{align*}

Now, remember that 
\begin{align*}
\Wtrain' &= \Ytrain - \Atrain \alphahat_1 = \Ytrain - \Atrain \alphastar + \Atrain(\alphastar - \alphahat_1) \\
&= \Wtrain + \Atrain (\alphastar - \alphahat_1) ,
\end{align*}

and thus we have
\begin{align*}
\vecnorm{\Wtrain'}{2}^2 &\leq 2 \vecnorm{\Wtrain}{2}^2 + 2\vecnorm{\Atrain(\alphastar - \alphahat_1)}{2}^2 \\
&= 2 \vecnorm{\Wtrain}{2}^2 + 2n\mathcal{E}_{\mathsf{pred}}(\alphahat_1)
\end{align*}

where we recall the definition of the prediction error $\mathcal{E}_{\mathsf{est}}(\alphahat_1)$.
Substituting this in the above expression for the tesr error completes the proof of Equation~\eqref{eq:hybridinterpolator} and thus the proof of Proposition~\ref{prop:sparseinterpolators}.


\qed

Recall that $\alphahat_1$ is a suitable estimator that uses the sparsity level $k$ or the noise variance $\sigma^2$ to estimate $\alphastar$ directly.
So $\alphahat_1$ can be provided by a suitable estimator used for sparse recovery (like Lasso or OMP, but also many other choices).
The results were stated as corollaries in Section~\ref{sec:hybrid}.
We provide the proofs of those corollaries below.

\subsubsection{Proof of Corollary~\ref{cor:minimaxoptimal} (order-optimality)}

Substituting the conditions for order-optimality (Equations~\eqref{eq:minimaxest} and~\ref{eq:minimaxpred}) into Equation~\eqref{eq:hybridinterpolator}, we get
\begin{align*}
\testerr(\alphahat_{\mathsf{hybrid}}) \leq \frac{Ck \ln \left(\frac{d}{k}\right)}{n} + \frac{\vecnorm{\Wtrain}{2}^2 + C'k \ln\left(\frac{d}{k}\right)}{\lambda_{min}(\Atrain \Atrain^\top)} .
\end{align*}

In the proof of Corollary~\ref{cor:minimaxoptimal}, we proved that $\vecnorm{\Wtrain}{2}^2 \geq n\sigma^2(1 - \delta)$ and $\lambda_{min}(\Atrain \Atrain^\top) \geq (\sqrt{d} - 2\sqrt{n})^2$ with probability at least $(1 - e^{-n \delta^2/8} - e^{-n/2})$.
Substituting these inequalities above, we get
\begin{align*}
\testerr(\alphahat_{\mathsf{hybrid}}) &\leq \frac{Ck \ln \left(\frac{d}{k}\right)}{n} + \frac{n \sigma^2 ( 1- \delta) + C'k \ln\left(\frac{d}{k}\right)}{(\sqrt{d} - 2\sqrt{n})^2} \\
&\leq \frac{Ck \ln \left(\frac{d}{k}\right)}{n} + \frac{n \sigma^2 ( 1- \delta)}{(\sqrt{d} - 2\sqrt{n})^2} + \frac{C''k \ln\left(\frac{d}{k}\right)}{(\sqrt{n})^2} \\
&= \frac{(C + C'')k \ln \left(\frac{d}{k}\right)}{n} + \frac{n \sigma^2 ( 1- \delta)}{(\sqrt{d} - 2\sqrt{n})^2} ,
\end{align*}

where in the inequality we used $d \geq cn$ for some $c > 4$, which implies that $\sqrt{d} \geq \sqrt{c} \sqrt{n} \implies \sqrt{d} - 2\sqrt{n} \geq (\sqrt{c} - 2) \sqrt{n} = c'\sqrt{n}$ for some $c' > 0$.
This matches the statement of Equation~\eqref{eq:minimaxoptimal} and completes the proof.

\subsubsection{Recovery guarantees for Lasso (Corollary~\ref{cor:lassoomp}) and square-root-Lasso (Corollary~\ref{cor:sqrootlasso})}

In this section, we provide corollaries for hybrid interpolators that use the Lagrangian Lasso and the square-root Lasso.
We follow notation from Wainwright's book on high-dimensional non-asymptotic statistics~\cite[Chapter $7$]{wainwright2019high} and provide original citations where-ever applicable.
\begin{definition}[Lagrangian Lasso in presence of noise]
We define the Lagrangian lasso with regularization parameter $\lambda_n$ (which can in general depend on $\sigma^2$ and $n$) as the estimator that solves the following optimization problem:
\begin{align*}
    \alphahat_{1,\mathsf{Las.}} := {\arg \min} \frac{1}{2n} \vecnorm{\Ytrain - \Atrain \alphabold}{2}^2 + \lambda_n \vecnorm{\alphabold}{1} .
\end{align*}
\end{definition}

\begin{definition}[Square-root-Lasso]
We define the square-root-Lasso with regularization parameter $\gamma_n$ (which can in general depend on $n$) as the estimator that solves the following optimization problem:
\begin{align*}
    \alphahat_{1,\mathsf{sq.root.Las.}} := {\arg \min} \frac{1}{\sqrt{n}} \vecnorm{\Ytrain - \Atrain \alphabold}{2} + \gamma_n \vecnorm{\alphabold}{1} .
\end{align*}
\end{definition}

For recovery guarantees on both the Lagrangian Lasso and the square-root-Lasso, in addition to $k$-sparsity, we require the following \textit{restricted eigenvalue condition} on the design matrix:
\begin{definition}
The matrix $\Atrain$ satisfies the restricted eigenvalue condition over set $\mathsf{supp}(\alphastar)$ with parameters $(\kappa, \beta)$ if 
\begin{align*}
\frac{1}{n} \vecnorm{\Atrain \Delta}{2}^2 \geq \kappa \vecnorm{\Delta}{2}^2 \text{ for all } \Delta \in \cone_{\beta}(\mathsf{supp}(\alphastar) 
\end{align*}

where for any subset $S$, we define
\begin{align*}
    \cone_{\beta}(S) := \{\Delta \in \reals^d: \vecnorm{\Delta_{S^c}}{1} \leq \beta \vecnorm{\Delta_S}{1}\} .
\end{align*}
\end{definition}

The following lemma by Raskutti, Wainwright and Yu~\cite{raskutti2010restricted} shows that the matrix $\Atrain$ with iid standard normal entries (thus satisfying Assumption~\ref{as:ideal3} atisfies this property with high probability as restated in the following lemma\footnote{In fact the original result in the paper shows this property for generally correlated Gaussian design.}.
\begin{lemma}[Corollary 1,~\cite{raskutti2010restricted}]
The iid Gaussian matrix $\Atrain$ satisfies the restricted eigenvalue condition $(\kappa, \beta=3)$ for any $\kappa >0$ with probability greater than or equal to $(1 - 2e^{-cn})$, provided that $n \geq C\frac{k \ln d}{\kappa^2}$.
\end{lemma}

Under this condition, Bickel, Ritov and Tsybakov~\cite{bickel2009simultaneous} proved the following bounds on estimation error as well as prediction error of the Lagrangian Lasso.
\begin{theorem}[Theorems $7.13$ and $7.20$,~\cite{wainwright2019high}]\label{thm:lassorecovery}
Let $\Atrain$ satisfy the restricted eigenvalue condition with parameters $(\kappa, \beta=3)$.
Then, any solution of the Lagrangian Lasso with regularization parameter $\lambda_n \geq 2 \vecnorm{\frac{\Atrain^\top\Wtrain}{n}}{\infty}$ satisfies
\begin{subequations}
\begin{align*}
    \mathcal{E}_{\mathsf{est}}(\alphahat_{1,\mathsf{Las.}}) &:= \vecnorm{\alphahat_{1,\mathsf{Las.}} - \alphastar}{2}^2 \leq \frac{9 k \lambda_n^2}{\kappa} \\
    \mathcal{E}_{\mathsf{pred}}(\alphahat_{1,\mathsf{Las.}}) &:= \vecnorm{\Atrain(\alphahat_{1,\mathsf{Las.}} - \alphastar)}{2}^2 \leq \frac{9 k \lambda_n^2}{\kappa} .
\end{align*}
\end{subequations}

As a corollary, for any $\delta > 0$ and regularization parameter choice $\lambda_n = 2C' \sigma \left(\sqrt{\frac{2 \ln d}{n}} + \delta\right)$ we have 
\begin{subequations}
\begin{align}
    \mathcal{E}_{\mathsf{est}}(\alphahat_{1,\mathsf{Las.}}) &\leq \frac{36 C' k \sigma^2}{\kappa^2}\left(\frac{2 \ln d}{n} + \delta \right) \label{eq:lassoest}\\
    \mathcal{E}_{\mathsf{pred}}(\alphahat_{1,\mathsf{Las.}}) &\leq \frac{36 C' k \sigma^2}{\kappa^2}\left(\frac{2 \ln d}{n} + \delta \right) \label{eq:lassopred}
\end{align}
\end{subequations}

with probability greater than or equal to $(1 - 2e^{-n \delta/2})$.
\end{theorem}

Observe from Equation~\eqref{eq:lassoest} and~\eqref{eq:lassopred} that $\mathcal{E}_{\mathsf{est}} \sim \mathcal{E}_{\mathsf{pred}} \leq C'' \frac{\sigma^2 k \ln d}{n}$ with high probability.
Thus, a similar argument as in the proof of Corollary~\ref{cor:minimaxoptimal} will also give the statement of Corollary~\ref{cor:lassoomp} for the case of the Lasso.
We omit the argument for brevity.
\qed

Similarly, the following guarantee was proved for the square-root-Lasso by Belloni, Chernozhukov and Wang~\cite{belloni2011square}, again assuming the restricted eigenvalue condition.
\begin{theorem}[~\cite{belloni2011square}]\label{thm:sqrootlassorecovery}
Let $\Atrain$ satisfy the restricted eigenvalue condition with parameters $(\kappa, \beta=3)$.
Then, any solution of the square-root-Lasso with regularization parameter $\gamma_n \geq 2 \frac{\vecnorm{\Atrain^\top \Wtrain}{\infty}}{\sqrt{n} \vecnorm{\Wtrain}{2}}$ satisfies
\begin{subequations}
\begin{align*}
    \mathcal{E}_{\mathsf{est}}(\alphahat_{1,\mathsf{Las.}}) &:= \vecnorm{\alphahat_{1,\mathsf{Las.}} - \alphastar}{2}^2 \leq \frac{C k \vecnorm{\Wtrain}{2}^2 \gamma_n^2}{n \kappa} \\
    \mathcal{E}_{\mathsf{pred}}(\alphahat_{1,\mathsf{Las.}}) &:= \vecnorm{\Atrain(\alphahat_{1,\mathsf{Las.}} - \alphastar)}{2}^2 \leq \frac{C' k \vecnorm{\Wtrain}{2}^2 \gamma_n^2}{n \kappa} .
\end{align*}
\end{subequations}

for constants $C, C' > 0$.
As a corollary, for any $\delta > 0$ and regularization parameter choice $\gamma_n = 2C' \left(\sqrt{\frac{2 \ln d}{n}} + \delta\right)$ we have 
\begin{subequations}
\begin{align}
    \mathcal{E}_{\mathsf{est}}(\alphahat_{1,\mathsf{Las.}}) &\leq \frac{C'' k \sigma^2}{\kappa^2}\left(\frac{2 \ln d}{n} + \delta \right) \label{eq:sqrootlassoest}\\
    \mathcal{E}_{\mathsf{pred}}(\alphahat_{1,\mathsf{Las.}}) &\leq \frac{C''' k \sigma^2}{\kappa^2}\left(\frac{2 \ln d}{n} + \delta \right) \label{eq:sqrootlassopred}
\end{align}
\end{subequations}

with probability greater than or equal to $(1 - 2e^{-n \delta/2})$ for new constants $C'', C''' > 0$.
\end{theorem}

Observe that the above choice of regularization parameter $\gamma_n$ \textit{does not} depend on the noise variance $\sigma$ -- thus, the square root Lasso can be implemented successfully, giving the rates in Equations~\eqref{eq:sqrootlassoest} and~\ref{eq:sqrootlassopred}, without requiring this knowledge.
As in the case of the Lagrangian Lasso, we have $\mathcal{E}_{\mathsf{est}} \sim \mathcal{E}_{\mathsf{pred}} \leq C'' \frac{\sigma^2 k \ln d}{n}$ with high probability.
Thus, a similar argument as in the proof of Corollary~\ref{cor:minimaxoptimal} will also give the statement of Corollary~\ref{cor:sqrootlasso} for the case of the square-root-Lasso.
This completes the proof of Corollary~\ref{cor:sqrootlasso}.
\qed

\subsubsection{Recovery guarantee for OMP (Corollary~\ref{cor:lassoomp})}

In this section we provide a formal statement for the recovery guarantee for orthogonal matching pursuit (OMP) in Corollary~\ref{cor:lassoomp}.
The first theoretical guarantees on OMP were proved for noiseless recovery, when OMP is run to completion in the absence of noise~\cite{tropp2007signal}.
Subsequent work~\cite{fletcher2009necessary} improved the sampling requirement, also in the noiseless setting.
Recovery guarantees in the presence of noise were then proved for an appropriate stopping condition~\cite{cai2011orthogonal}.
We formalize the version of OMP that we use below.
\begin{definition}[OMP for recovery in presence of Gaussian  noise~\cite{cai2011orthogonal}]
Denote that $j^{th}$ column of $\Atrain$ by $\avec_j$.
The OMP algorithm is defined iteratively according to the following steps:
\begin{enumerate}
    \item (Iteration $t = 1$.) Initialize residual $r_0 = \Ytrain$ and initialize the set of selected variables $S = \phi$.
    \item Select variable $s_t := {\arg \max} |\inprod{\avec_j}{\mathbf{r}_{t-1}}|$ and add it to $S$.
    \item Denote $\Atrain(S)$ as the submatrix of $\Atrain$ whose columns are in $S$.
    Let $\mathbf{P}_t = \Atrain(S) \Atrain(S)^\dagger$ denote the projection of $\Ytrain$ onto the linear space spanned by the elements of $\Atrain(S)$.
    Update $\mathbf{r}_t = (\mathbf{I}_d - \mathbf{P}_t) \Ytrain$.
    \item We stop the algorithm under one of two conditions: a) If $\vecnorm{\Atrain \mathbf{r}_t}{\infty} \leq \sigma \sqrt{2(1 + \eta) \ln d}$ for algorithmic parameter $\eta > 0$; b) if the number of steps is equal to $k_0$, where $k_0 > k$ is guaranteed.
    Otherwise, set $t = t+1$ and go back to Step 2.
    We refer to the stopping conditions henceforth as Condition a) and Condition b) respectively.
    \item (Once algorithm has terminated) return $\alphahat_{1,\mathsf{OMP}}(\Atrain, \Ytrain) = \alphahat_{\mathsf{OLS}}(\Atrain(S), \Ytrain)$.
\end{enumerate}
\end{definition}

We also define \textit{pairwise incoherence} 
\begin{align*}
    \mu(\Atrain) := \max_{j \neq j'} \frac{\inprod{\avec_j}{\avec_{j'}}}{\vecnorm{\avec_j}{2} \vecnorm{\avec_{j'}}{2}} .
\end{align*}

In general for recovery guarantees, we would like this quantity to be small.
We state the following well-known lemma for the entries of $\Atrain$ being iid $\NORMAL(0,1)$:
\begin{lemma}
For iid Gaussian $n \times d$ matrix $\Atrain$, we have $\mu(\Atrain) < \frac{1}{2k-1}$ with high probability as long as $n = \Omega(k^2 \ln d)$.
\end{lemma}

We now restate the result that is most relevant for our purposes, and holds for any design matrix $\Atrain$.
\begin{theorem}[Theorem $8$ from~\cite{cai2011orthogonal}.]\label{thm:OMPrecovery}
Suppose that $\mu < \frac{1}{2k - 1}$ and all the non-zero coefficients $\alpha^*_i, i \in \mathsf{supp}(\alphastar)$ satisfy
\begin{align}\label{eq:SNRcondition}
    |\alpha^*_i| \geq \frac{2 \sigma \sqrt{2(1 + \eta) \ln d}}{1 - (2k - 1)\mu(\Atrain)}
\end{align}

for some $\eta > 0$.
Then, the above version of OMP has the following guarantees:
\begin{enumerate}
\item Under stopping condition a), OMP will recover $S = \mathsf{supp}(\alphastar)$ with probability greater than or equal to $(1 - \frac{k}{d^{\eta} \sqrt{2 \ln d}})$.
This directly implies that
\begin{subequations}
\begin{align}
    \mathcal{E}_{\mathsf{test}}(\alphahat_{1,\mathsf{OMP}}) &= \vecnorm{\alphahat_1 - \alphastar}{2}^2 \leq \frac{k \sigma^2}{n} \label{eq:OMPerrorwithvariableselection_a} \\
    \mathcal{E}_{\mathsf{pred}}(\alphahat_{1,\mathsf{OMP}}) &= \vecnorm{\Atrain(\alphahat_1 - \alphastar)}{2}^2 \leq \frac{k \sigma^2}{n}
\end{align}
\end{subequations}
with probability greater than or equal to $\left(1 - \frac{k}{d^{\eta} \sqrt{2 \ln d}}\right)$.
\item Under stopping condition b), OMP will recover $S \supset \mathsf{supp}(\alphastar)$ with probability greater than or equal to $(1 - \frac{k_0}{d^{\eta} \sqrt{2 \ln d}})$.
This directly implies that
\begin{subequations}
\begin{align}
    \mathcal{E}_{\mathsf{test}}(\alphahat_{1,\mathsf{OMP}}) &= \vecnorm{\alphahat_1 - \alphastar}{2}^2 \leq \frac{k_0 \sigma^2}{n} \label{eq:OMPerrorwithvariableselection_b}\\
    \mathcal{E}_{\mathsf{pred}}(\alphahat_{1,\mathsf{OMP}}) &= \vecnorm{\Atrain(\alphahat_1 - \alphastar)}{2}^2 \leq \frac{k_0 \sigma^2}{n}
\end{align}
\end{subequations}
with probability greater than or equal to $\left(1 - \frac{k_0}{d^{\eta} \sqrt{2 \ln d}}\right)$.
\end{enumerate}

\end{theorem}

We observe that the obtained guarantees for OMP in the noisy setting are \textit{support recovery guarantees}, which are much stronger than directly bounding the $\ell_2$-norm of the error.
While the expressions do not have an implicit dependence on the dimension $d$, it requires the absolute values of the coefficients $|\alpha^*_i|$ to be bounded above by $\Omega(\sigma \sqrt{\ln d})$.
This can effectively be thought of as requiring the SNR to be high enough to ensure signal recovery - the more we overparameterized the number of features $d$, the more stringent this condition becomes.
Also observe that the versions of OMP with stopping conditions a) and b) use different kinds of side information: Stopping condition a) uses knowledge of the noise variance, while stopping condition b) only uses a guaranteed \textit{upper bound} on the true sparsity level $k$.
Stopping condition a) gives the bound in Equation~\eqref{eq:OMPerrorwithvariableselection_a} which is adaptive to unknown sparsity level.
Stopping condition b) has a weaker kind of side information, but gives the bound in Equation~\eqref{eq:OMPerrorwithvariableselection_b} which, while having a good convergence rate, does not adapt to the unknown sparsity level.

As before, a similar argument as in the proof of Corollary~\ref{cor:minimaxoptimal} will also give the statement of Corollary~\ref{cor:lassoomp} for the case of OMP with both stopping conditions.
We omit the argument for brevity.

\section{The interpolation threshold: Regularity vs randomness of training data}\label{sec:threshold}
 
We observe that Corollary~\ref{cor:crabpot} is meaningful primarily for the heavily overparameterized regime where $d >> n$ (more formally, if we vary $d$ as a function of $n$, we have $\lim_{n \to \infty} \frac{n}{d(n)} = 0$).
It does not mathematically explain the magnitude of the interpolation peak at $d \sim n$ that we observe in Figure~\ref{fig:poly_MSE}, which reflects the \textit{mean} of the minimum possible test MSE that results purely from fitting noise at the interpolation threshold.
It is well-known that the behavior of the minimum singular value of a random matrix with iid Gaussian entries is very different when the matrix is approximately square -- a ``hard-edge" phenomenon arises~\cite{edelman1988eigenvalues,szarek1991condition} and the minimum singular value actually exhibits heavy-tailed behavior~\cite{rudelson2008littlewood}.
It is unclear whether the same phenomena hold for lifted features on data such as Fourier, Legendre or even Vandermode features.
It thus makes sense to consider the statistics of the test MSE at the interpolation threshold more carefully.

\multifigureexterior{fig:MSE_meanandmedian}{Test MSE of the ideal interpolator (and other practical interpolation schemes) of Legendre and Fourier features for regularly spaced and randomly drawn training points, true fit is a $2$-degree polynomial fit.
}{
\subfig{0.5\textwidth}{Legendre features, regularly spaced training points.}{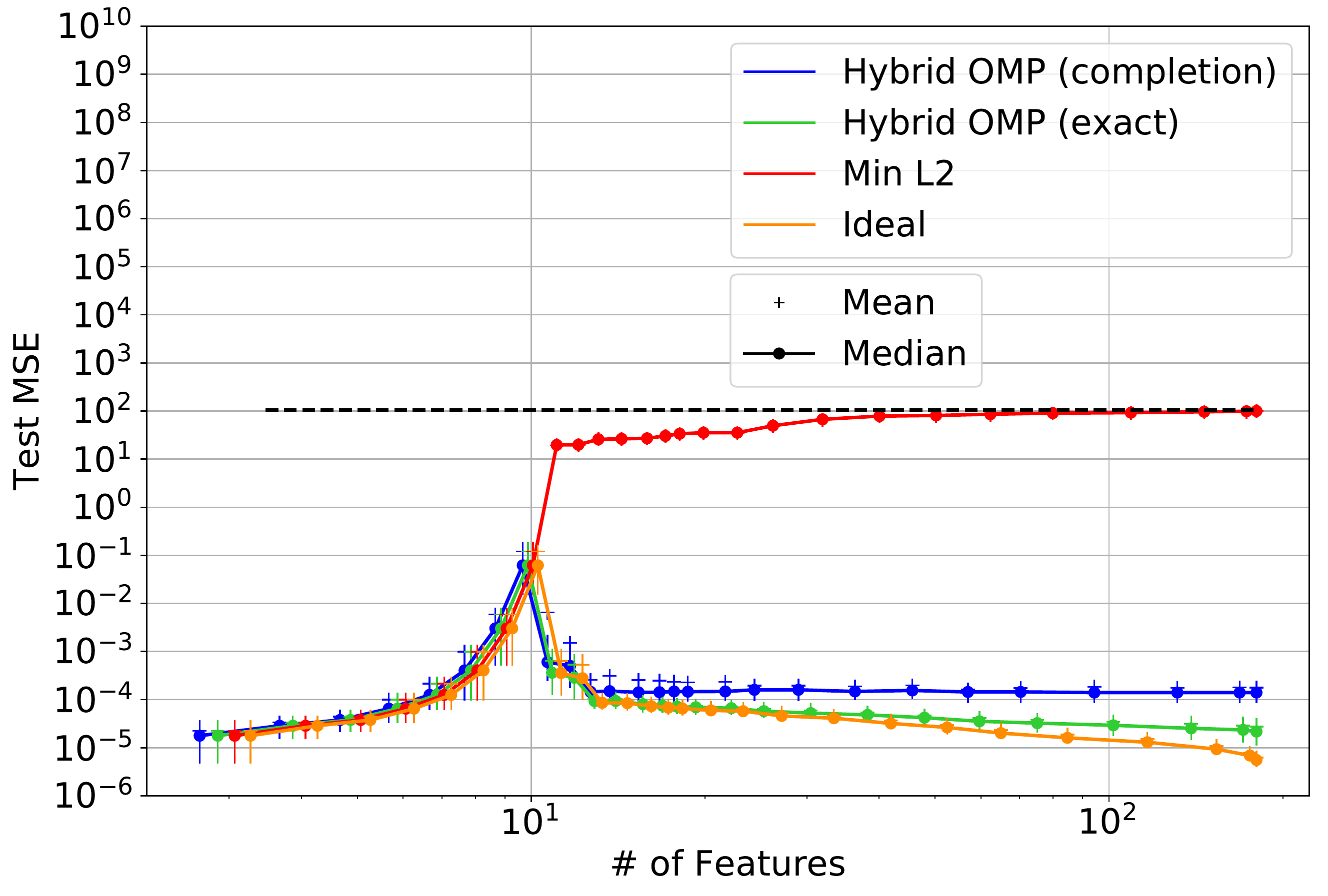}
\subfig{0.5\textwidth}{Legendre features, randomly drawn training points.}{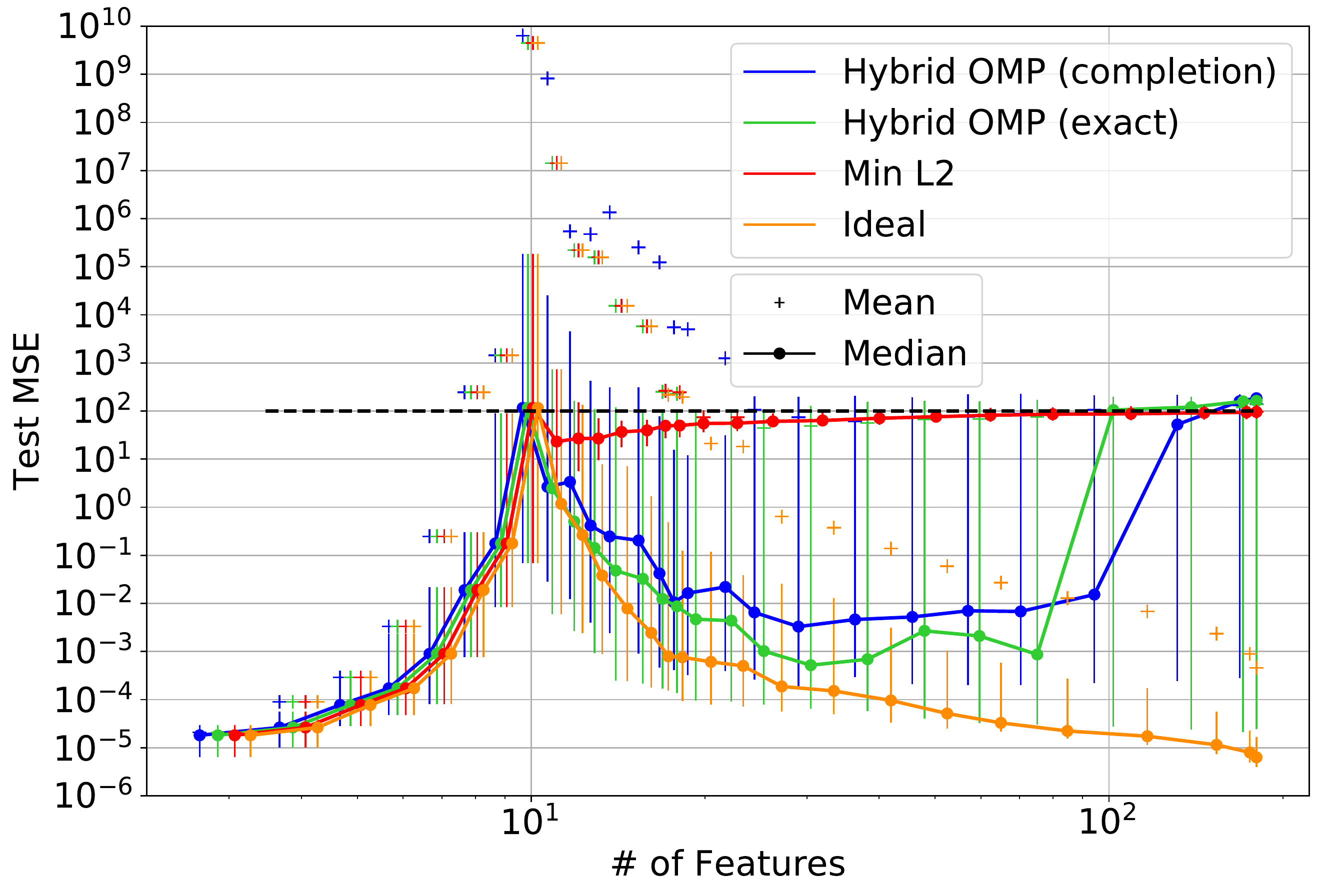}
\subfig{0.5\textwidth}{Fourier features, regularly spaced training points.}{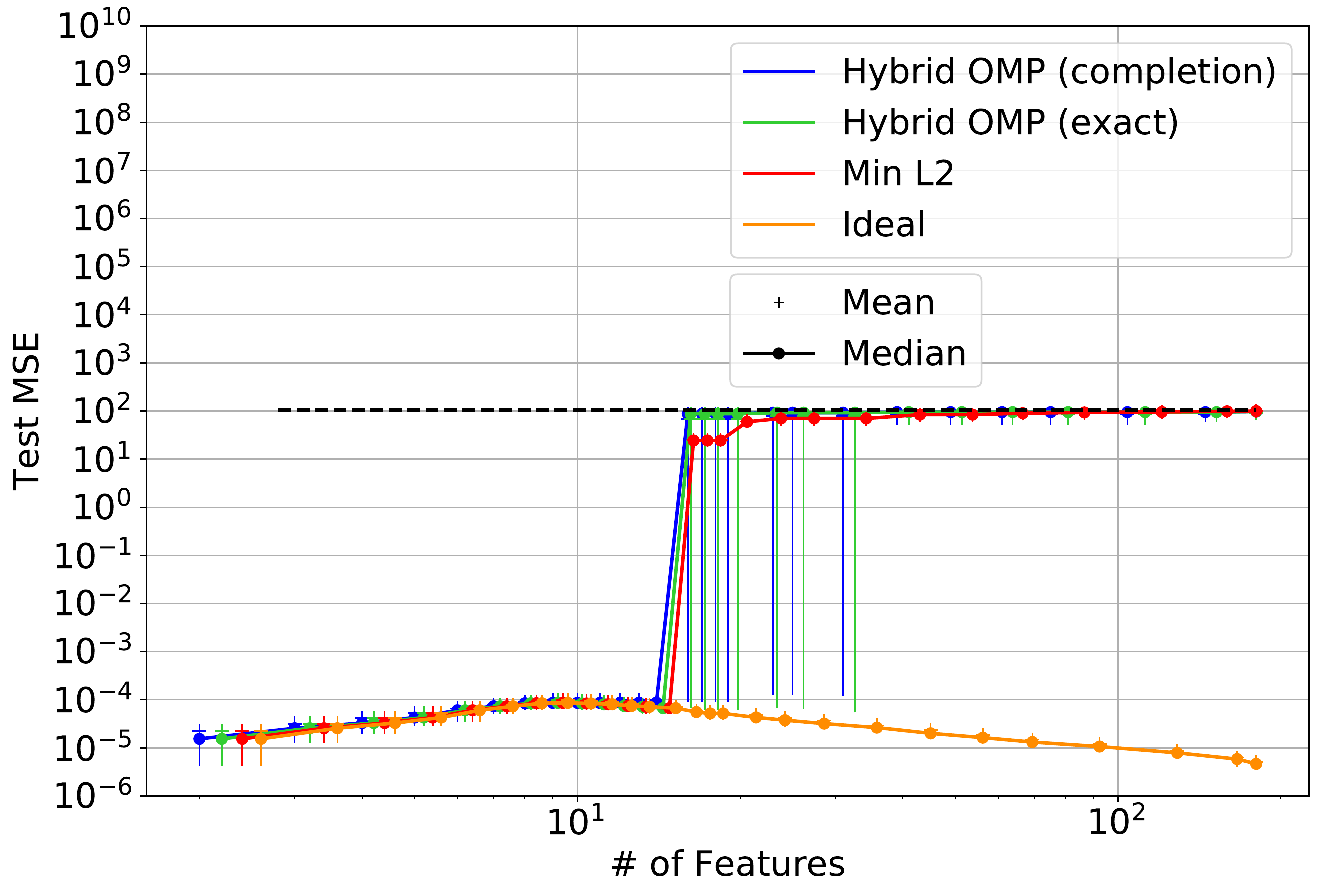}
\subfig{0.5\textwidth}{Fourier features, randomly drawn training points.}{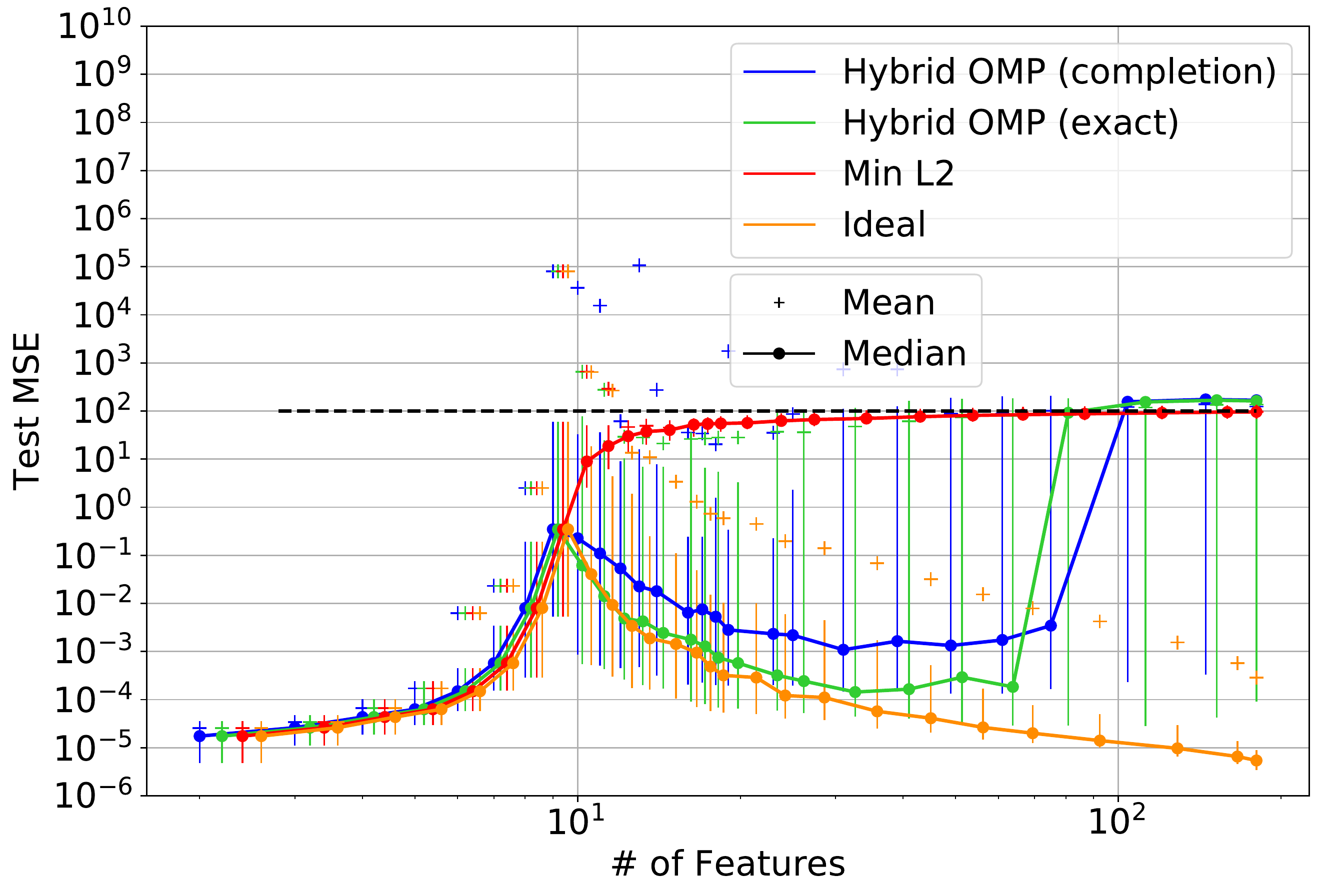}
} 

Figure~\ref{fig:MSE_meanandmedian} shows the generalization performance of various interpolating solutions on approximately ``whitened" Legendre and Fourier features on data.
In particular, we consider two generative assumptions for the training data $\{X_i\}_{i=1}^n$:
\begin{enumerate}
\item \textit{Regularly spaced data}: $X_i = -1 + \frac{2(i-1)}{(n-1)} \text{ for all } i \in [n]$.
\item \textit{Randomly drawn data}: $X_i \sim \text{Unif}[-1,1]$.
\end{enumerate}

and the generative model for the data to be $Y_i = f^*(X_i) + W_i$ where $f^*: [-1,1] \to \reals$ is a polynomial of degree $2$.
We make $50$ random draws of both the training data (in the case of randomly drawn training points) and noise, and plot both the \textit{mean} and the \textit{median} test MSE evaluated over $10^4$ test points, together with $85\%$ confidence intervals.
The results depicted in Figure~\ref{fig:MSE_meanandmedian} have several implications.
In particular, the distributional assumptions on the data make a significant difference: regularly spaced data approximately preserves the orthogonality of the design matrix $\Atrain \approx \Btrain $\footnote{recall that the Legendre and Fourier features are orthogonal in function space, which corresponds to an \textit{integral} over the interval $[-1,1]$; regular spacing of training points appropriately approximates this integral.}.
On the other hand, randomly drawn training points constitute a more complex \textit{random matrix} $\Atrain$, for which the minimum singular value is much lower and moreover exhibits heavy-tailed behavior\footnote{Intuitively, this is because of connections of the minimum singular value of the random matrix to the non-uniformity of random spacings~\cite{rudelson2008littlewood}; one can prove that that the smallest spacing between randomly drawn points is on the order of $1/n^2$ (as opposed to $1/n$) with high probability.
}.
This can be observed from the large size of the confidence intervals, as well as the much lower \textit{median} test MSE (which is still higher than the test MSE incurred from regularly spaced training data).
Randomly drawn training data is what we use in the conventional machine learning step, and this analysis tells us that we truly need to overparameterize to ensure harmless fitting of noise, unlike in logistic regression.

\section{Mathematical facts}\label{app:facts}

In this section, we collect miscellaneous mathematical facts that were useful for proving some of our results.
\begin{fact}\label{fact:binomialcoeffbound}
For positive numbers $(n,k)$, we have
\begin{align*}
{n \choose k} \leq \left(\frac{en}{k}\right)^k .
\end{align*}
\end{fact}

\begin{proof}
Observe that 
\begin{align*}
{n \choose k} = \frac{n!}{k!(n-k)!} = \frac{n(n-1)\ldots(n-k+1)}{k!} \leq \frac{n^k}{k!} .
\end{align*}

Next, we have $e^k = \sum_{i = 0}^{\infty} \frac{k^i}{i!} \geq \frac{k^k}{k!}$.
Rearranging this gives us $k! \geq \frac{k^k}{e^k}$, and substituting it above gives us 
\begin{align*}
{n \choose k} \leq \frac{n^k \cdot e^k}{k^k} = \left(\frac{en}{k}\right)^k ,
\end{align*}
completing the proof of this fact.
\end{proof}

\section{Calculations for the regularly spaced Fourier case}\label{app:fourier}

\newcommand{\truek}{\ensuremath{k^*}}

\newcommand{\vecbestalias}{\mathbf{B}}
\newcommand{\survival}{\mathsf{SU}(
\truek)}
\newcommand{\contamination}{C(\truek)}
\newcommand{\vecalpha}{\ensuremath{\boldsymbol{\alpha}}}
\newcommand{\vecx}{\ensuremath{\mathbf{x}}}
\newcommand{\vecX}{\ensuremath{\mathbf{X}}}
\newcommand{\vecy}{\ensuremath{\mathbf{y}}}
\newcommand{\xtrain}{\ensuremath{\vecx_\mathsf{train}}}

\newcommand{\weightmatrix}{\boldsymbol{\Gamma}}
\newcommand{\vecestimatedalpha}{\boldsymbol{\hat{\alpha}}}
\newcommand{\xtrainj}{\ensuremath{x_\mathsf{train, j}}}

\newcommand{\vecytest}{\ensuremath{\mathbf{Y}}}
\newcommand{\vecestimatedytest}{\ensuremath{\mathbf{\hat{Y}}}}
\newcommand{\E}{\mathbb{E}}

\newcommand{\vecxtest}{\ensuremath{\vecX}}
\newcommand{\xtest}{\ensuremath{X}}
\newcommand{\vecweights}{\ensuremath{\mathbf{w}}}
\newcommand{\vectruncweights}{\ensuremath{\mathbf{\tilde{w}}}}
\newcommand{\truealpha}{\ensuremath{\alpha^{*}}}
\newcommand{\vectruealpha}{\ensuremath{\boldsymbol{\alpha^{*}}}}

\newcommand{\weight}{\ensuremath{w}}
\newcommand{\truncweight}{\ensuremath{\tilde{w}}}
\newcommand{\trueweightedalpha}{\ensuremath{\beta^{*}}}
\newcommand{\vecdummyweightedalpha}{\ensuremath{\boldsymbol{\beta}}}
\newcommand{\truncdummyweightedalpha}{\ensuremath{\xi}}
\newcommand{\vectruncdummyweightedalpha}{\ensuremath{\boldsymbol{\xi}}}
\newcommand{\vectruncestimatedweightedalpha}{\ensuremath{\hat{\boldsymbol{\xi}}}}
\newcommand{\dummyweightedalpha}{\ensuremath{\beta}}

\newcommand{\vecestimatedweightedalpha}{\ensuremath{\hat{\boldsymbol{\beta}}}}
\newcommand{\estimatedweightedalpha}{\ensuremath{\hat{{\beta}}}}
\newcommand{\estimatedalpha}{\ensuremath{\hat{{\alpha}}}}
\newcommand{\f}[1]{\ensuremath{f_{#1}(\vecx_{\mathsf{train}})}}
\newcommand{\F}[1]{\ensuremath{f_{#1}(\vecxtest)}}
\newcommand{\real}[1]{\mathbb{R}^{#1}}
\newcommand{\ytrue}{\ensuremath{\vecy_{\mathsf{train}}}}
\newcommand{\ytruej}{\ensuremath{y_{\mathsf{train,j}}}}
\newcommand{\numaliases}{\ensuremath{M}}
\newcommand{\aliasset}{\ensuremath{S(\truek)}}
\newcommand{\aliassetk}{\ensuremath{S(\truek)}}

\newcommand{\erf}[1]{\ensuremath{\mathsf{erf}\left(#1\right)}}

\newcommand{\comp}[1]{\mathbb{C}^{#1}}

\newcommand{\intset}{\ensuremath{\{\{\truek\} \cup \aliasset \}}}
 Consider observation $\ytrue \in \comp n$  and $d$  Fourier features, $f_k(\xtrainj) = e^{i2\pi k \xtrainj}$ for $k = \{0, 1, \dots, d-1\}.$ Here $\xtrain$ contains regularly spaced samples with $\xtrainj = \frac{j}{n}$ for $j \in \{0, 1, \dots, n-1\}.$  For ease of exposition we assume $d = (\numaliases+1) n$ for positive integer $\numaliases$.
 We wish to understand how the minimum weighted $\ell_2$ norm interpolating solution behaves in this setting.  The first $n$ features, form an orthogonal basis for $\comp n$ and thus it suffices to understand the case where $\ytrue$ is each basis vector separately. 
Let
\begin{align}
  \ytrue &= \f{\truek},   \label{eq:truey} \\
 \text{i.e} \: \: \ytruej &= e^{i\pi \truek \xtrainj}, j \in [n]. \nonumber
\end{align}
for some $\truek \in \{0, 1, \dots, n-1\}$.Without loss of generality we consider $\truek$ in the range $[0,n-1]$ since subsequent blocks of $n$ features will be aliases of these features.

Note that we can write,
\begin{align*}
    \ytrue = \sum_{j=0}^{d-1} \truealpha_j \f j, 
\end{align*}
where $\vectruealpha = \mathbf{e_{\truek}} \in \real n$ and $\mathbf{e}_j$ denotes the $j^{th}$ standard basis vector.
For any solution $\vecalpha \in \comp n$, the interpolating constraint is,
\begin{align*}
     \ytrue = \sum_{j=0}^{d-1} \alpha_j \f j.
\end{align*}
If we scale feature $\f j$ by real weight $\weight_j$, then the interpolating constraint becomes,
\begin{align*}
     \ytrue = \sum_{j=0}^{d-1} \dummyweightedalpha_j \weight_j \f j,
\end{align*}
with $\dummyweightedalpha_j = \frac{\alpha_j}{w_j}$.

We are interested in the minimum weighted $\ell_2$ norm solution subject to the interpolating constraint given by,

\begin{align}
    \vecestimatedalpha &= \arg \min_{\vecalpha \in \comp d} \norm{ \weightmatrix^{-1} \vecalpha}_2 \label{prob:weightedls}\\
     \subt \: \: &  \ytrue = \sum_{j=0}^{d-1} \alpha_j \f j. \nonumber
\end{align}
where $\weightmatrix = \text{diag}(\weight_0, \weight_1, \dots, \weight_{d-1})$.
Note that this is equivalent to the minimum $\ell_2$-minimizing coefficients corresponding to the weighted featurues, as defined in~\cite{bartlett2019benign}:
\begin{align}
    \vecestimatedweightedalpha &= \arg\min_{\vecdummyweightedalpha} \norm{\vecdummyweightedalpha}_2 \label{prob:weightedfeatures}\\
    \subt \: \: &  \ytrue = \sum\limits_{j = 0}^{d-1} \dummyweightedalpha_j \weight_j \f j. \nonumber
\end{align}
We will solve the problem in \eqref{prob:weightedfeatures} next.
First, we list some properties of the regularly spaced Fourier features. Denote by $\aliasset$, the set of indices corresponding to features that are exact aliases of $\f \truek$. Then,
\begin{align}
    \aliasset = \{  \truek+n, \truek + 2n , \dots, \truek + \numaliases n\}.\label{eq:aliasset}
\end{align}  Note that $ \numaliases = \abs{\aliasset} = \frac{d}{n} - 1$. We have,
\begin{align}
    \f j &= \f \truek , \: j \in S \label{eq:proprsf1}\\
    \inprod{\f j }{\f\truek} &= 0,  \:  j \notin \intset. \label{eq:proprsf2}
\end{align}

Using \eqref{eq:truey}, \eqref{eq:proprsf1} and \eqref{eq:proprsf2} we can rewrite the optimization problem in \eqref{prob:weightedfeatures} as,
\begin{align*}
    \vecestimatedweightedalpha &= \arg\min_{\vecdummyweightedalpha} \norm{\vecdummyweightedalpha}_2 \\
    \subt \: \: &  \dummyweightedalpha_{\truek} \weight_{\truek} + \sum\limits_{j \in \aliasset} \dummyweightedalpha_j \weight_j  = 1.
\end{align*}

Clearly to minimize the objective we must have $\estimatedweightedalpha_j = 0$ for $j \notin \intset$ and thus it suffices to consider the problem restricted to the indices in  $\intset$. By mapping these indices to the set $\{0, 1, \dots, \numaliases \}$  and denoting the weight vector restricted to this set as $\vectruncweights$  we  write an equivalent optimization problem,
\begin{align*}
    \vectruncestimatedweightedalpha &= \arg\min_{\vectruncdummyweightedalpha} \norm{\vectruncdummyweightedalpha}_2 \\
    \subt \: \: &   \sum\limits_{j=0}^{\numaliases} \truncdummyweightedalpha_j \truncweight_j  = 1.
\end{align*}

Note that $\truncweight_j = \weight_{\truek + jn}$ for $j = 0, 1, \dots, M$.\\
We find an optimal solution to this problem by  using the Cauchy Schwarz inequality which states,
\begin{align*}
    \norm{\vectruncestimatedweightedalpha}_2 \norm{\vectruncweights} \geq \abs{\inprod{\vectruncestimatedweightedalpha}{\vectruncweights}},
\end{align*}
where equality occurs if and only if $\vectruncestimatedweightedalpha = c \vectruncweights$ for some $c \in \comp{}$. Using the fact that $\weight_j \in \real{}$ and solving for $c$ using the interpolating constraint we get,
\begin{align*}
    \vectruncestimatedweightedalpha = \frac{\vectruncweights}{\norm{\vectruncweights}_2^2}.
\end{align*}
Mapping this back to original indices we get,
\begin{align*}
    \estimatedweightedalpha_{j} &=
    \begin{cases}
			 \frac{\weight_{j}}{V}, & j \in \{ \{\truek\} \cup \aliasset \}\\
            0, & \text{otherwise}.
		 \end{cases}
\end{align*}
and
\begin{align}
    \estimatedalpha_{j} &=
    \begin{cases}
			 \frac{\weight^2_{j}}{V}, & j \in \{ \{\truek\} \cup \aliasset \}\\
            0, & \text{otherwise},    
		 \end{cases} \label{eq:estalpha}
\end{align}
where, $$V = \sum\limits_{j \in \intset} \weight_j^2.$$
Next we consider the effect on a test point $\vecxtest \in \real{}$ with i.i.d. entries $\vecxtest \sim U[0, 1]$,  when the ground truth observation is $\vecytest = \F \truek$. On this point we predict,
\begin{align*}
    \vecestimatedytest = \sum_{j = 0}^{d-1} \estimatedalpha_j \F j.
\end{align*}
We want to understand how different $\vecestimatedytest$ is from $\vecytest$.
Using \eqref{eq:estalpha} we have,
\begin{align*}
    \vecestimatedytest &= \estimatedalpha_{\truek} \F \truek + \sum_{j \in S} \estimatedalpha_j \F j \\
    &= \estimatedalpha_{\truek} \vecytest + \sum_{j \in S} \estimatedalpha_j \F j.
\end{align*}
The prediction $\vecestimatedytest$ consists of two components. The first component is the true signal attenuated by a factor $\estimatedalpha_{\truek}$ due to the effect of signal bleed.  The signal bleeds to features that are orthogonal to the true signal and this leads to the  second component, a contamination term that we denote by,
\begin{align*}
    \vecbestalias = \sum_{j \in \aliasset} \estimatedalpha_j \F j.
\end{align*}
Let $\survival$ denote the fraction of the true coefficient that survives post signal bleed. Then,
\begin{align}
    \survival &= \abs{\frac{\estimatedalpha_{\truek}}{\truealpha_{\truek}}} 
    = \frac{\weight^2_{\truek}}{\sum\limits_{j \in \intset} \weight^2_j}. \label{eq:survival}
\end{align}
Let $\contamination$ denote the standard deviation of the contamination given by,
\begin{align*}
    \contamination &= \sqrt{\mathbb E[\abs{\vecbestalias)}^2]}.
\end{align*}
Using the property of Fourier features when $\vecX$ is spaced uniformly in $[0,1]$ namely,
\begin{align*}
    \mathbb{E}[\inprod{\F i}{\F j}] &=\begin{cases}
    0, & i \neq j\\
    1, & i = j.
    \end{cases}
\end{align*}
to get,
\begin{align*}
    \E(\abs{\vecbestalias}^2) = \sum\limits_{j \in \aliasset} \abs{\estimatedalpha_j}^2 \E[\abs{\F j}^2]
    = \sum\limits_{j \in \aliasset} \abs{\estimatedalpha_j}^2.
\end{align*}
Using this we have,
\begin{align}
    \contamination &= \sqrt{\sum\limits_{j \in \aliasset} \abs{\estimatedalpha_j}^2} 
    = \frac{\sqrt{\sum\limits_{j \in \aliasset} \weight_j^2}}{\sum\limits_{j \in \intset} \weight_j^2}. \label{eq:contamination}
\end{align}

Next we consider examples of weighting schemes for a given $n,d$ pair with large enough $\frac{d}{n}$ when the true signal is at $\truek$. The set of indices containing aliases of the true signal is denoted as $\aliassetk$ as in \eqref{eq:aliasset}.

\begin{example}
    \item Uniform weights, $\weight_j = 1$.
    \begin{align*}
        \estimatedalpha_j &= \begin{cases}\frac{1}{1 + \frac{d}{n} - 1} = \frac{n}{d}, & j \in \intset\\
        0, & \text{otherwise}.
        \end{cases}\\
        \survival &= \frac{n}{d},\\
        \contamination &= \frac{\sqrt{\frac{d}{n}}}{1 + \frac{d}{n} - 1} = \sqrt{\frac{n}{d}}.
    \end{align*}
    \end{example}
    
    \begin{example}
        
  Spiked weights on low frequency features: This selects a fraction of energy to put on the favored set of $s < n$ features. Namely. for $\gamma \in [0,1]$ and $s < n$.  $$\weight_j=  \begin{cases}
			 \sqrt{\frac{\gamma d}{s}}, & 0 \leq j < s\\
            \sqrt{\frac{(1 - \gamma)d}{d-s}}, & \text{otherwise}.
		 \end{cases}$$
	For $0 \leq \truek < s$, 
	\begin{align*}
	    \estimatedalpha_j &=   \begin{cases} \frac{\frac{\gamma d}{s}}{\frac{\gamma d}{s} + (\frac{d}{n}-1).\frac{(1- \gamma)d}{d-s}} \approx \frac{1}{1 + \frac{(1- \gamma)}{n\gamma(\frac{1}{s}-\frac{1}{d})}} \approx \frac{1}{1 + \frac{s}{n}\left(\frac{1}{\gamma} - 1\right)}, & j = \truek \\
	    \frac{\frac{(1- \gamma)d}{d-s}}{\frac{\gamma d}{s} +\left(\frac{d}{n} -1 \right)\frac{(1- \gamma)d}{d-s}} = \frac{1}{\frac{\gamma(d-s)}{s(1-\gamma)} + \frac{d}{n} -1} \approx \frac{1}{\left(\frac{\gamma}{1 - \gamma}\right)\frac{d}{s} + \frac{d}{n}}, & j \in \aliassetk\\
	    0, & \text{otherwise}.
	    \end{cases}\\
	    \survival & \approx   \frac{1}{1 + \frac{s}{n}\left(\frac{1}{\gamma} - 1 \right) }\\
	    \contamination & \approx  \frac{\sqrt{\frac{d}{n}}}{\left(\frac{\gamma}{1 - \gamma}\right)\frac{d}{s} + \frac{d}{n}} = \sqrt{\frac{n}{d}}. \frac{1}{\left( \frac{\gamma}{1 - \gamma}\right)\frac{n}{s} +  1} = \frac{s}{\sqrt{nd}}. \frac{1}{\left(\frac{\gamma}{1 - \gamma}\right) + \frac{s}{n}} .
	\end{align*}
    For $s \leq \truek < n$,
    \begin{align*}
        \estimatedalpha_j &= \begin{cases}\frac{1}{1 + \frac{d}{n} - 1} = \frac{n}{d}, & j \in \intset\\
        0, & \text{otherwise}.
        \end{cases}\\
        \survival &= \frac{n}{d},\\
        \contamination &= \frac{\sqrt{\frac{d}{n}}}{1 + \frac{d}{n} - 1} = \sqrt{\frac{n}{d}}.
    \end{align*}
   \end{example}

\end{document}